\documentclass[11pt]{article}

\usepackage[margin=1in]{geometry}
\usepackage{amsmath,amssymb,amsthm}
\usepackage{thmtools}
\usepackage{authblk}

\usepackage[numbers,sort&compress]{natbib}
\usepackage{algorithm}
\usepackage{algpseudocode}
\usepackage{xcolor}
\usepackage[colorlinks=true,citecolor=blue,linkcolor=blue,urlcolor=blue]{hyperref}
\usepackage[capitalise,nameinlink]{cleveref}
\usepackage{enumitem}
\usepackage{booktabs}
\usepackage{array}
\newcommand{\R}{\mathbb{R}}
\newcommand{\E}{\mathbb{E}}
\newcommand{\Pp}{\mathbb{P}}
\newcommand{\KL}{\mathrm{KL}}
\newcommand{\kl}{\mathrm{kl}}
\newcommand{\Cov}{\mathrm{Cov}}
\newcommand{\argmin}{\operatorname*{arg\,min}}
\newcommand{\poly}{\operatorname{poly}}
\newcommand{\TV}{\mathrm{TV}}

\newcommand{\cA}{\mathcal{A}}
\newcommand{\cP}{\mathcal{P}}
\newcommand{\cX}{\mathcal{X}}
\newcommand{\cH}{\mathcal{H}}

\newtheorem*{informaltheorem}{Theorem}
\newtheorem{lemma}{Lemma}
\newtheorem{proposition}{Proposition}
\newtheorem{corollary}{Corollary}

\newenvironment{proofsketch}
{\begin{proof}[Proof sketch]}
{\end{proof}}

\title{An Efficient Minimax-Optimal Algorithm\\ for Adversarial $m$-Set Bandits\thanks{A Lean formalization accompanying this paper,
together with instructions for independent verification,
is available in the
\href{https://github.com/robertocolomboni/an-efficient-minimax-optimal-algorithm-for-adversarial-m-set-bandits-lean-formalization}{companion repository}.}}
\author[1]{Francesco Bacchiocchi}
\author[2]{Tommaso Cesari}
\author[3]{Roberto Colomboni}

\affil[1]{DEIB, Politecnico di Milano, Milano, Italy}
\affil[2]{School of Electrical Engineering and Computer Science, University of Ottawa, Ottawa, Canada}
\affil[3]{School of Mathematics, University of Bristol, Bristol, United Kingdom}

\affil[ ]{\footnotesize
\texttt{francesco.bacchiocchi@polimi.it},
\texttt{tcesari@uottawa.ca},
\texttt{roberto.colomboni@bristol.ac.uk}
}
\date{\today}

\begin{document}
\addtocontents{toc}{\protect\setcounter{tocdepth}{-1}}
\maketitle

\begin{abstract}
We study adversarial combinatorial bandits with $m$-set actions, where at
each round the learner selects $m$ out of $d$ items and observes only the
aggregate loss of the selected items. The resulting action set contains
$K=\binom{d}{m}$ elements and can therefore be exponentially large.
We propose a computationally efficient algorithm that exploits this structure without explicitly
enumerating the action set. Against adaptive non-anticipating adversaries, it
guarantees, with probability at least $1-\delta$, regret against the best action in hindsight of
\(
    R_T
=
    O(\sqrt{dT\log(K/\delta)}).
\)
This matches the high-probability regret bound of the finite-action EXP3 with Kiefer-Wolfowitz exploration algorithm of \citet{zimmert2022return}, whose direct
implementation may require exponential space.
Our algorithm instead represents each sampling distribution with $d$ parameters and runs in polynomial time without enumerating the action set. This resolves an open problem posed by \citet{maiti2025efficient}.
We complement this upper bound with a matching high-probability lower bound.
For all sufficiently small $\delta$ and all sufficiently large horizons, every randomized algorithm admits an adaptive non-anticipating adversary for which, with probability at least $\delta$,
\(
    R_T
=
    \Omega(\sqrt{dT\log(K/\delta)}).
\)
Thus, the rate is minimax optimal up to universal constants in this regime.
In particular, setting $m=1$ shows that the $\log K$ factor in the high-probability regret bound for ordinary $K$-armed bandits against adaptive non-anticipating adversaries is unavoidable.
This strengthens the lower bound of \citet{gerchinovitz2016refined}, closing the remaining $\sqrt{\log K}$ gap between high-probability upper and lower regret bounds.
\end{abstract}
\newpage
\section{Introduction}
\label{sec:introduction}

In this paper, we study adversarial $m$-set bandits.
At each round, the learner selects $m$ out of $d$ items and incurs the sum of their losses, but observes only the aggregate loss, not the loss of each selected item. Since any $m$-set may be chosen, the learner has $K=\binom{d}{m}$ possible actions, and its goal is to minimize regret with respect to the best fixed $m$-set in hindsight.
In this work, we assume that, at every round, the loss of each action,
namely each $m$-set, lies in $[-1,1]$. The loss vectors are chosen by an adaptive non-anticipating adversary and may depend on the past interaction history but not on the learner's random choice.

\citet[Theorem~6 and Algorithm~3]{zimmert2022return} show that, in this setting, EXP3 with Kiefer-Wolfowitz exploration, abbreviated as EXP3-KW, achieves
$O(\sqrt{dT\log(K/\delta)})$ regret with probability at least $1-\delta$, where $T$ represents the number of rounds  and $\delta$ the failure probability.
For $m$-sets with $m\le d/2$, this gives $\widetilde O(\sqrt{dmT})$.\footnote{The notation $\widetilde O$ suppresses logarithmic factors.}
However, the direct implementation maintains one weight for each of the $K=\binom dm$ actions and may therefore require exponential space.

\citet[Appendix~E.6]{maiti2025efficient} obtain a polynomial-time high-probability algorithm with regret $\widetilde O(d\sqrt{mT})$ and ask whether the extra $\sqrt d$ factor can be removed efficiently.
Furthermore, applying their general directed acyclic graph (DAG) guarantee (Theorem~7)
to the $m$-set DAG constructed in their Appendix~E.6, which has
$O(md)$ edges and $\binom dm$ paths, yields the sharper bound
$\widetilde O(m\sqrt{dT})$.

Our algorithm attains the finite-action EXP3-KW regret rate in polynomial time, representing each sampling distribution with only $d$ parameters.
Writing $s:=\min\{m,d-m\}$, its regret is $\widetilde O(\sqrt{dsT})$. For $m\le d/2$, this becomes $\widetilde O(\sqrt{dmT})$, improving the two guarantees of \citet{maiti2025efficient} by factors of $\sqrt d$ and $\sqrt m$, respectively.
Thus, this resolves their open question.
\Cref{tab:comparison} compares our regret bound and computational complexity with those of state-of-the-art algorithms.

\begin{table}[!htp]
\centering
\small
\renewcommand{\arraystretch}{1.3}
\begin{tabular}{@{}
>{\raggedright\arraybackslash}p{0.4\textwidth}
>{\raggedright\arraybackslash}p{0.2\textwidth}
>{\raggedright\arraybackslash}p{0.3\textwidth}@{}}
\toprule
Work & Regret bound & Computational complexity \\
\midrule
\citet{zimmert2022return} (Theorem~6 and Algorithm~3)
& $\widetilde O(\sqrt{dmT})$
& Exponential \\

\citet{maiti2025efficient}
(Theorem~{7} applied to the DAG of Appendix~E.6)
& $\widetilde O(m\sqrt{dT})$
& Polynomial  \\

\textbf{This work}
& $\boldsymbol{\widetilde O(\sqrt{dmT})}$
& \textbf{Polynomial} \\
\bottomrule
\end{tabular}
\caption{Statistical and computational comparison for the regime $m\le d/2$. All upper bounds hold with high probability against adaptive non-anticipating adversaries.}
\label{tab:comparison}
\end{table}

We also provide a matching lower bound. For every $d$ and $m$, all sufficiently small $\delta$, and all horizons above a polynomial threshold, every randomized algorithm admits an adaptive non-anticipating adversary for which
$R_T=\Omega(\sqrt{dT\log(K/\delta)})$ with probability at least $\delta$.
Together, these results identify the minimax high-probability rate in this regime.

The special case $m=1$, and hence $d=K$, gives a consequence for ordinary adversarial $K$-armed bandits.
Here it is essential to distinguish oblivious adversaries, which fix the entire loss sequence in advance, from adaptive non-anticipating adversaries.
For oblivious adversaries, the algorithm of \citet[Theorem~18]{audibert2010regret} {gives} an upper bound of $O(\sqrt{KT\log(1/\delta)})$, matched by the lower bound of \citet[Theorem~1]{gerchinovitz2016refined}.
For adaptive non-anticipating adversaries, the upper bound presented in \citet[Theorem~6.10]{cesabianchi2006prediction} is instead $O(\sqrt{KT\log(K/\delta)})$ . Our lower bound shows that the additional $\log K$ is necessary in this model, closing the remaining $\sqrt{\log K}$ gap.
We summarize these results in \Cref{tab:mab-high-probability}.

\begin{table}[!htp]
\centering
\small
\renewcommand{\arraystretch}{1.35}
\begin{tabular}{@{}>{\raggedright\arraybackslash}p{0.12\textwidth}
>{\centering\arraybackslash}p{0.39\textwidth}
>{\centering\arraybackslash}p{0.39\textwidth}@{}}
\toprule
& Oblivious & Adaptive non-anticipating \\
\midrule
Upper bound
& $O\!\left(\sqrt{KT\log(1/\delta)}\right)$\newline
  \citet[Theorem~18]{audibert2010regret} (2010)
& $O\!\left(\sqrt{KT\log(K/\delta)}\right)$\newline
  \citet[Theorem~6.10]{cesabianchi2006prediction} (2006) \\
\addlinespace
Lower bound
& $\Omega\!\left(\sqrt{KT\log(1/\delta)}\right)$\newline
  \citet[Theorem~1]{gerchinovitz2016refined} (2016)
& $\boldsymbol{\Omega\!\left(\sqrt{KT\log(K/\delta)}\right)}$\newline
  \textbf{This work} (\Cref{cor:ordinary-bandits-lower-bound}) \\
\bottomrule
\end{tabular}
\caption{High-probability regret bounds for ordinary \(K\)-armed bandits (\(m=1\), \(d=K\)), with losses in \([0,1]\). Upper bounds hold with probability at least \(1-\delta\), and lower bounds with probability at least \(\delta\). The comparison is for sufficiently small \(\delta\) and sufficiently large \(T\).}
\label{tab:mab-high-probability}
\end{table}
\section{Setting and problem formulation}
\label{sec:setting}

In this section, we introduce the notation used throughout the paper and formulate the adversarial $m$-set bandit problem.

\paragraph{Preliminary definitions.}
For a positive integer $n$, we write $[n]:=\{1,\ldots,n\}$.
In the following, we fix two integers $d\ge2$ and $1\le m\le d-1$.
Then, we define the action set and its size as
\begin{equation}
\label{eq:XiSet}
    \cX
:=
    \cX_{d,m}
:=
    \left\{x\in\{0,1\}^d:\|x\|_1=m\right\},
\qquad
    K:=|\cX|=\binom{d}{m}.
\end{equation}
Each action $x\in\cX$ selects exactly $m$ of the $d$ items, where $x_i=1$ if item $i$ is selected and $x_i=0$ otherwise.
The set of items selected by each action is an $m$-element subset of $[d]$, which we call an \emph{$m$-set}.
Conversely, each $m$-set $S\subseteq[d]$ corresponds to the indicator vector $\mathbf 1_S\in\cX$.
We therefore use actions in $\cX$ and $m$-sets interchangeably. Finally, to express the regret bounds in terms of $d$ and $m$, we also introduce $s:=\min\{m,d-m\}$.

\paragraph{Linear-algebra notation.}
We write $\mathbf0:=(0,\ldots,0)^\top$ for the zero vector.
For a symmetric matrix $A\in\R^{d\times d}$, we denote its kernel by
$
    \ker(A)
:=
    \{z\in\R^d:Az=\mathbf0\}$,
its range (or image) by
$
    \operatorname{range}(A)
:=
    \{Az:z\in\R^d\}$,
and its Moore-Penrose pseudoinverse by $A^+$. A symmetric matrix $A$ is positive semidefinite if $z^\top Az\ge0$ for every $z\in\R^d$.
We write $A\succ0$ if $A$ is positive definite, that is, $z^\top Az>0$ for every nonzero $z\in\R^d$.
Let $\mathbf 1:=(1,\ldots,1)^\top$, and define
\[
    \operatorname{span}\{\mathbf 1\}
:=
    \{a\mathbf 1:a\in\R\},
\qquad
    H_0
:=
    \{y\in\R^d:y^\top \mathbf 1=0\}.
\]
Thus, $H_0$ is the subspace orthogonal to $\operatorname{span}\{\mathbf 1\}$. Formally, $H_0^\perp=\operatorname{span}\{\mathbf 1\}$.

\paragraph{Probability notation.}
We denote the probability simplex over $\cX$ by
\[
    \Delta(\cX)
:=
    \left\{ p: \cX\to[0,1] : \sum_{x\in\cX}p(x)=1 \right\}.
\]
For probability laws $P$ and $Q$ on a common measurable space, the Kullback-Leibler (KL) divergence is $\KL(P\|Q):=\int\log(dP/dQ)\,dP$ if $P\ll Q$, and $+\infty$ otherwise. Here $P\ll Q$ means that $P$ is absolutely continuous with respect to $Q$, and $dP/dQ$ is the Radon-Nikodym derivative of $P$ with respect to $Q$. For $p,q\in\Delta(\cX)$, this specializes to
\(
    \KL(p\,\|\,q)
{=}
    \sum_{x\in\cX}p(x)\log({p(x)}/{q(x)}).
\)\footnote{We use the convention $0\log(0/q)=0$; a term with $p(x)>0=q(x)$ is $+\infty$.}
To describe the moments of a distribution, let $X$ be a random vector drawn from $p\in\Delta({\cX})$.
We define the marginal vector, second-moment matrix, and covariance matrix of $p$ as
\[
    \mu(p)
:=
    \E_{X\sim p}[X],
\qquad
    M(p)
:=  
    \E_{X\sim p}[XX^\top],
\qquad
    \Sigma(p)
:=
    \Cov_{X\sim p}(X)
=
    M(p)-\mu(p)\mu(p)^\top.
\]
The $i$-th marginal probability $\mu_i(p)$ satisfies \(\mu_i(p)
=
    \Pp_{X\sim p}[X_i=1]
\)
and represents the probability that action $X$ selects item $i$.
In addition, if $M(p)$ is invertible, we refer to the quantity
\begin{equation}\label{eq:leverage-score}
    x^\top M(p)^{-1}x
\end{equation}
as the \emph{leverage score} of $x$ under $p$.

We also use the total variation distance $\TV(P,Q):=\sup_E|P(E)-Q(E)|$, where the supremum is over measurable events.
We write $\kl(a,b):=\KL(\operatorname{Ber}(a)\|\operatorname{Ber}(b))$ for binary KL, with the same zero-probability conventions as above.

\paragraph{Weighted $m$-set distributions.}
Given $\theta\in\R^d$, define $w_i:=e^{\theta_i}>0$ for every $i\in[d]$.
For every $m$-set $S\subseteq[d]$, the weighted $m$-set distribution $p_\theta$ is
\[
    p_\theta(S)
:=
    \frac{\exp\left(\sum_{i\in S}\theta_i\right)}{\sum_{A\subseteq[d]:\,|A|=m}\exp\left(\sum_{i\in A}\theta_i\right)}
= 
    \frac{\prod_{i\in S}w_i}{\sum_{A\subseteq[d]:\,|A|=m}\prod_{i\in A}w_i}.
\]
Thus, $p_\theta(S) \propto e^{\theta^\top \mathbf{1}_S}$. Notice that, although $p_\theta$ may assign positive probability to exponentially many actions, only the $d$ parameters $\theta_1,\ldots,\theta_d$ represent it. 
Finally, we define a set of distributions whose marginal probabilities stay away from $0$ and $1$.
Let
$
    r
:=
    m/d
$
and fix $\lambda\in(0,1)$.
Define
\begin{equation}\label{eq:P_lambda_set}
    \cP_\lambda
:=
    \left\{ p\in\Delta(\cX) : \lambda r \le \mu_i(p) \le 1-\lambda(1-r) \text{ for every }i\in[d] \right\}.
\end{equation}
Notice that the uniform distribution $U$ over $\cX$ satisfies $\mu_i(U)=r$ for every $i\in[d]$ and therefore belongs to $\cP_\lambda$ for every $\lambda \in (0,1)$.

\paragraph{Adversarial $m$-set bandit protocol.}
Fix a time horizon $T\ge 1$.
Let $(\mathcal F_t)_{t=0}^T$ denote the filtration generated by the interaction history, where $\mathcal F_{t-1}$ contains the information available before round $t$.
At each round $t\in[T]$, the learner selects an $\mathcal F_{t-1}$-measurable distribution $p_t\in\Delta(\cX)$, where $\cX$, defined in \Cref{eq:XiSet}, represents the set of actions available to the learner.
The adversary selects a loss vector $\ell_t\in\mathbb R^d$ that is also $\mathcal F_{t-1}$-measurable, where $\ell_{t,i}$ represents the loss associated with item $i \in [d]$.
Thus, the adversary is \emph{adaptive non-anticipating}.
The loss vector $\ell_t$ may depend on the entire interaction history up to round $t-1$, but not on the learner's random draw at round $t$.
Therefore, conditionally on $\mathcal F_{t-1}$, both $p_t$ and $\ell_t$ are fixed.
The learner then draws an action $X_t\sim p_t$ and observes only the (action) loss $Y_t:=X_t^\top \ell_t.$
We use the standard action-level normalization
\[
    |x^\top \ell_t|
\le
    1
\qquad
    \forall x\in\cX,\ t\in[T],
\]
also adopted by \citet[Assumption~1]{maiti2025efficient} and \citet[Section~1]{zimmert2022return}.
Under this protocol, the learner's realized regret after $T$ rounds, measured against the best fixed action in hindsight, is
\begin{equation}\label{eq:regret}
    R_T
:=
    \sum_{t=1}^T{X_t^\top \ell_t} - \min_{x\in\cX} \sum_{t=1}^T{x^\top \ell_t}.
\end{equation}
Under an adaptive adversary, this action may
depend on the realized interaction history, while remaining fixed
across rounds.\footnote{More generally, for a comparator distribution $q\in\Delta(\cX)$, which is fixed across rounds but may be chosen in hindsight, one can define the regret with respect to $q$ as
	\(
	R_T(q):=\sum_{t=1}^T\left(X_t^\top \ell_t-\E_{X\sim q}\!\left[X^\top \ell_t\right]\right).
	\)
	Our notion of regret introduced in \eqref{eq:regret} then satisfies $R_T=\max_{q\in\Delta(\cX)}R_T(q)$, with the maximum attained by the Dirac mass at a best action in hindsight.}
Throughout, we use the shorthand
$
    \E_t[\,\cdot\,]
:=
    \E[\,\cdot\, | \mathcal F_{t-1}]
$
for conditional expectation given the information available before round $t \in [T]$.

\section{Original contributions}
\label{sec:contributions}

\subsection{Upper bound}
\label{sec:upper-contribution}

Our first contribution is an algorithm (\Cref{alg:exp3-kw}) for adversarial $m$-set bandits that achieves the regret guarantee summarized in the following.
\begin{informaltheorem}[{Informal version of Theorem~\ref{thm:main}}]
There exists an algorithm that can be implemented in polynomial time such that, against any adaptive non-anticipating adversary, with probability at least $1-\delta$, it achieves
\[
    R_T
=
    O \left(\sqrt{dT \left(\log\binom dm+\log\frac1\delta\right)}\right).
\]
For example, one may take \Cref{alg:exp3-kw}.
\end{informaltheorem}

\paragraph{Challenges and techniques.}
\label{sec:techniques}

In their finite-action EXP3-KW algorithm, \citet[Algorithm~3]{zimmert2022return} obtain a tight high-probability regret bound by applying exponential weights (also called multiplicative weights) to the biased surrogate loss
\begin{equation}
\label{eq:kw-surrogate-loss}
    x^\top \widehat{\ell}_t - \eta x^\top M_t^{-1}x,
\end{equation}
where $\eta>0$ is the learning rate,
\(
    M_t
=
    \E_{X\sim p_t}[XX^\top]
\)
is the second-moment matrix of the sampling distribution $p_t$, and
\begin{equation}
\label{eq:kw-loss-estimate}
    \widehat{\ell}_t
:=
    M_t^{-1}X_t (X_t^\top \ell_t)
\end{equation}
is the loss estimate constructed after drawing $X_t\sim p_t$ and observing $X_t^\top \ell_t$.

The main obstacle to an efficient implementation of this approach is that the leverage score $x^\top M_t^{-1}x$ is quadratic in $x$. As a result, even if we start from a weighted $m$-set distribution and use the biased surrogate loss defined in \eqref{eq:kw-surrogate-loss} in a multiplicative-weights update, the updated distribution need not remain in the weighted $m$-set family. Thus, this update does not in general preserve the $d$-parameter weighted $m$-set representation.

Our first step is to replace the leverage score $x^\top M_t^{-1}x$ in \eqref{eq:kw-surrogate-loss} with an affine upper bound.
Combining the pseudoinverse estimate of \citet[Corollary~1.2]{cesari2026sharp} with an identity relating the second-moment matrix to the covariance, we derive
\begin{equation}
\label{eq:affine-leverage-bound}
    x^\top M_t^{-1}x
\leq \phi_{p_t}(x)
:=
    1+\sum_{i=1}^d \frac{(x_i-\mu_{i}(p_t))^2} {\mu_i(p_t)(1-\mu_{i}(p_t))}.
\end{equation}
Although the expression defining $\phi_{p_t}$ appears quadratic, it is affine on $\cX$.
Indeed, every $x\in\cX$ satisfies $x_i^2=x_i$ and $\sum_{i=1}^d x_i=m$.
For later convenience, we scale this affine majorant by four and write
\[
    c_{p_t}^\top x
=
    4\phi_{p_t}(x)
\qquad
    \text{for every }x\in\cX.
\]
We then define the surrogate loss
\begin{equation}
\label{eq:biased-omd-loss}
    Z_t(x)
:=
    x^\top \widehat{\ell}_t - \eta c_{p_t}^\top x.
\end{equation}
Here, $\widehat\ell_t$ is the loss estimate defined in \eqref{eq:kw-loss-estimate}.
The correction in \Cref{eq:biased-omd-loss} plays the same role as the quadratic correction used by \citet{zimmert2022return}, but remains affine in $x$.
Consequently, the multiplicative-weights step maps a weighted $m$-set distribution to another weighted $m$-set distribution.

However, the affine leverage majorant $\phi_{p_t}(x)$ can become arbitrarily large when a marginal probability $\mu_i(p_t)$ approaches $0$ or $1$. The bound in \eqref{eq:affine-leverage-bound} then no longer provides uniform control of the estimated action losses $x^\top \widehat\ell_t$. Therefore, we combine the exponential-weights update with an approximate KL projection toward the set $\cP_\lambda$, introduced in \Cref{sec:setting}.
Specifically, at each round, we first apply an exponential-weights update to $p_t$ using the surrogate loss $Z_t(x)$ defined in \eqref{eq:biased-omd-loss}, obtaining a weighted $m$-set distribution $\widetilde p_{t+1}$. We then compute a distribution $p_{t+1}$ that approximates the exact KL projection $p_{t+1}^\star$, defined by
\begin{equation}
\label{eq:techniques-kl-projection}
p_{t+1}^\star
\in
\argmin_{p\in\cP_\lambda}
\KL\!\left(p\,\middle\|\,\widetilde p_{t+1}\right).\footnote{This two-step procedure is an approximate implementation of online mirror descent (OMD) over $\cP_\lambda$ with the negative entropy regularizer, as discussed in Section~\ref{sec:algorithm-regret}.}
\end{equation}
The distribution $p_{t+1}$ has two crucial properties. It retains the weighted $m$-set structure and belongs to the slightly larger set $\cP_{\lambda/2}$ for a suitably chosen $\lambda$. The former condition  ensures that every iterate can be represented as $p_{\theta_t}$ using the $d$ parameters
\(
\theta_{t,1},\ldots,\theta_{t,d},
\)
rather than one weight for each of the $K$ actions, while the latter condition keeps the marginal probabilities away from $0$ and $1$, ensuring that the estimated action losses remain under control.

One of the challenges in our analysis is that our algorithm samples from distributions in $\cP_{\lambda/2}$. However, an optimal action $x$ in hindsight corresponds to the Dirac distribution $\delta_x$. This distribution puts all its probability on $x$, has marginal probabilities in $\{0,1\}$, and therefore does not belong to $\cP_\lambda$.
Therefore, in our analysis, we use the smoothed comparator
\begin{equation}
	\label{eq:smoothed-comparator}
	q_x
	:=
	(1-\lambda)\delta_x+\lambda U,
\end{equation}
where $U$ is the uniform distribution over $\cX$.
By construction, $q_x\in\cP_\lambda$. Thus, 
under our action-level normalization, replacing the pure comparator
$\delta_x$ with $q_x$ increases the regret bound by at most
$2\lambda T$.

Another challenge in achieving the same high-probability regret rate as
\citet{zimmert2022return} arises from replacing their quadratic correction $x^\top M_t^{-1}x$
in \eqref{eq:kw-surrogate-loss} with the larger affine correction $c_{p_t}^\top x$
in \eqref{eq:biased-omd-loss}.
This is because the exponential-weights analysis yields a regret bound
that depends on the sum over the $T$ rounds of the second moments $\E_{X\sim p_t}[Z_t(X)^2]$. Since the larger correction $c_{p_t}^\top x$ may increase
this quantity, we need to show that its contribution to the regret
bound remains of the same order as in their analysis. 
To control this contribution, we show that the affine correction $c_{p_t}^{\top}x$ is nonnegative and satisfies
\begin{equation}
\label{eq:techniques-affine-bounds}
\E_{X\sim p_t}[c_{p_t}^{\top}X]=O(d),
\qquad
\max_{x\in\cX}c_{p_t}^{\top}x=O(d/\lambda).
\end{equation}
For sufficiently large $T$, combining these two bounds and using our choice $\lambda=\Theta(\sqrt{d\log(K/\delta)/T})$, we obtain, with probability at least $1-\delta/2$,
\begin{equation}
\label{eq:techniques-second-moment}
\sum_{t=1}^{T}\E_{X\sim p_t}[Z_t(X)^2]
\le O(dT).
\end{equation}
Since the sum in \eqref{eq:techniques-second-moment} enters the regret bound multiplied by $\eta=\Theta(\sqrt{\log(K/\delta)/(dT)})$, its contribution is $O(\sqrt{dT\log(K/\delta)})$.
Thus, replacing the quadratic correction $x^\top M_t^{-1}x$ with the larger affine correction $c_{p_t}^{\top}x$ in the surrogate loss defined in \eqref{eq:biased-omd-loss} does not increase the order of this term in our algorithm's regret bound.

These bounds also allow us to avoid the Kiefer-Wolfowitz exploration used by \citet{zimmert2022return}.
Indeed, their algorithm mixes the exponential-weights distribution with a fixed exploration distribution to guarantee that $x^\top M_t^{-1}x=O(d/\lambda)$ for every action $x\in\cX$ and every round $t\in[T]$.
This bound, in turn, allows them to control the sum over $t\in[T]$ of the second moments of their surrogate losses defined in \eqref{eq:kw-surrogate-loss}, obtaining a contribution to regret of the same order as the one we obtain from \eqref{eq:techniques-second-moment}.
Intuitively, Kiefer-Wolfowitz exploration is unnecessary here because our projection step keeps each sampling distribution $p_t$ in $\cP_{\lambda/2}$ while preserving its weighted $m$-set structure.
Equations~\eqref{eq:affine-leverage-bound} and~\eqref{eq:techniques-affine-bounds} therefore already guarantee that the leverage score $x^\top M_t^{-1}x$ is $O(d/\lambda)$ for every action $x\in\cX$ and every round $t\in[T]$.

Finally, we discuss how to implement our algorithm in polynomial time.
Since our algorithm guarantees that each $p_t$ is a weighted $m$-set distribution, the methods of \citet[Sections~2 and~4]{chen1997statistical} allow us to sample an action $X_t\sim p_t$ and compute the marginal vector $\mu(p_t)$ and second-moment matrix $M(p_t)$ in $\poly(d,m)$ time.\footnote{The notation $\poly(\cdot)$ denotes a quantity bounded above by a polynomial in the input arguments.}

The step that requires careful attention is the KL projection in \eqref{eq:techniques-kl-projection}.
We need to implement this step in polynomial time while ensuring that the resulting distribution is a weighted $m$-set distribution.
To achieve this, we allow the resulting distribution to lie in the larger set $\cP_{\lambda/2}$ and to satisfy the optimality condition for the KL projection approximately.
We compute such a distribution by solving a suitably defined convex optimization problem using the ellipsoid method, with a computational cost of $\poly(d,m,\log T)$ per projection.

\subsection{Lower bound}
\label{sec:lower-contribution}

Our second contribution is a matching high-probability lower bound.
\begin{informaltheorem}[{Informal version of Theorem~\ref{thm:lower-bound}}]
There exists a universal constant $\delta_0>0$ such that the following holds.
For every $\delta\in(0,\delta_0]$ and every horizon $T$ above a threshold bounded by a polynomial in $d$, $m$, and $\log(1/\delta)$, for every randomized  algorithm there exists an adaptive non-anticipating adversary with action losses in $[0,1]$ for which, with probability at least $\delta$,
\[
    R_T
=
    \Omega\left(\sqrt{dT\left(\log\binom dm+\log(1/\delta)\right)}\right).
\]
\end{informaltheorem}
Together with \cref{thm:main}, this identifies the minimax high-probability regret up to universal constants for every $d$ and $m$.

\paragraph{Challenges and techniques.}
\label{sec:lower-techniques}
We give an intuition for how adaptivity creates the $\log K$ term in the lower bound.
To make the argument easier to follow, we start with the case $m=1$ and $d=K$. This coincides with the ordinary $K$-armed bandit problem. Thus, we use the terms item and arm interchangeably.
In the following, we choose a uniformly random set $W$ of $K'+1$ arms, where $K'=K/32$, and select $J$ uniformly from $W$.
For simplicity, we assume that $K/32$ and $4T/K$ are integers.

We compare two instances, writing $P$ and $Q$ for the corresponding probability distributions of $W$, $J$, and all the learner's selected arms and observed losses.
In $P$, the favorable arms are those in $W\setminus\{J\}$. In $Q$, they are all the arms in $W$.
By a favorable arm, we mean an arm that receives a loss discount, which we call a bonus, while it is active. Here, by an active arm we mean an arm that has been played fewer than $N_0:=4T/K$ times. Indeed, we construct the instances so that the adversary switches off the bonus after the arm's $N_0$th play, making it adaptive. Arms outside the favorable set receive only the baseline loss throughout.
Thus, the two instances differ only on $J$. Specifically, the arm $J$ receives the bonus while active in $Q$, and no bonus in $P$.

At each round $t$, we draw $Z_t$ from the same fixed distribution,
independently of $W$, $J$, the previous draws, and the learner's
randomness. We use this random variable, shared by all arms, to define
the loss of a baseline arm (an arm receiving no bonus) as $b(Z_t)$,
while we let the loss of an active favorable arm be
$b(Z_t)-\varepsilon g(Z_t)$, where $b$ and $g$ are fixed functions.

We choose this distribution and the functions $b$ and $g$ so that,
for sufficiently small $\varepsilon>0$, both losses lie in $[0,1]$.
Moreover, we require $0\le g(Z_t)\le1$ and $\E[g(Z_t)]=\gamma$,
where $\gamma$ is a sufficiently large universal constant.\footnote{Throughout this argument, $\Pp$ and $\E$ refer to the learner's randomness, the choices of $W$ and $J$, and the variables $Z_1,\ldots,Z_T$ used to generate the adversary's losses, in instance $P$. }

Writing $\nu_0$ and $\nu_\varepsilon$ for the distributions of
$b(Z_t)$ and $b(Z_t)-\varepsilon g(Z_t)$, respectively, their mean
difference is
\(
    \E_{\nu_0}[X]-\E_{\nu_\varepsilon}[X]
    =\varepsilon\gamma=\Theta(\varepsilon),
\)
and we also require
\(
    \KL(\nu_0\Vert\nu_\varepsilon)=O(\varepsilon^2).
\)
The bonus obtained at round $t$ by playing an active favorable arm
compared with a non-favorable one is therefore $\varepsilon g(Z_t)$,
and the common baseline $b(Z_t)$ cancels when computing regret.

We notice that in instance $P$, the learner can receive a bonus on at most $N_0K'=T/8$ rounds, so its total bonus is at most $\varepsilon T/8$.
A favorable arm that remains active throughout all $T$ rounds, however, receives the bonus on every round, for a total of $\varepsilon\sum_{t=1}^T g(Z_t)$.
This sum has expectation $\varepsilon\gamma T$, since $\E[g(Z_t)]=\gamma$.
For our sufficiently large choice of $\gamma$, this exceeds the learner's maximum total bonus by an amount of order $\varepsilon T$.
Thus, ideally, to incur small regret, the learner would like to switch off all favorable arms with high probability.
We now explain why this is not the case and some favorable arm remains active with constant probability under $P$.

Let $n_i(T)$ be the number of times the learner plays arm $i$, and let $A$ be the event that every arm in $W\setminus\{J\}$ is played at least $N_0$ times.
Consider the event

\[
\begin{aligned}
    E=\bigl\{n_i(T)\ge N_0\text{ for every }i\in W\setminus\{J\}\bigr\}
       \cap\{n_J(T)<N_0\}
     =A\cap\{n_J(T)<N_0\}.
\end{aligned}
\]
Thus, $E$ is the event $A$ together with the requirement that $J$ remains active.
We observe that, for any algorithm, at most $T/N_0=K/4$ arms can be played at least $N_0$ times, so at least $3K/4$ arms remain active at the end of the $T$ rounds, according to the definition of an active arm.

In instance $P$, the initial sampling is equivalent to first choosing the favorable set $W\setminus\{J\}$ uniformly among all sets of $K'$ arms and then selecting $J$ uniformly outside this set.
Consequently, conditional on the set $W\setminus\{J\}$, $J$ remains uniform among the arms outside $W\setminus\{J\}$ and is independent of the learner's selected arms.

Assume first that $P(A)>0$. Under the event $A$, at the end of the $T$ rounds, all the active arms are outside $W\setminus\{J\}$, so the probability that $J$ is active is at least $3/4$. Consequently, averaging over the favorable set and using the fact that at least $3K/4$ arms remain active at the end of the $T$ rounds, we have $P(E)=P(E\mid A)P(A)\ge3P(A)/4$. If $P(A)=0$, then $P(E)=0$ since $E\subseteq A$, so the same bound holds.

Under $Q$, every arm in $W$ receives the bonus up to and including its $N_0$th play, regardless of which arm is selected as $J$.
Given $W$, $J$ therefore remains uniform in $W$ and is independent of which arms remain active.
For $E$ to occur, exactly one arm in $W$ must remain active at the end of the $T$ rounds, and it must be $J$.
If exactly one arm in $W$ remains active, the probability that this arm is $J$ is $1/(K'+1)$. If none or more than one remain active in $W$ at the end, $E$ cannot occur.
Hence
\(
    Q(E)\le{1}/{(K'+1)}\le{32}/{K}.
\)

We now choose $\varepsilon$ so that $N_0\varepsilon^2=c\log K$, where $c>0$ is a sufficiently small fixed constant.
The two instances differ only when the learner plays $J$ while it is active, which happens at most $N_0$ times. The KL chain rule of \citet[Exercise~14.12, Eq.~(14.17)]{lattimore2020bandit} therefore gives $\KL(P\Vert Q)=O(N_0\varepsilon^2)$.
By data processing applied to the indicator of $E$, and the binary-KL inequality\footnote{For $a\in[0,1]$ and $b\in(0,1)$, we have
		$\KL(\operatorname{Ber}(a)\|\operatorname{Ber}(b))\ge a\log(1/b)-\log2$.}, we obtain, for sufficiently large $K$,
\[
    P(A)\log K
    \le O\!\left(\KL(P\Vert Q)+1\right)
    \le O(N_0\varepsilon^2).
\]
For a sufficiently small $c$ and sufficiently large $K$, this gives $P(A)\le1/4$. Thus, $P(A^c)\ge3/4$, where $A^c$ is the event that at least one favorable arm remains active in instance $P$ at the end of round $T$.

Finally, we translate this into a lower bound on realized regret.
On $A^c$, some favorable arm remains active throughout and receives total bonus $\varepsilon\sum_{t=1}^T g(Z_t)$.
On $A$, every arm's total bonus is still nonnegative.
Since the common term $b(Z_t)$ cancels, the $\varepsilon T/8$ bound on the learner's total bonus gives
\begin{align*}
{\E}[R_T]
&\ge \varepsilon\left(
    {\E}\!\left[
        \sum_{t=1}^T g(Z_t)\mathbf1_{A^c}
    \right]-\frac T8
\right)\\
&= \varepsilon\left(
    {\E}\!\left[
        \sum_{t=1}^T g(Z_t)
        \left( \mathbf1_{A^c}+\mathbf1_A-\mathbf1_A \right)
    \right]-\frac T8
\right)\ge \varepsilon\left(\gamma T-TP(A)-\frac T8\right)=\Omega(\varepsilon T).
\end{align*}
Here we used ${\E}[\sum_{t=1}^T g(Z_t)]=\gamma T$, the choice of $\gamma$, $P(A)\le1/4$, and $0\le\sum_{t=1}^T g(Z_t)\le T$.

Since $R_T\le\varepsilon T$ on every realization and
${\E}[R_T]=\Omega(\varepsilon T)$, an elementary probabilistic
argument (see Lemma~\ref{lem:lower-mean-to-probability}) gives ${\Pp}(R_T\ge c_0\varepsilon T)\ge c_0$ for some
universal constant $c_0>0$.
Our choice of $\varepsilon=\Theta(\sqrt{cK\log K/T})$ gives
$\varepsilon T=\Theta(\sqrt{KT\log K})$, proving the claimed
order with constant probability for sufficiently large $K$
and $T\ge CK\log K$, where $C$ is a sufficiently large
universal constant.
Notice that the adversary above randomizes its choices.
The probability guarantee is an average over its choices of $W$, $J$, and $Z_1,\ldots,Z_T$.
By an averaging argument, some fixed choice of these values preserves the same guarantee over the learner's randomness alone.
This gives a deterministic adaptive non-anticipating adversary of the type introduced in Section~\ref{sec:setting}.

In our lower bound construction, we can take $Z_t$ uniform on $[-1,1]$, $b(z)=(1+z)/2$, and $g(z)=1-z^2$.
These choices satisfy the properties used in the preceding argument.
The bonus $\varepsilon g(Z_t)$ always lies between $0$ and $\varepsilon$ and has mean $\Theta(\varepsilon)$, while the common baseline $b(Z_t)$ cancels in the regret.
Moreover, since $g(-1)=g(1)=0$, for sufficiently small $\varepsilon$ the losses with and without the bonus have the same range of possible values, $[0,1]$.
This choice of $b$ and $g$ also guarantees the KL bound used above.

The same properties are useful when the learner selects $m$ items and observes only the sum of their losses.
By the complementation reduction in \Cref{lem:lower-complement}, it suffices to consider $m\le d/2$.
Here, we divide each item's loss by $m$. If $h$ selected items are favorable and active, the observed loss has the distribution of $b(Z_t)-(\varepsilon h/m)g(Z_t)$.
Changing one additional selected active item from non-favorable to favorable changes this distribution to that of $b(Z_t)-(\varepsilon(h+1)/m)g(Z_t)$.
The common noise term $b(Z_t)$ therefore remains of the same size, while the two bonus coefficients differ by only $\varepsilon/m$.
The choice of $b$ and $g$ guarantees that the KL divergence between these distributions is $O(\varepsilon^2/m^2)$, regardless of $h$.
This smaller KL allows us to choose $\varepsilon=\Theta(\sqrt{dm\log(d/m)/T})$ while keeping at least $m$ favorable items active with constant probability.

We use a construction based on this idea, choosing the favorable items using i.i.d. Bernoulli random variables with a fixed success probability, to prove a lower bound of order $\sqrt{dT\log K}$, where $K=\binom{d}{m}$, when $m/d$ is at most a universal constant $r_\star>0$, as shown in Appendix~\ref{app:lower-combinatorial}. 
We also use a similar construction to prove the same bound in the regime considered in Appendix~\ref{app:lower-bounded-density}, where $r_\star<m/d\le1/2$. Here, all items remain active throughout, and the adversary is therefore oblivious. This construction gives a lower bound of order $m\sqrt T$, which already yields $\Omega(\sqrt{dT\log K})$ since $m>r_\star d$ and $\log K\le d\log2$.
Finally, a further construction gives the lower bound of order $\sqrt{dT\log(1/\delta)}$ in Appendix~\ref{app:lower-confidence}.

All three constructions use item losses with a structure similar to the one above, combining a random baseline shared by all items with small changes proportional to $g(Z_t)$, with different scales and fixed additional terms.
In each construction, the common random baseline cancels in the regret and the losses are guaranteed to remain in a bounded range.

Finally, we choose the construction that gives the larger bound.
Since $\max\{\log(K),\log(1/\delta)\}\ge(\log(K) +\log(1/\delta))/2$, we obtain the lower bound $\Omega(\sqrt{dT(\log(K)+\log(1/\delta))})$ with probability at least $\delta$, under the confidence and time horizon assumptions of Theorem~\ref{thm:lower-bound}.

\section{Related work}
Our work belongs to the literature on adversarial combinatorial bandits,
where, over $T$ rounds, a learner selects one of $K$ feasible subsets of
$d$ items and observes only the total loss of the selected subset.
In our setting, the feasible actions are all subsets of size $m$, so
$K=\binom{d}{m}$.

\citet{cesabianchi2012combinatorial} introduced ComBand for general
adversarial combinatorial bandits.
For finite action sets, \citet{bubeck2012towards} proposed EXP2 with
John's exploration, which achieves $O(\sqrt{dT\log K})$ pseudo-regret
against adaptive non-anticipating adversaries under our action-level
normalization. For oblivious adversaries, this also bounds expected
regret against the best action in hindsight.
Later, \citet{zimmert2022return} obtained the high-probability bound
$O(\sqrt{dT\log(K/\delta)})$ with probability at least $1-\delta$
against adaptive non-anticipating adversaries.
The direct implementations of these finite-action algorithms maintain
one weight per action and can therefore require exponential space in $d$.

In parallel, computational efficiency has been studied through structured
representations of the action set.
\citet{combes2015combinatorial} introduced CombEXP, which admits a
polynomial-time implementation for $m$-set actions.
For $m\le d/2$, item losses in $[0,1/m]$, and sufficiently large $T$,
it achieves $\widetilde O(\sqrt{dmT})$ expected regret against oblivious adversaries.
Subsequently, \citet{maiti2025efficient} gave a polynomial-time
algorithm whose guarantee yields $\widetilde O(m\sqrt{dT})$
high-probability regret against adaptive adversaries for $m\le d/2$.
In this paper, we improve this bound to $\widetilde O(\sqrt{dmT})$ while retaining
computational efficiency.

Under our action-level normalization, \citet{ito2019improved}
obtained an expected-regret lower bound of $\Omega(\sqrt{dmT})$
against oblivious adversaries for the full class of $m$-set actions
considered here, when $m\le d/2$ and $T=\Omega(dm^{3/2})$.
In this work, we prove a high-probability lower bound such that for every sufficiently small
$\delta$ and sufficiently large $T$, every randomized algorithm admits
an adaptive non-anticipating adversary for which regret is at least
$c\sqrt{dT\log(K/\delta)}$ with probability at least $\delta$,
where $c>0$ is a universal constant. However, the two lower bounds are not directly comparable, as theirs
concerns expected regret, whereas ours holds with a given probability.

For ordinary bandits, corresponding to $m=1$ and hence $K=d$,
\citet{gerchinovitz2016refined} prove a high-probability lower bound
of order $\sqrt{KT\log(1/\delta)}$ against oblivious adversaries.
Against adaptive non-anticipating adversaries, our lower bound
has order $\sqrt{KT\log(K/\delta)}$, matching the corresponding
upper bound and closing the $\sqrt{\log K}$ gap.

Finally, we combine the pseudoinverse estimate of
\citet[Corollary~1.2]{cesari2026sharp} with an identity for the
inverse second-moment matrix to derive an affine upper bound on
the leverage score for weighted $m$-set distributions.
This yields a surrogate loss that preserves this family under
exponential-weights updates.
An extended discussion appears in Appendix~\ref{sec:related-work}.
\section{An efficient minimax-optimal algorithm for \texorpdfstring{$m$}{m}-set bandits}
\label{sec:main-results}

\subsection{Affine upper bound on the leverage score}
\label{sec:affine-leverage}

For weighted $m$-set distributions $p\in\cP_\lambda$, in \cref{lem:linear-majorant-final} we derive an affine upper bound
on the leverage score $x^\top M(p)^{-1}x$ introduced in
\cref{eq:leverage-score} by combining an identity relating the inverse of the
second-moment matrix $M(p)$ to its covariance pseudoinverse with the pseudoinverse estimate of
\citet[Corollary~1.2]{cesari2026sharp}. The complete proof appears in
Appendix~\ref{app:structural-proofs}.

\begin{restatable}
{lemma}{linearKWMajorant}
\label{lem:linear-majorant-final}
Let $p\in\cP_\lambda$ be a weighted $m$-set distribution, where $1\le m\le d-1$, and let $\mu$, $M$, and $\Sigma$ be its marginal vector,
second-moment matrix, and covariance matrix, respectively.
For every $x\in\cX$, define
\[
    \phi_p(x)
:=
    1+\sum_{i=1}^d \frac{(x_i-\mu_i)^2}{\mu_i(1-\mu_i)}.
\]
Then, for every $x\in\cX$,
$
    x^\top M^{-1}x
\le
    \phi_p(x)
$.
Moreover,
$
    4\phi_p(x)
=
    c_p^\top x
$,
where
\begin{equation}
\label{eq:affine-correction-coefficients}
    c_{p,i}
:=
    4\left( \frac{ 1+\sum_{j=1}^d\frac{\mu_j}{1-\mu_j} }{m} + \frac{1-2\mu_i}{\mu_i(1-\mu_i)} \right)
\qquad
    \text{for every }i\in[d].
\end{equation}
Finally,
$
    \E_{X\sim p}[\phi_p(X)]
=
    d+1
$,
and, for every $x\in\cX$,
$
    0
\le
    \phi_p(x)
\le
    1+{2d}/{\lambda}
$.
\end{restatable}
\subsection{Algorithm and high-probability regret}
\label{sec:algorithm-regret}

In this section, we present \Cref{alg:exp3-kw} and provide its high-probability regret guarantee.
At each round $t$, the learner draws an action $X_t$ from a weighted $m$-set distribution $p_t=p_{\theta_t}\in\cP_{\lambda/2}$, represented by $\theta_t\in\R^d$, and
observes only the scalar loss $Y_t=X_t^\top \ell_t$.
Using this feedback, it forms the loss estimate $\widehat\ell_t=M_t^{-1}X_tY_t$, where $M_t=\E_{X\sim p_t}[XX^\top]$,
and can thus estimate, for every $x\in\cX$, the surrogate loss
$Z_t(x)=x^\top \widehat\ell_t-\eta c_{t}^\top x$
defined in \Cref{eq:biased-omd-loss}, where we write $c_t:=c_{p_t}$.
The distribution $p_t$ is then updated through an approximate version of online mirror descent (OMD) with the negative entropy regularizer $\sum_{x\in\cX}p(x)\log p(x)$, with $0\log0=0$.
The exact update would be
\begin{equation}
\label{eq:main-exact-omd-update}
    p_{t+1}^\star
    \in\argmin_{p\in\cP_\lambda}
    \left\{\eta\E_{X\sim p}[Z_t(X)]+\KL(p\|p_t)\right\}.
\end{equation}
Equivalently, this consists of the multiplicative-weights (MWU) step
\begin{equation}
\label{eq:main-mwu-update}
    \widetilde p_{t+1}(x)
    =\frac{p_t(x)e^{-\eta Z_t(x)}}
           {\sum_{y\in\cX}p_t(y)e^{-\eta Z_t(y)}},
    \qquad x\in\cX,
\end{equation}
followed by the exact KL projection onto $\cP_\lambda$, that is,
\begin{equation}
\label{eq:main-exact-kl-projection}
    p^\star_{t+1}
    \in\argmin_{p\in\cP_\lambda}
    \KL(p\|\widetilde p_{t+1}).
\end{equation}
Since $Z_t(x)=z_t^\top x$, where
$z_t:=\widehat\ell_t-\eta c_{p_t}$  and $p_t(x)=p_{\theta_t}(x)\propto e^{\theta^{\top}_t x}$, the MWU step in~\eqref{eq:main-mwu-update} preserves the
weighted $m$-set form. Specifically, $\widetilde{p}_{t+1}$ is a weighted m-set distribution described by the vector $\widetilde\theta_{t+1}:=\theta_t-\eta z_t.$

To implement the KL projection in~\eqref{eq:main-exact-kl-projection}, we use the deterministic procedure
described in Appendix~\ref{app:kl-projection}. 
At a high level, the procedure uses the ellipsoid method
\citep{groetschel1988geometric} to approximately solve a suitably defined convex
optimization problem.
The procedure eventually returns a vector
$\theta_{t+1}\in\R^d$ such that the resulting distribution
$p_{t+1}=p_{\theta_{t+1}}$ is a weighted $m$-set distribution that belongs to
$\cP_{\lambda/2}$ and satisfies, for every $q\in\cP_\lambda$, the following inequality
\begin{equation}
\label{eq:alg-approximate-projection}
    \KL(q\|p_{t+1})+\KL(p_{t+1}\|\widetilde p_{t+1})
    \le\KL(q\|\widetilde p_{t+1})+\varepsilon_{\rm p},
\end{equation}
where $\varepsilon_{\rm p}:=\eta/T$.
In Theorem~\ref{thm:main}, we provide the regret guarantees achieved by Algorithm~\ref{alg:exp3-kw}.
\begin{algorithm}[t]
\caption{}
\label{alg:exp3-kw}
\begin{algorithmic}[1]
\Require $T\ge1, \delta \in (0,1)$
\vspace{1mm}
\State Set $\eta
\gets\min \left \{
\frac{1}{256d},\sqrt{\frac{\log(12K/\delta)}
{320dT}} \right\}$ \label{line:alg-learning-rate}
\vspace{2mm}
\State Set $\lambda\gets128\eta d$ and
$\varepsilon_{\rm p}\gets\eta/T$ \label{line:alg-smoothing}
\State Initialize $\theta_1\gets\mathbf0$ and $p_1\gets p_{\theta_1}=U$
\Comment{{ $U$ is the uniform distribution}} \label{line:alg-initialize}

\For{$t=1,\ldots,T$} \label{line:alg-loop}
    \State Compute \label{line:alg-moments}
    $
    \displaystyle
    \mu_t\gets\mu(p_t),
    $ $
    M_t\gets M(p_t)
    =
    \E_{X\sim p_t}[XX^\top]
    $
    \State Compute $c_t\gets c_{p_t}$ using
    \eqref{eq:affine-correction-coefficients}
    \label{line:alg-affine}
    \State Draw $X_t\sim p_t$ \label{line:alg-sample}
    \State Observe $Y_t=X_t^\top \ell_t$ \label{line:alg-feedback}
    \vspace{1mm}
    \State Compute \label{line:alg-estimate}
    $
    \widehat{\ell}_t
    \gets
    M_t^{-1}X_tY_t
    $
    \vspace{2mm}
    \State Set $z_t\gets\widehat\ell_t-\eta c_t$,
    $\widetilde\theta_{t+1}\gets\theta_t-\eta z_t$, and
    $\widetilde p_{t+1}\gets p_{\widetilde\theta_{t+1}}$
    \label{line:alg-unprojected}
    \Comment{MWU update}
    \State Run the projection in
    Appendix~\ref{app:kl-projection} with input $(\widetilde\theta_{t+1},\eta,T)$ to obtain $\theta_{t+1}$
    \label{line:alg-kl-projection}
    \State Set $p_{t+1}\gets p_{\theta_{t+1}}$
    \label{line:alg-projection}
\EndFor
\end{algorithmic}
\end{algorithm}

\begin{restatable}
{theorem}{mainRegretTheorem}
\label{thm:main}
There exists a universal constant $C>0$ such that the following holds.
Let $T\ge1$ be an integer, let $\delta\in(0,1)$, and assume
$1\le m\le d-1$.
Run \Cref{alg:exp3-kw}.
Then, with probability at least $1-\delta$,
\begin{equation}
\label{eq:main-regret-guarantee}
    R_T
\le
    C\sqrt{ dT\left( \log K+\log\frac1\delta \right) },
\end{equation}
where $K=\binom{d}{m}$.
One may take $C=160$.
\end{restatable}
\begin{proofsketch}
For every $x\in\cX$, let $q_x:=(1-\lambda)\delta_x+\lambda U\in\cP_\lambda$,
where $\delta_x$ is the Dirac distribution on action $x$. We establish the bounds simultaneously for all $x$, so we
can subsequently choose a best action in hindsight. In the following, we denote $\widehat y_t(x):=x^\top \widehat\ell_t$, where $\widehat\ell_t$ is defined in \eqref{eq:kw-loss-estimate}. Since the matrix $M_t$ is symmetric and invertible and satisfies $M_t\mathbf1=m\mu_t$, we have
\begin{equation*}
    \E_{X\sim p_t}[\widehat y_t(X)]
    =\mu_t^\top M_t^{-1}X_tY_t
    =\frac{\mathbf1^\top X_t}{m}Y_t=Y_t.
\end{equation*}
Thus the learner's average estimated loss equals its realized loss.
Since action losses belong to $[-1,1]$, it also holds
\begin{equation}
\label{eq:main-regret-decomposition}
    R_T (x) {:=R_T(\delta_x)} = \sum_{t=1}^T\bigl(Y_t- {x^\top \ell_t}\bigr)
    \le
    \underbrace{\sum_{t=1}^T
    \left(Y_t-\E_{X\sim q_x}[{X^\top \ell_t}]\right)}_{=:(\star)}
    +2\lambda T.
\end{equation}
Adding and subtracting $\E_{X\sim q_x}[\widehat y_t(X)]$ gives
\begin{align}
\label{eq:main-central-decomposition}
    (\star)
    =\sum_{t=1}^T\left(
    \E_{X\sim p_t}[\widehat y_t(X)]
    -\E_{X\sim q_x}[\widehat y_t(X)]\right)+\sum_{t=1}^T
    \E_{X\sim q_x}[\widehat y_t(X)-X^\top \ell_t].
\end{align}
We now bound the two terms in \eqref{eq:main-central-decomposition}.
\begin{enumerate}
	\item For the first term in \eqref{eq:main-central-decomposition}, we use $\widehat y_t(X)=Z_t(X)+\eta c_t^\top X$ and
	$\KL(q_x\|p_1)\le\log K$.
	An exponential-weights argument, as in \citet{zimmert2022return},
	combined with the approximate KL projection guarantee in
	\eqref{eq:alg-approximate-projection}
	(see \Cref{lem:omd-stability}), gives
	\begin{equation*}
		\begin{aligned}
			&\sum_{t=1}^T\left(
			\E_{X\sim p_t}[\widehat y_t(X)]
			-\E_{X\sim q_x}[\widehat y_t(X)]\right)\\
			&\quad\le\frac{\log K}{\eta}
			+2\eta\sum_{t=1}^T\E_{X\sim p_t}[Z_t(X)^2]
			+\eta\sum_{t=1}^T\E_{X\sim p_t}[c_t^\top X]
			-\eta\sum_{t=1}^T\E_{X\sim q_x}[c_t^\top X] + O(1),
		\end{aligned}
	\end{equation*}
	where the term $O(1)$ accounts for the accumulated projection error $\varepsilon_{\rm p}$.
	To control the second moment of $Z_t(X)$, we observe that
	\(
		\E_{X\sim p_t}[Z_t(X)^2]
		\le 2\E_{X\sim p_t}[\widehat y_t(X)^2]
		+2\eta^2\E_{X\sim p_t}[(c_t^\top X)^2].
	\)
	Thanks to \Cref{lem:linear-majorant-final} and the fact that distributions $p_t$ belong to $\mathcal{P}_{\lambda/2}$ at every round $t \in [T]$ (see Proposition~\ref{prop:efficient-implementation}), we have $\E_{X\sim p_t}[c_t^\top X]=4(d+1)$ and
	$\max_{x\in\cX}c_t^\top x=O(d/\lambda).$ Thus, the following holds:
	\begin{equation*}
			\E_{X\sim p_t}[\widehat y_t(X)^2]
			=Y_t^2X_t^\top M_t^{-1}X_t\le\phi_{p_t}(X_t),\quad
			\E_{X\sim p_t}[(c_t^\top X)^2]
			\le O(d^2/\lambda).
	\end{equation*}
	Using the above observations and a concentration argument over the quantity $\sum_{t=1}^T\phi_{p_t}(X_t)$
	(see \Cref{lem:leverage}) gives,
	with probability at least $1-\delta/2$,
	the following:
	\begin{equation*}
		\begin{aligned}
			2\eta\sum_{t=1}^T\E_{X\sim p_t}[Z_t(X)^2]
			\le O\left(\eta dT+\frac{\eta d}{\lambda}\log\frac2\delta
			+\frac{\eta^3d^2T}{\lambda}\right)
			=O\left(\eta dT+\log\frac 1\delta\right),
		\end{aligned}
	\end{equation*}
	where the last step uses $\lambda=\Theta(\eta d)$ and $\eta\le1$.
	Moreover, $\log(2/\delta)=\log(1/\delta)+\log2$, and
	$\eta dT\ge1/256$ allows us to absorb the constant $\log2$
	into $O(\eta dT)$.
	Therefore, we have
\begin{equation}
	\label{eq:main-sketch-estimated-loss}
	\begin{aligned}
		\sum_{t=1}^T\left(
		\E_{X\sim p_t}[\widehat y_t(X)]
		-\E_{X\sim q_x}[\widehat y_t(X)]\right)
		\le\frac{\log K}{\eta}
		-\eta\sum_{t=1}^T\E_{X\sim q_x}[c_t^\top X] + O\left(\eta dT+\log\frac1\delta\right).
	\end{aligned}
\end{equation}
	\item  For the second term in \eqref{eq:main-central-decomposition}, a concentration
	bound (see \Cref{lem:kw-concentration})
	with failure probability $\delta/(2K)$, followed by a union bound, gives
	simultaneously for every $x\in\cX$, with probability at least $1-\delta/2$,
	\begin{equation}
		\label{eq:main-sketch-comparator-error}
		\sum_{t=1}^T
		\E_{X\sim q_x}[\widehat y_t(X)- {X^\top \ell_t}]
		\le\frac\eta4\sum_{t=1}^T\E_{X\sim q_x}[c_t^\top X]
		+\frac{\log(4K/\delta)}{\eta}.
	\end{equation}
\end{enumerate}
Combining \eqref{eq:main-sketch-estimated-loss} and
\eqref{eq:main-sketch-comparator-error} in \eqref{eq:main-central-decomposition} yields, with probability
at least $1-\delta$, simultaneously for all $x\in\cX$,
\begin{equation}
\label{eq:main-central-bound}
    (\star)
    \le O\left(\frac{\log K+\log(1/\delta)}{\eta}+\eta dT+1\right)
    +\underbrace{\left(-\eta+\frac\eta4\right)
    \sum_{t=1}^T\E_{X\sim q_x}[c_t^\top X]}_{\le0}.
\end{equation}
Taking $x$ to be a best action in hindsight in
\eqref{eq:main-regret-decomposition}, applying
\eqref{eq:main-central-bound}, and substituting the algorithm's
parameters gives the claimed rate.
When $\eta=1/(256d)$, we obtain the claimed rate using $R_T\le2T$.
The full proof is provided in Appendix~\ref{app:proof-main}.
\end{proofsketch}

We observe that, since
$
    \log K
\le
    s\log\left({ed}/{s}\right),
$
the bound in \Cref{thm:main} is $\widetilde O(\sqrt{dsT})$, where $s=\min\{m,d-m\}$.
In particular, when $m\le d/2$, it is $\widetilde O(\sqrt{dmT})$.

\subsection{Polynomial-time implementation}
\label{sec:implementation}

We now show that \Cref{alg:exp3-kw} can be implemented in polynomial
time and space without enumerating $\cX$, and that every sampling
distribution $p_t$ is a weighted $m$-set distribution in
$\cP_{\lambda/2}$.\footnote{In this paper, we focus on the arithmetic
	model, in which each arithmetic operation on real numbers and each
	evaluation of $\exp$ or $\log$ has unit cost, independently of the
	precision needed to represent the numbers.
	{We also allow exact comparisons of real numbers and independent samples
	from $\operatorname{Unif}[0,1]$ at unit cost.}}

\begin{restatable}{proposition}{efficientimplementation}
\label{prop:efficient-implementation}
{For every $t\in \{2,\dots,T+1\}$, \Cref{alg:exp3-kw} maintains
a weighted $m$-set distribution $p_t=p_{\theta_t}\in\cP_{\lambda/2}$,
represented by $\theta_t\in\R^d$ satisfying
Equation~\eqref{eq:alg-approximate-projection} at round $t-1$.
Moreover, the algorithm can be implemented in
$T\cdot\poly(d,m,\log T)$ time and $\poly(d,m)$ space,
without enumerating $\cX$.}
\end{restatable}

\begin{proofsketch}
We first show by induction that every distribution $p_t=p_{\theta_t}$ chosen by our algorithm is a weighted
$m$-set distribution belonging to $\cP_{\lambda/2}$.
Initially, $\theta_1=\mathbf0$, so
$p_1=U\in\cP_\lambda\subseteq\cP_{\lambda/2}$.
Suppose the claim holds at round $t$.
Since the surrogate loss is linear, $Z_t(x)=z_t^\top x$, the MWU update
preserves the weighted $m$-set structure. Thus, $\widetilde p_{t+1}=p_{\widetilde\theta_{t+1}}$, where
$\widetilde\theta_{t+1}=\theta_t-\eta z_t$.

The challenging step is to show that the approximate KL projection at
line~\ref{line:alg-kl-projection} of Algorithm~\ref{alg:exp3-kw} also preserves
the weighted $m$-set structure and returns a distribution in $\cP_{\lambda/2}$.
To achieve this, we formulate a convex optimization problem in $d$ variables.
Specifically, in Appendix~\ref{app:kl-projection}, we formulate the exact projection
onto $\cP_\lambda$ through the minimization, with respect to $v$, of
\begin{equation}
\label{eq:proof-sketch-projection-objective}
    G_{\widetilde\theta_{t+1}}(v)
    =F(\widetilde\theta_{t+1}+v)-{\underline{\mu}}^\top v_++{\overline{\mu}}^\top v_-,
\end{equation}
where $F(\theta):=\log\sum_{x\in\cX}\exp(\theta^\top x)$ is the log-partition function,
	and ${\underline{\mu}},{\overline{\mu}}\in\R^d$ have coordinates ${\underline{\mu}}_i=\lambda r$ and ${\overline{\mu}}_i=1-\lambda(1-r)$,
	with $r=m/d$. Moreover, $(v_+)_i=\max\{v_i,0\}$ and $(v_-)_i=\max\{-v_i,0\}$.
Using the ellipsoid method, we show how to approximately minimize \eqref{eq:proof-sketch-projection-objective} over a suitably bounded domain and obtain a vector $v$. Then the procedure returns $\theta_{t+1}=\widetilde\theta_{t+1}+v$.

As established in Appendix~\ref{app:implementation}, the resulting weighted
$m$-set distribution $p_{t+1}=p_{\theta_{t+1}}$ belongs to
$\cP_{\lambda/2}$ and satisfies~\eqref{eq:alg-approximate-projection}
with error $\varepsilon_{\rm p}=\eta/T$.
This closes the induction for all $t$.

It remains to show that our algorithm can be implemented efficiently. The methods of
\citet[Sections~2 and~4]{chen1997statistical}, recalled in
Appendix~\ref{app:moment-sampling}, compute the second-moment matrix $M(p_t)$, the marginal vector $\mu(p_t)$, and sample
from $p_{\theta_t}$ in $\poly(d,m)$ time and space, without enumerating
$\cX$.
Computing the loss estimate $\widehat\ell_t$ requires solving a $d$-dimensional linear system,
while the MWU update requires only updating the $d$ parameters.
For the projection step, which amounts to minimizing the objective in
\eqref{eq:proof-sketch-projection-objective}
over a suitable domain, we observe that this objective is convex and that its value and a subgradient\footnote{For a convex function $f$, a subgradient at $v$ is a vector $\zeta$ such that $f(u)\ge f(v)+\zeta^\top(u-v)$ for every $u$.} can likewise be
evaluated in $\poly(d,m)$ time.
The ellipsoid method, as discussed in
Appendix~\ref{app:kl-projection}, therefore computes the required approximate
solution in $\poly(d,m,\log T)$ arithmetic operations and $\poly(d,m)$ space.
By storing the required quantities only for the current round and reusing the workspace, we can prove that our algorithm requires
time $T\cdot\poly(d,m,\log T)$ and space $\poly(d,m)$.
\end{proofsketch}

\subsection{Matching high-probability lower bound}
\label{sec:lower-bound}

We complement the upper bound with a matching high-probability lower bound.

\begin{restatable}
{theorem}{matchingLowerBound}
\label{thm:lower-bound}
There exist universal constants $c>0$ and $\delta_0\in(0,1)$ and a threshold $\mathfrak T(d,m,\delta)$ bounded by a polynomial in $d$, $m$, and $\log(1/\delta)$ such that the following holds.
Let $d\ge2$, $1\le m\le d-1$, $\delta\in(0,\delta_0]$, and $T\ge\mathfrak T(d,m,\delta)$.
For every randomized algorithm on $\cX_{d,m}$, there exists an adaptive non-anticipating adversary whose action losses belong to $[0,1]$ and such that
\[
    \Pp \left(
    R_T\ge
    c\sqrt{dT\left(\log\binom dm+\log\frac1\delta\right)}
    \right)\ge\delta.
\]
\end{restatable}

\begin{proofsketch}
To develop our lower bound, we work on $\cX_{d,m}$ with $m\le d/2$.
We do this without loss of generality since, by
\Cref{lem:lower-complement}, we can transfer the lower bound
obtained for $m\le d/2$ to $\cX_{d,m}$ with $m>d/2$,
with the same regret and probability guarantees.

In addition, we temporarily allow randomized non-anticipating
adversaries. Given $\mathcal F_{t-1}$, the adversary may use its
own randomness to choose $\ell_t$ without observing the learner's
current action $X_t$. This auxiliary randomness is sampled
independently of the learner's private randomness. This is without loss of generality, as \Cref{lem:lower-fixing-table}
allows us to transfer any high-probability lower bound obtained against
a randomized adversary to a deterministic adversary of the type introduced in \cref{sec:setting},
while preserving the same guarantee.

We divide the analysis into three cases. The first two establish
the $\log K$ contribution, according to whether $r\le r_\star$
or $r>r_\star$, where $r_\star:=2^{-32}$, and the third establishes the $\log(1/\delta)$
contribution.
In all three cases, we show that the scalar feedback $Y_t$
observed by the algorithm, once the instance, the past history,
and the current action $X_t$ are fixed, is of the form
\(
Y_t=a_t+b_t{\psi}_{u_t}(Z_t),
\)
where $Z_t\sim\operatorname{Unif}[-1,1]$ is fresh independent
noise, and
\[
g(z):=1-z^2,
\qquad {\psi}_u(z):=z/2+u g(z),
\qquad z\in[-1,1].
\]
Here, $u_t$ is a fixed parameter under this conditioning,
as are $a_t$ and $b_t>0$.  This allows us to compare the feedback distributions in two
instances by comparing their values of $u_t$. Writing $p_u$ for the density of ${\psi}_u(Z_t)$,
\Cref{lem:lower-feedback-density} gives
\(
\KL(p_u\|p_v)\le80(u-v)^2
\) for \( |u|,|v|\le1/16.\)

Since the common shift $a_t$ and scaling $b_t>0$ are fixed
conditional on the past history and current action $X_t$,
they do not change KL divergence. Hence, the same bound holds for the feedback distributions
conditional on the same quantities.
\begin{itemize}[leftmargin=1.5em]
\item \textbf{The $\Omega(\sqrt{dT\log K})$ lower bound.}
We distinguish two cases, with $r_\star:=2^{-32}$.
\begin{itemize}[leftmargin=1.5em]
\item \textbf{The case $r\le r_\star$.}
To build the instances, we independently draw indicator variables $U_i\sim\operatorname{Ber}(\rho)$ and
noise variables $Z_t\sim\operatorname{Unif}[-1,1]$, independently of the algorithm's
randomness, with $\rho:=1/256$ and $N_0:=\lceil8Tr\rceil$.
Let $A_{t,i}:=\mathbf1\{\sum_{\tau<t}X_{\tau,i}<N_0\}$.
We say that item $i$ is active at round $t$ if it has been selected
fewer than $N_0$ times before round $t$, equivalently, if $A_{t,i}=1$.
Choose $\varepsilon:=\Theta(m\sqrt{\log(1/r)/{N_0}})$, in a way that $\varepsilon\le1/16$, and define
\begin{equation}
\label{eq:lower-sketch-loss}
    \ell_{t,i}:=\frac{1+Z_t}{2m}
    -\frac{\varepsilon}{m}g(Z_t)A_{t,i}U_i.
\end{equation}
We observe that item $i$ receives a loss discount when $U_i=1$ and it is
still active. Thus, when $U_i=1$ we say that item $i$ is favorable. Here, the adversary is adaptive because it removes this
discount once the algorithm has selected the item $N_0$ times.
We also introduce the following scores
\begin{equation}
\label{eq:lower-sketch-scores}
    S_i:=U_i\sum_{t=1}^T A_{t,i}g(Z_t),
    \qquad
    S_{\rm L}:=\sum_{t=1}^T g(Z_t)
    \sum_{i=1}^d A_{t,i}X_{t,i}U_i.
\end{equation}
Since ${(1+Z_t)}/{(2m)}$ is common to all items, $(\varepsilon/m)S_i$ represents the item $i$'s total discount, while
$(\varepsilon/m)S_{\rm L}$ is the discount collected by the algorithm.
The best fixed action selects the $m$ largest scores $S_i$, whose sum we
denote by $\operatorname{Top}_m(S)$, giving
\begin{equation}
\label{eq:lower-sketch-regret}
    R_T=\frac{\varepsilon}{m}\bigl(\operatorname{Top}_m(S)-S_{\rm L}\bigr).
\end{equation}

Each item contributes to $S_{\rm L}$ at most $N_0$ times, since, by the definition of $A_{t,i}$, it becomes inactive after its $N_0$-th selection.
Since $g\le1$, we have $S_{\rm L}\le N_0\sum_iU_i$, hence
\(\E[S_{\rm L}]\le (N_0d)/{256}\le (9mT)/{256}, \)
where the last inequality uses $N_0\le9Tr$.

We now want to lower bound $\E[\operatorname{Top}_m(S)]$. A favorable item $i$ (that is, one with $U_i=1$) that is still active at the start of round $T$ must have been active throughout all $T$ rounds, and therefore has score $S_i=\sum_{t=1}^T g(Z_t)$. We will show that at least $m$
favorable items remain active with probability at least $15/16$.      

Fix the algorithm's private randomness and all the $U_j$ with $j \neq i$,
and let $P_i^0,P_i^1$ be the history laws for $U_i=0,1$.\footnote{As formalized in Appendix~\ref{app:lower-bound},
we denote by $\cH_t=(X_1,Y_1,\ldots,X_t,Y_t)$ the history of
 actions and observed feedback up to round $t$, where $Y_t=X_t^\top\ell_t$.}
At each round $t$, given the same past history, the action $X_t$ and
the active items are the same under both laws.
The feedback in both instances is
\[
Y_t=\frac12+\psi_{u_t}(Z_t)
\quad\text{and}\quad
u_t=-\frac{\varepsilon}{m}\sum_j A_{t,j}X_{t,j}U_j ,
\]
with $U_i=0$ under $P_i^0$ and $U_i=1$ under $P_i^1$. In addition, in both cases, $|u_t|\le\varepsilon\le1/16$, so $\psi_{u_t}$ maps $[-1,1]$ onto $[-1/2,1/2]$ and the action losses lie in $[0,1]$. For the same past history and current action $X_t$, the two values
of $u_t$ are uniquely determined and differ in absolute value
by $\varepsilon/m$ when item $i$ is selected while active, that is,
when $X_{t,i}=A_{t,i}=1$, and coincide otherwise. Thus, \Cref{lem:lower-feedback-density} bounds each such round's KL contribution by $80\varepsilon^2/m^2$.

In addition, there are at most $N_0$ such rounds, so the KL chain rule of \citet[Exercise~14.12, Eq.~(14.17)]{lattimore2020bandit} gives
\begin{equation}
\label{eq:lower-sketch-history-kl}
    \KL(P_i^0\|P_i^1)\le\frac{80\varepsilon^2N_0}{m^2}=O(\log(1/r)).
\end{equation}

Let $N_{\mathrm{act}}:=\sum_iU_iA_{T,i}$ be the number of favorable active items
at the start of round $T$. Then $\Pp(N_{\mathrm{act}}<m)$ is the probability that fewer than $m$ such items remain. For each item $i$, we introduce the event
$H_i:=\{A_{T,i}=1,\ \sum_{j\ne i}U_jA_{T,j}<m\}$, and we write
\[
\bar h_0:=\frac1d\sum_i\Pp(H_i\mid U_i=0),
\qquad \bar h_1:=\frac1d\sum_i\Pp(H_i\mid U_i=1).
\]

An item becomes inactive only after being selected
$N_0=\lceil8Tm/d\rceil$ times. Before round $T$, the algorithm has made
$m(T-1)$ item selections in total. Thus, the number of inactive items
is at most $m(T-1)/N_0\le d/8$.
Therefore, at least $d-d/8=7d/8$ items remain active.
This holds for every realization of the randomness of the algorithm.
For an unfavorable item $i$, we have 
\(
    \{U_i=0\}\cap H_i
    =\{U_i=0,\ A_{T,i}=1,\ N_{\mathrm{act}}<m\}.
\)
Recalling the definitions of $H_i$ and $N_{\mathrm{act}}$, we obtain
\(
    (1-U_i)\mathbf1_{H_i}
    =(1-U_i)A_{T,i}\mathbf1_{\{N_{\mathrm{act}}<m\}}.
\)
Summing over items gives
\[
\begin{aligned}
    \sum_{i=1}^d(1-U_i)\mathbf1_{H_i}
    =\left(\sum_{i=1}^d A_{T,i}-N_{\mathrm{act}}\right)
      \mathbf1_{\{N_{\mathrm{act}}<m\}} \ge (7d/8-m)\mathbf1_{\{N_{\mathrm{act}}<m\}},
\end{aligned}
\]
where the inequality uses the preceding bound on the number of active items and $N_{\mathrm{act}}<m$ on the indicated event.
Taking expectations and using $\Pp(U_i=0)=1-\rho$ gives
\[
\begin{aligned}
    (1-\rho)d\bar h_0
    =\sum_{i=1}^d\Pp(H_i,U_i=0)=\E\!\left[\sum_{i=1}^d(1-U_i)\mathbf1_{H_i}\right]
    \ge (7d/8-m)\Pp(N_{\mathrm{act}}<m).
\end{aligned}
\]
For a favorable active item $i$, with a similar argument, we obtain
\begin{equation}
	\rho d\bar h_1=\sum_{i=1}^d\mathbb P(H_i, U_i=1)=\mathbb E\!\left[\sum_{i=1}^d U_i\mathbf1_{H_i}\right]
	=\mathbb E\!\left[N_{\mathrm{act}}\mathbf1_{\{N_{\mathrm{act}}\le m\}}\right]
	\le m{.}
\end{equation}
Hence $\bar h_0\ge3\Pp(N_{\mathrm{act}}<m)/4$ and $\bar h_1\le256r$.
With the algorithm's private randomness and all $U_j$ for $j\ne i$
fixed, $H_i$ is determined by the observed history. Applying data
processing to this event gives
\[
\KL\!\left(\operatorname{Ber}(P_i^0(H_i))
\,\middle\|\,\operatorname{Ber}(P_i^1(H_i))\right)
\le \KL(P_i^0\|P_i^1).
\]
Using \eqref{eq:lower-sketch-history-kl}, averaging over items, the algorithm's randomness, and the remaining
variables $U_j$ with $j\ne i$, and using the joint convexity of
KL divergence, we obtain
\(
\KL\!\left(\operatorname{Ber}(\bar h_0)
\,\middle\|\,\operatorname{Ber}(\bar h_1)\right)
\le {80\varepsilon^2N_0}/{m^2}.
\)
If $\bar h_1=0$, the claim follows
	immediately. Otherwise, the binary-KL inequality gives
\[
\KL\!\left(\operatorname{Ber}(\bar h_0)
\,\middle\|\,\operatorname{Ber}(\bar h_1)\right)
\ge \bar h_0\log(1/\bar h_1)-\log2.
\]

Combining these inequalities with the bounds on $\bar h_0$ and
$\bar h_1$ and $r\le r_\star$, and choosing $\varepsilon$ sufficiently small, we obtain
\(
\Pp(N_{\mathrm{act}}<m)\le 1/{16}.
\)

On $\{N_{\mathrm{act}}\ge m\}$, the comparator can select $m$ favorable items
that stay active throughout all $T$ rounds, so
$\operatorname{Top}_m(S)=m\sum_tg(Z_t)$.
In addition, we have
\[
\begin{aligned}
	\E[\operatorname{Top}_m(S)]
	&\ge m\E\!\left[\left(\sum_t g(Z_t)\right)\mathbf1_{\{N_{\mathrm{act}}\ge m\}}\right]\\
	&=m\E\!\left[\sum_t g(Z_t)\right]
	-m\E\!\left[\left(\sum_t g(Z_t)\right)\mathbf1_{\{N_{\mathrm{act}}<m\}}\right] \ge\left(\frac23-\frac1{16}\right)mT.
\end{aligned}
\]
This uses nonnegative scores, $\sum_t g(Z_t)\le T$, and the bound $\Pp(N_{\mathrm{act}}<m)\le1/16$; it does not require independence between this event and the noise. The full argument appears in Appendix~\ref{app:lower-combinatorial}.
Together, the two score bounds give $\E[R_T]\ge\varepsilon T/4$.
Since $R_T\le\varepsilon T$, \Cref{lem:lower-mean-to-probability} gives
$\Pp(R_T\ge\varepsilon T/8)\ge1/8$.
Our choice of $\varepsilon$ gives
$\varepsilon T=\Theta(\sqrt{dmT\log(d/m)})=\Theta(\sqrt{dT\log K})$.

\item \textbf{The case $r>r_\star$.}
We consider the same losses as in
Equation~\eqref{eq:lower-sketch-loss}, with the only difference being
that we now set $A_{t,i}=1$ for every $t\in[T]$ and take the sign variables
$U_i$ to be uniformly distributed over $\{-1,1\}$ for $i\in[2m]$,
with $U_i=0$ otherwise.
These signs are independent of one another, the noises, and the algorithm's private randomness. We also set $\varepsilon:=\Theta(m/\sqrt T)$, {in a way that $\varepsilon\le1/16$ and the KL bound in \eqref{eq:large-rstar-kl-bound} holds}.
The scores in \eqref{eq:lower-sketch-scores} have the same interpretation
as in the case $r\le r_\star$, and the regret identity
\eqref{eq:lower-sketch-regret} remains valid. Notice that, since the losses no longer depend on the algorithm's past actions,
the adversary is oblivious.

To lower bound $\E[\operatorname{Top}_m(S)]$ we arrange the first $2m$ items into pairs
$(1,2),\ldots,(2m-1,2m)$, and select the item $i$ with the larger score $S_i$
from each pair. This is a valid fixed action with $m$ items, so it can be used to lower bound $\operatorname{Top}_m(S)$.
Then, we have
\begin{equation*}
	\E[\operatorname{Top}_m(S)]
	\ge \E\!\left[\sum_{j=1}^m\max\{S_{2j-1},S_{2j}\}\right]
	= \E\!\left[\sum_{t=1}^Tg(Z_t)\right]
	\sum_{j=1}^m\E[\max\{U_{2j-1},U_{2j}\}]
	=\frac{mT}{3},
\end{equation*}
where the first equality uses $g(Z_t)\ge0$ and independence.
The second uses $\E[g(Z_t)]=2/3$ for every $t$ and
$\E[\max\{U_{2j-1},U_{2j}\}]=1/2$.

To derive an upper bound on the expected score of the algorithm, we fix the randomness of the algorithm and all signs except $U_i$, and let
$P_i^+,P_i^-$ be the history laws for $U_i=1,-1$, respectively.
As in the previous case, we compare the feedback laws conditional
on the same past history, with the algorithm's private randomness
and all signs other than $U_i$ fixed. Changing $U_i$ from $-1$ to $1$
changes the parameter of ${{\psi}}_u$ by $2{{\varepsilon}}/m$ in absolute value
when item $i$ is selected, and leaves it unchanged otherwise.
Thus, \Cref{lem:lower-feedback-density} bounds the KL contribution
of each round by $80(2{{\varepsilon}}/m)^2X_{t,i}$.
Summing these contributions using the KL chain rule gives
\begin{equation}\label{eq:large-rstar-kl-bound}
\KL(P_i^+\|P_i^-)
{{\le}}80\left(\frac{2{{\varepsilon}}}{m}\right)^2
\E_{P_i^+}\!\left[\sum_{t=1}^T X_{t,i}\right]
\le\frac1{32}.
\end{equation}
Furthermore, using Pinsker's inequality and averaging over the
algorithm's randomness and the signs $U_i$, we obtain
\(
\E [U_i\sum_{t=1}^T X_{t,i}]\le{T}/{16}.
\)
Moreover, the fresh noise $Z_t$ is independent of $X_t$ and the signs, and $\E[g(Z_t)]=2/3$. Consequently, using the definition of $S_{\rm L}$ and the previous bound, we obtain
\[
    \E[S_{\rm L}]
    =\frac23\sum_{i=1}^{2m}
       \E\!\left[U_i\sum_{t=1}^T X_{t,i}\right]
    \le\frac23\cdot2m\cdot\frac T{16}=\frac{mT}{12}.
\]
Using the lower bound on $\E[\operatorname{Top}_m(S)]$, the upper
bound on $\E[S_{\rm L}]$, and \Cref{eq:lower-sketch-regret}, we obtain
$\E[R_T]\ge{{\varepsilon}} T/4$. {Since both scores lie in $[-mT,mT]$, we have $R_T\le2\varepsilon T$.} \Cref{lem:lower-mean-to-probability} then gives
\(
\Pp(R_T\ge{{\varepsilon}} T/8)\ge1/16.
\)
To conclude the proof, we observe that
$ \varepsilon T=\Theta(m\sqrt T)=\Omega(\sqrt{dT\log K}),$ because $m>r_\star d$ and $\log K\le d\log2$.
\end{itemize}

\item \textbf{The $\Omega(\sqrt{dT\log(1/\delta)})$ lower bound.}
We use here a different construction. Specifically, we fix $m-1$ core items $C$ and let $V:=[d]\setminus C$.
We sample independent noises $Z_t\sim\operatorname{Unif}[-1,1]$ and set
$\Delta:=\Theta(\sqrt{|V|\log(1/\delta)/T})$, with a sufficiently small numerical constant. The horizon assumption then ensures $m\Delta<1/8$ and $\Delta\le1/64$. We define a reference instance $0$ with losses $\ell_t^{(0)}$
and, for each item $i\in V$ outside the core, an instance $i$
with losses $\ell_t^{(i)}$ as follows:
\[
    \ell_{t,j}^{(0)}:=\frac{1/4+Z_t/8}{m}
        +\frac{\mathbf1\{j\in V\}}{8m},
    \qquad
    \ell_{t,j}^{(i)}:=\ell_{t,j}^{(0)}
        -\Delta g(Z_t)\mathbf1\{j=i\}.
\]
Since $m\Delta<1/8$, in every instance and at every round each core item has a lower loss than every item in $V$. In instance $i\in V$, item $i$ has its loss reduced by $\Delta g(Z_t)$, making it no larger than that of any other item outside the core. Therefore, since $|C|=m-1$, the action selecting $C\cup\{i\}$ is optimal
in instance $i$. We also introduce the following quantities
\[
    D_t:=\frac1{8m}\sum_{\tau=1}^t
        \left(\sum_{j\in V}X_{\tau,j}-1\right),
    \qquad
    R_T^{(i)}=D_T+\Delta\sum_{t=1}^Tg(Z_t)(1-X_{t,i}).
\]
In instance $i$, the term $D_t$ accounts for the regret accumulated
up to round $t$ by selecting more than one item outside the core
per round. The term $R_T^{(i)}$ represents the cumulative regret
relative to always playing $C\cup\{i\}$ and is therefore the sum
of $D_T$ and the cumulative additional loss incurred by not
selecting item $i$.

We construct an auxiliary algorithm that follows the original until its
next proposed action would make $D_t\ge\Delta T/8$. It then selects
the core items and one fixed item of $V$, chosen identically in every instance, for all remaining rounds.
Run both algorithms with the same internal randomness and noises $Z_t$.
If the two algorithms first recommend different actions at some
round $t$, the original algorithm's action makes $D_t\ge\Delta T/8$,
and hence $R_T^{(i)}\ge\Delta T/8$.
If their behavior coincides through round $T$, they incur the same regret.
Thus, any lower bound on the probability that the auxiliary algorithm
incurs regret at least $\Delta T/8$ also holds for the original algorithm.

For ease of presentation, in the rest of the proof, $X_t$ and $D_t$
refer to the auxiliary algorithm. We now fix its internal randomness
and let $P,Q_i$ denote the laws of its observed feedback history
in the reference instance and instance $i$, respectively.
In instance $i$, the feedback is
\[
Y_t=\frac14+\frac1{8m}\sum_{j\in V}X_{t,j}
       +\frac14 {{\psi}}_{-4\Delta X_{t,i}}(Z_t),
\]
while the reference instance has the same expression with
${{\psi}}_0(Z_t)$ instead of ${{\psi}}_{-4\Delta X_{t,i}}(Z_t)$.
For a fixed history and action, subtracting the first two terms and
multiplying by four gives a common invertible transformation.
Thus, by \Cref{lem:lower-feedback-density}, the KL contribution is at
most $80(4\Delta)^2$ when item $i$ is selected and zero otherwise.

Applying the KL chain rule and averaging over $i\in V$ gives
\[
    \frac1{|V|}\sum_{i\in V}\KL(P\|Q_i)
    \le {\frac{80(4\Delta)^2}{|V|}}
       \E_P\!\left[\sum_{t=1}^T\sum_{i\in V}X_{t,i}\right]
    \le{O\!\left(\frac{\Delta^2T}{|V|}\right)},
\]
where the second inequality uses $D_T<\Delta T/8$. After replacing the proposed action, the auxiliary algorithm selects exactly one item outside the core, so $D_t$ no longer increases.
 Indeed, the definition of $D_T$ gives \( \sum_{t=1}^T\sum_{j\in V}X_{t,j}
=T+8mD_T<T(1+m\Delta)<2T.\)

For each item $i \in V$, we consider the event
\(
    E_i:=\left\{D_T+\Delta\sum_{t=1}^Tg(Z_t)(1-X_{t,i})
                \ge\Delta T/8\right\}.
\)
As shown in Appendix~\ref{app:lower-confidence}, $E_i$ is a measurable event of the observed history, with the same definition under both $P$ and $Q_i$.

Furthermore, $Q_i(E_i)\ge\delta$ means that, in instance $i$, the auxiliary algorithm incurs regret at least $\Delta T/8$ with probability at least $\delta$.
Under $P$, using the definition of $D_T$, $|V|\ge2$, $0\le g\le1$, and $\Delta/|V|\le1/(8m)$, we obtain
\[
\frac1{|V|}\sum_{i\in V}
\left(D_T+\Delta\sum_{t=1}^T g(Z_t)(1-X_{t,i})\right)
\ge\frac{\Delta}{2}\sum_{t=1}^T g(Z_t).
\]
We observe that $D_T+\Delta\sum_{t=1}^T g(Z_t)(1-X_{t,i})$ is at most $9\Delta T/8$ on $E_i$
and below $\Delta T/8$ on its complement $E_i^c$.
Taking expectations, using $\E_P[g(Z_t)]=2/3$, and splitting
each expectation over $E_i$ and $E_i^c$, we obtain, after rearranging,
\(
\sum_{i\in V}P(E_i)\ge|V|/{16}.
\)

Applying data processing to these events, namely
\(
\KL\!\left(
\operatorname{Ber}(P(E_i))
\,\middle\|\,
\operatorname{Ber}(Q_i(E_i))
\right)
\le\KL(P\|Q_i),
\)
and using joint convexity of KL divergence together with the binary-KL inequality,
we obtain
\[
\frac1{16}\log\frac1{|V|^{-1}\sum_{i\in V}Q_i(E_i)}
\le\frac1{|V|}\sum_{i\in V}\KL(P\|Q_i)+\log2
\le{O\!\left(\frac{\Delta^2T}{|V|}\right)}{{}+\log2},
\]
where the first inequality uses $|V|^{-1}\sum_iP(E_i)\ge1/16$ and the last uses the preceding bound on the average KL divergence.
By suitably choosing the numerical constants in the definition of
$\Delta$ and taking $\delta$ sufficiently small, the preceding inequality
implies
\(
\frac1{|V|}\sum_{i\in V}Q_i(E_i)\ge\delta.
\)
Averaging over the algorithm's private randomness before choosing the item, we find a fixed $i^\star\in V$, independent of that randomness, such that, in instance $i^\star$,
 the auxiliary algorithm incurs regret at least $\Delta T/8$
 with probability at least $\delta$.
 By the comparison between the two algorithms, the same guarantee
 holds for the original algorithm.
 Finally, our choice of $\Delta$ and $|V|\ge d/2$ give
$\Delta T/8=\Omega(\sqrt{dT\log(1/\delta)})$.
\end{itemize}
Finally, we choose the construction corresponding to the larger
of $\log K$ and $\log(1/\delta)$. Since
\(
\max\{\log K,\log(1/\delta)\}
\ge \bigl(\log K+\log(1/\delta)\bigr)/2,
\)
this gives the claimed lower bound up to a universal constant.
The complete proof is in Appendix~\ref{app:lower-bound}.
\end{proofsketch}

Finally, we remark that, when \(m=1\) and \(d=K\), our problem reduces to the standard adversarial \(K\)-armed bandit problem. Thus, the following corollary holds.

\begin{corollary}
\label{cor:ordinary-bandits-lower-bound}
There exist universal constants $c>0$ and $\delta_0\in(0,1)$ and a threshold $\mathfrak T_{\mathrm{MAB}}(K,\delta)$ bounded by a polynomial in $K$ and $\log(1/\delta)$ such that the following holds.
Let $K\ge2$, $\delta\in(0,\delta_0]$, and $T\ge\mathfrak T_{\mathrm{MAB}}(K,\delta)$.
For every randomized algorithm, there exists an adaptive non-anticipating adversary with losses $\ell_t\in[0,1]^K$ such that
\[
    \Pp \left(
    R_T\ge
    c\sqrt{KT\left(\log K+\log\frac1\delta\right)}
    \right)\ge\delta.
\]
\end{corollary}
\begin{proof}
Apply \Cref{thm:lower-bound} with $d=K$ and $m=1$, and set $\mathfrak T_{\mathrm{MAB}}(K,\delta):=\mathfrak T(K,1,\delta)$.
\end{proof}
\par

\clearpage
\appendix
\section*{Appendix}
\subsection*{Contents}
\addtocontents{toc}{\protect\setcounter{tocdepth}{2}}
\begingroup
\makeatletter
\@starttoc{toc}
\makeatother
\endgroup
\clearpage
\section{Related work}
\label{sec:related-work}

\paragraph{Optimal design and bandit linear optimization.}
\citet{kiefer1959optimum} introduced the general framework of approximate
optimal design, and \citet{kiefer1960equivalence} proved the equivalence
between $D$- and $G$-optimality.
For a finite action set $\mathcal A\subset\R^d$,
\citet[Theorem~4]{bubeck2012towards} combined exponential weights with
John's exploration and obtained
$O(\sqrt{dT\log|\mathcal A|})$ expected regret against oblivious adversaries under action-normalized
losses.
For general convex action sets,
\citet[Theorem~1]{abernethy2008competing} used self-concordant barriers
to obtain an efficient expected-regret bound of
$\widetilde O(d^{3/2}\sqrt T)$ {against oblivious adversaries}.

Early general high-probability guarantees with $\sqrt T$ dependence
against adaptive adversaries had a larger polynomial dependence on the
dimension.
\citet[Theorem~1]{bartlett2008highprobability} obtained
$\widetilde O(d^{3/2}\sqrt T)$ regret.
\citet[Theorem~4 and Section~5.4]{abernethy2009beating} then gave a
general high-probability reduction whose efficient instantiation requires
problem-specific geometric subroutines.
Later, \citet[Theorem~3.1]{lee2020bias} gave the first general
polynomial-time algorithm with high-probability regret of order $\sqrt T$ up to logarithmic factors, with worst-case regret
$\widetilde O(d^{7/2}\sqrt T)$.
\citet{zimmert2022return} improved this rate to
$\widetilde O(d^2\sqrt T)$. The efficiency guarantees for methods based on self-concordant barriers
assume access to the barrier and its derivatives, as specified in the
respective sources.

In the finite-action setting,
\citet[Theorem~6 and Algorithm~3]{zimmert2022return} combined EXP3 with
Kiefer-Wolfowitz exploration and obtained
$
O(
\sqrt{dT\log{(|\mathcal A|/\delta{)}}}
)
$
regret with high probability.
When applied to $\cX=\cX_{d,m}$, however, its direct implementation
maintains one weight for each of the
$K=\binom dm$ actions, and may therefore require exponential space.

\paragraph{High-probability adversarial multi-armed bandits.}

For oblivious $K$-armed bandits with losses in $[0,1]$, setting
$G_0=T/\log(1/\delta)$ in the bound of
\citet[Theorem~18]{audibert2010regret} gives an upper bound of
$O(\sqrt{KT\log(1/\delta)})$ { for $T\ge81K\log(1/\delta)$}, matched by the lower bound of
\citet[Theorem~1]{gerchinovitz2016refined}.
For adaptive non-anticipating adversaries, the upper bound is
$O(\sqrt{KT\log(K/\delta)})$
\citep[Theorem~6.10]{cesabianchi2006prediction}.
The $m=1$ specialization of our lower bound matches this rate, closing the
$\sqrt{\log K}$ gap at fixed confidence (see \Cref{cor:ordinary-bandits-lower-bound}).
These comparisons hold for sufficiently small $\delta$ and sufficiently large $T$.

\paragraph{Expectation-based regret in combinatorial bandits.}
Early adversarial bandit algorithms for structured action sets were given
by \citet{mcmahan2004online} and \citet{awerbuch2004adaptive}.
\citet{cesabianchi2012combinatorial} subsequently introduced ComBand, a
general framework for adversarial combinatorial bandits.

The expectation-based statistical picture was further clarified by
\citet[Section~4, after Theorem~4.1]{audibert2014regret}, who considered
full-information, semi-bandit, and full-bandit feedback and established
several lower bounds.
Their results use coordinatewise-bounded item losses,
$\ell_t\in[0,1]^d$, so an action selecting $m$ items may incur
instantaneous loss as large as $m$.
Under this normalization, they conjectured that the minimax bandit regret
in expectation should scale as $m\sqrt{dT}$.

Under the same coordinatewise normalization,
\citet[Theorems~1 and~5]{cohen2017tight} disproved this conjecture.
Using strongly correlated losses, they proved an expected-regret lower
bound of
\(
\Omega(\sqrt{{dm^3T}/{\log T}})
\) {against oblivious adversaries}
for a multitask action class in which every action selects $m$ items.
After rescaling the losses by $1/m$ to match our action-level
normalization, this lower bound becomes
\(
\Omega(\sqrt{{dmT}/{\log T}}).
\)

\citet[Theorem~2]{ito2019improved} later removed the
$1/\sqrt{\log T}$ factor using i.i.d.\ binary correlated losses.
Crucially, their multiple-play action class is exactly
\(
\left\{x\in\{0,1\}^d:\|x\|_1=m\right\},
\)
the $m$-set action class considered here.
In the regime $m\le d/2$ and $T=\Omega(dm^{3/2})$, their
expected-regret lower bound is $\Omega(\sqrt{dm^3T})$ {against oblivious adversaries} under
coordinatewise normalization, and hence $\Omega(\sqrt{dmT})$ after
rescaling to our action-level normalization.

These lower bounds concern expected regret.
Our results establish matching high-probability upper and lower
bounds against adaptive non-anticipating adversaries.

\paragraph{Efficient algorithms for structured action sets.}
A separate line of work studies how to avoid explicitly enumerating
exponentially large action sets.
\citet[Theorems~6 and~7]{combes2015combinatorial} introduced CombEXP,
which retains the regret scaling of earlier combinatorial-bandit
algorithms while using approximate projections and decompositions.
Their implementation is polynomial-time when the action polytope admits
a polynomial-size description by linear equalities and inequalities.
Thus, for $m\le d/2$ and item losses in $[0,1/m]$, their expected-regret
bound against oblivious adversaries becomes
$O(\sqrt{dmT\log(d/m)})$ when
$T\ge d\log(d/m)$.
This efficient guarantee concerns expected regret against oblivious
adversaries; our algorithm provides a high-probability guarantee against
adaptive non-anticipating adversaries.

For general bandit linear optimization under expectation-based regret,
\citet[Corollary~35]{hazan2016volumetric} use volumetric spanners, while
\citet[Theorem~1]{ito2019oracle} use a linear-optimization oracle.
Under their respective access models, both approaches achieve
$\widetilde O(\sqrt T)$ expected regret with polynomial dependence on
the dimension.

Earlier efficient high-probability guarantees with $T^{2/3}$ dependence were obtained
by \citet[Theorem~3.1]{braun2016efficient}, using a linear-optimization oracle
over the action polytope, and by Sakaue et al.~\cite[Theorem 1]{sakaue2018efficient}, using a
zero-suppressed binary decision diagram. These approaches are efficient when the
required oracle or compact representation is available, but their high-probability
guarantees do not attain $\sqrt T$ dependence.

{Kontogiannis et al.~\cite[Theorem~3.2 and Appendix~G.5]{kontogiannis2025kernelized}
give an efficient kernelized algorithm for $m$-set bandits with
coordinatewise-bounded losses.
Their high-probability regret bound against adaptive adversaries
scales as $T^{2/3}$ in the horizon, up to logarithmic factors,
compared with our $\sqrt T$ dependence.}

The work most closely related to ours is Maiti et al.~\cite[Appendix E.6]{maiti2025efficient}.
Under action-normalized losses, they represent $m$-sets as paths in a DAG and obtain a
polynomial-time high-probability algorithm against adaptive adversaries with regret
$\widetilde O(d\sqrt{mT})$. Moreover, applying their general DAG guarantee
\cite[Theorem 7]{maiti2025efficient} to the $m$-set DAG of
\cite[Appendix E.6]{maiti2025efficient} yields $\widetilde O(m\sqrt{dT})$ for $m \le d/2$.

\paragraph{Weighted $m$-set distributions.}
The weighted $m$-set distribution introduced in \Cref{sec:setting} is
also known as a conditional Bernoulli or rejective-sampling distribution
\citep{hajek1964asymptotic,chen1994weighted,chen1997statistical}.
Classical methods based on elementary symmetric polynomials allow us to
compute its normalizing constant and marginal probabilities, and to
sample from it efficiently
\citep{chen1997statistical,kulesza2011kdpps}.

Beyond these computational properties, the covariance structure of
weighted $m$-set distributions is central to our analysis.
We use the pseudoinverse estimate of
\citet[Corollary~1.2]{cesari2026sharp} to control leverage scores
through an affine majorant.

\section{Proofs from Section \ref{sec:affine-leverage}}
\label{app:structural-proofs}

We first establish the identity relating the inverse second-moment
	matrix of a weighted $m$-set distribution to the pseudoinverse of
	its covariance matrix. We then combine this identity with the
	pseudoinverse estimate of
	\citet[Corollary~1.2]{cesari2026sharp} to prove \Cref{lem:linear-majorant-final}.

\begin{lemma}
\label{lem:inverse-covariance-identity}
Let $p \in \Delta(\cX)$ be a weighted $m$-set distribution, where $1\le m\le d-1$.
Then its second-moment matrix $M$ is positive definite and, for every $y \in H_0$,
\[
    y^\top M^{-1}y
=
    y^\top \Sigma^+ y.
\]
Consequently, for every $x\in\cX$,
\[
    x^\top M^{-1}x
=
    1+(x-\mu)^\top\Sigma^+(x-\mu),
\]
where $\mu=\mu(p)$ is the marginal vector of $p$ and $\Sigma=\Sigma(p)$ is the covariance matrix of $p$.
\end{lemma}
\begin{proof}
Let $\Sigma:=\Sigma(p)$ and $\mu:=\mu(p)$.
We first show that $M$ is positive definite. Suppose that
$a^\top Ma=0$. Since
\[
a^\top Ma
=
\E_{X\sim p}\bigl[(a^\top X)^2\bigr],
\]
we have $a^\top x=0$ for every $x\in\cX$, because $p$ has full
support.

Fix $i\neq j$. Since $1\le m\le d-1$, we can choose two actions
that differ only by exchanging item $i$ with item $j$. Comparing
their inner products with $a$ gives $a_i=a_j$. Thus,
$a=c\mathbf1$ for some $c\in\R$. Since every $x\in\cX$ has
$m$ nonzero coordinates,
\(
0
=
a^\top x
=
cm.
\)
Therefore, $c=0$, and hence $a=\mathbf0$. This proves that $M\succ0$.

By \citep[Corollary~1.2]{cesari2026sharp},
$
\ker(\Sigma)
=
\operatorname{span}\{\mathbf1\}.
$
Since $\Sigma$ is symmetric, its range satisfies
$\operatorname{range}(\Sigma)=\ker(\Sigma)^\perp=H_0$.
Every $X\in\cX$ satisfies
${X^\top \mathbf1}=m$. Therefore,
\[
M\mathbf1
=
\E_{X\sim p}
\bigl[X(X^\top \mathbf1)\bigr]
=
m\mu.
\]
Since $M$ is invertible,
\begin{equation}
\label{eq:M-inverse-mu}
M^{-1}\mu
=
\frac{1}{m}\mathbf1.
\end{equation}
Now fix $y\in H_0$ and let $z:=M^{-1}y$. By
\eqref{eq:M-inverse-mu},
\[
\mu^\top z
=
(M^{-1}\mu)^\top y
=
\frac{1}{m}\mathbf1^\top y
=
0.
\]
It follows that
\(
\Sigma z
=
(M-\mu\mu^\top)z
=
y.
\)
Moreover, $\Sigma^+y$ also solves the equation $\Sigma u=y$.
Therefore,
\(
z-\Sigma^+y
\in
\ker(\Sigma)
=
\operatorname{span}\{\mathbf1\}.
\)
Since $y\in H_0$, it is orthogonal to this difference. Hence,
\begin{equation}
\label{eq:inverse-covariance-quadratic-form}
y^\top M^{-1}y
=
y^\top\Sigma^+y.
\end{equation}
Finally, fix $x\in\cX$ and let $y:=x-\mu$. Since
$
\mathbf1^\top x
=
\mathbf1^\top\mu
=
m,
$
we have $y\in H_0$. Expanding $x=\mu+y$ gives
\[
\begin{aligned}
x^\top M^{-1}x
=
\mu^\top M^{-1}\mu
+
2y^\top M^{-1}\mu
+
y^\top M^{-1}y =1+y^\top\Sigma^+y,
\end{aligned}
\]
where the second equality follows from
\eqref{eq:M-inverse-mu} and
\eqref{eq:inverse-covariance-quadratic-form}. Substituting
$y=x-\mu$ proves that
\[
x^\top M^{-1}x
=
1+(x-\mu)^\top\Sigma^+(x-\mu).
\]
This concludes the proof.
\end{proof}

We now combine this lemma with
\citep[Corollary~1.2]{cesari2026sharp} to prove the affine leverage
bound stated in \Cref{lem:linear-majorant-final}.\footnote{In an earlier preprint of this work
	\citep[Lemmas~2 and~4]{bacchiocchi2026msets-v2}, {we used the majorant
	$\phi_p^{\mathrm{old}}(x):=1+2\sum_{i=1}^d (x_i-\mu_i)^2/[\mu_i(1-\mu_i)]$,
	which satisfies $\E_{X\sim p}[\phi_p^{\mathrm{old}}(X)] = 2d+1$ and
	$0\le\phi_p^{\mathrm{old}}(x)\le 1+4d/\lambda$.}
	Here we use the improved pseudoinverse estimate of
	\citet[Corollary~1.2]{cesari2026sharp}, which reduces the coefficient
	from $2$ to $1$, to obtain
	$\E_{X\sim p}[\phi_p(X)] = d+1$ and
	$0\le\phi_p(x)\le 1+2d/\lambda$ for every $x\in\cX$,
	as stated in \Cref{lem:linear-majorant-final}.}

\linearKWMajorant*

\begin{proof}
For $x\in\cX$, let $y=x-\mu\in H_0$.  By
\Cref{lem:inverse-covariance-identity},
\[
x^\top M^{-1}x=1+y^\top\Sigma^+y.
\]
The pseudoinverse estimate in \citep[Corollary~1.2]{cesari2026sharp} gives
\[
y^\top\Sigma^+y
\le\sum_{i \in [d]}\frac{y_i^2}{\mu_i(1-\mu_i)},
\]
and hence $x^\top M^{-1}x\le\phi_p(x)$.
Taking expectations and using
$\E[(X_i-\mu_i)^2]=\mu_i(1-\mu_i)$ yields
\(
\E_{X\sim p}[\phi_p(X)]=d+1.
\)
In addition, we write $r=m/d$. If $x_i=1$, then
\[
\frac{(x_i-\mu_i)^2}{\mu_i(1-\mu_i)}
=\frac{1-\mu_i}{\mu_i}\le\frac{1}{\lambda r};
\]
if $x_i=0$, then
\[
\frac{(x_i-\mu_i)^2}{\mu_i(1-\mu_i)}
=\frac{\mu_i}{1-\mu_i}\le\frac{1}{\lambda(1-r)}.
\]
Since $x$ has $m$ ones and $d-m$ zeros,
\[
\phi_p(x)
\le1+\frac{m}{\lambda r}
+\frac{d-m}{\lambda(1-r)}
=1+\frac{2d}{\lambda}.
\]
Finally, because $x_i\in\{0,1\}$,
\[
\frac{(x_i-\mu_i)^2}{\mu_i(1-\mu_i)}
=\frac{\mu_i}{1-\mu_i}
+x_i\frac{1-2\mu_i}{\mu_i(1-\mu_i)}.
\]
Thus
\[
\phi_p(x)
=1+\sum_{i \in [d]} \frac{\mu_i}{1-\mu_i}
+\sum_{i \in [d]} x_i\frac{1-2\mu_i}{\mu_i(1-\mu_i)}.
\]
Distribute the constant term evenly over the selected coordinates using
$\sum_i x_i=m$, and multiply the resulting affine representation by
four, to obtain $4\phi_p(x)=c_p^\top x$.
\end{proof}
\section{Proofs from Section \ref{sec:algorithm-regret}}
\label{app:adaptive-analysis}

Throughout this appendix, $T\ge1$ is an integer,
$\delta\in(0,1)$, {and the parameter choices in
\Cref{alg:exp3-kw} give}
\[
0<\eta\le\frac{1}{64},
\qquad
\lambda=128\eta d\le\frac{1}{2}.
\]
The loss sequence $(\ell_t)_{t=1}^T$ is chosen by a non-anticipating adaptive adversary, that is,
$\ell_t$ is $\mathcal F_{t-1}$-measurable for every $t\in[T]$,
and satisfies the action-normalized condition in \Cref{sec:setting}.
\noindent Recall from \Cref{alg:exp3-kw} that
\[
\phi_t:=\phi_{p_t},
\qquad
c_t:=c_{p_t},
\qquad
c_t^\top x=4\phi_t(x),
\qquad
Z_t(x):
=
\widehat y_t(x)-\eta c_t^\top x.
\]
Furthermore, $\widehat y_t(x)=x^\top\widehat\ell_t$, where
$\widehat\ell_t=M_t^{-1}X_tY_t$ is the loss estimator computed
in \Cref{alg:exp3-kw}.

\begin{lemma}
\label{lem:bounded-increments}
Fix $t\in[T]$. Suppose that
$0<\eta\le1/64$, $\lambda=128\eta d\le1/2$, the loss at round
$t$ satisfies the action-normalized condition, and $p_t$ is a
weighted $m$-set distribution in
$\cP_{\lambda/2}$. Then,
for every $x\in\cX$,
$
    \eta |Z_t(x)|\le 1/4.
$
\end{lemma}

\begin{proof}
Since $|Y_t|\le1$, the Cauchy-Schwarz inequality in the norm induced by $M_t^{-1}$ gives
\[
\begin{aligned}
    |\widehat y_t(x)|
    = |x^\top M_t^{-1}X_tY_t|  
    \le
    \sqrt{x^\top M_t^{-1}x}
    \sqrt{X_t^\top M_t^{-1}X_t}.
\end{aligned}
\]
Applying \Cref{lem:linear-majorant-final} to
$p_t\in\cP_{\lambda/2}$, both leverage scores in the last inequality
are at most $\phi_t(x)$ and $\phi_t(X_t)$, respectively, and both
$\phi_t$-values are at most \({1+{4d}/{\lambda}}.\)
Thus \(|\widehat y_t(x)|\le {\left(1+{4d}/{\lambda}\right)}.\)
Also $0\le c_t^\top x\le 4{\left(1+{4d}/{\lambda}\right)}$. Hence
\[
    |Z_t(x)|
    =|\widehat y_t(x)-\eta c_t^\top x|
    \le (1+4\eta)\left(1+\frac{4d}{\lambda}\right).
\]
The parameter choice implies
\[
    \eta \left(1+\frac{4d}{\lambda}\right)
    =\eta+\frac{4\eta d}{\lambda}
    =\eta+\frac1{32}
    \le\frac3{64}.
\]
Therefore
\[
    \eta|Z_t(x)|
    \le (1+4\eta)\eta \left(1+\frac{4d}{\lambda}\right)
    \le\left(1+\frac1{16}\right)\frac3{64}
    =\frac{51}{1024}
    <\frac14,
\]
where we used $\eta\le1/64$. This proves the claim.
\end{proof}

{As shown in the proof of
\Cref{prop:efficient-implementation}, at every round $t$,
$p_t$ is a weighted $m$-set distribution and
$p_t\in\cP_{\lambda/2}$.
Hence, \Cref{lem:linear-majorant-final} applies at every round
with smoothing parameter $\lambda/2$.}
\noindent 
Applying \Cref{lem:linear-majorant-final} with $\lambda/2$, for every $x\in\cX$, we have
\[
0\le c_t^\top x\le4\left(1+\frac{4d}{\lambda}\right),
\qquad
\E_{X\sim p_t}[c_t^\top X]
=
4(d+1),
\qquad
x^\top M_t^{-1}x
\le
\frac{c_t^\top x}{4}.
\]

\begin{lemma}
\label{lem:omd-stability}
When running \Cref{alg:exp3-kw}, for every distribution
$q\in\cP_\lambda$,
\[
\sum_{t=1}^T
\left(
\E_{X\sim p_t}[Z_t(X)]
-
\E_{X\sim q}[Z_t(X)]
\right)
\le
\frac{\KL(q\|p_1)}{\eta}
+1
+
2\eta
\sum_{t=1}^T
\E_{X\sim p_t}[Z_t(X)^2].
\]
\end{lemma}
\begin{proof}
\allowdisplaybreaks[1]
Fix $t\in[T]$ and define the unprojected multiplicative-weights
distribution
\[
    \widetilde p_{t+1}(x)
    :=\frac{p_t(x)\exp(-\eta Z_t(x))}
    {\E_{X\sim p_t}[\exp(-\eta Z_t(X))]},
    \qquad x\in\cX.
\]
Since $Z_t(x)=z_t^\top x$ and $p_t=p_{\theta_t}$, this is precisely
$p_{\widetilde\theta_{t+1}}$ from \Cref{line:alg-unprojected}.
By \Cref{prop:efficient-implementation}, the update $p_{t+1}$ satisfies
\eqref{eq:alg-approximate-projection}. Hence, for every
$q\in\cP_\lambda$,
\begin{align}
    \KL(q\|p_{t+1})
    &\le \KL(q\|\widetilde p_{t+1})
        -\KL(p_{t+1}\|\widetilde p_{t+1})+\varepsilon_{\rm p}\notag\\
    &\le \KL(q\|\widetilde p_{t+1})+\varepsilon_{\rm p}\notag\\
    &=\KL(q\|p_t)+\eta\E_{X\sim q}[Z_t(X)]+\log\E_{X\sim p_t}[\exp(-\eta Z_t(X))]
        +\varepsilon_{\rm p}.
        \label{eq:omd-kl-update}
\end{align}
Here, the first inequality is \eqref{eq:alg-approximate-projection},
the second uses nonnegativity of KL divergence, and the equality
follows from the definition of $\widetilde p_{t+1}$.

By \Cref{lem:bounded-increments}, $|\eta Z_t(x)|\le1/4$ for every
$x\in\cX$. Therefore,
\begin{align}
    \log\E_{X\sim p_t}[\exp(-\eta Z_t(X))]&\le\log\!\left(1-\eta\E_{X\sim p_t}[Z_t(X)]
        +2\eta^2\E_{X\sim p_t}[Z_t(X)^2]\right)\notag\\
    &\le-\eta\E_{X\sim p_t}[Z_t(X)]
        +2\eta^2\E_{X\sim p_t}[Z_t(X)^2].
        \label{eq:omd-log-moment}
\end{align}
The first inequality uses $e^{-u}\le1-u+2u^2$ for $|u|\le1/4$
and monotonicity of the logarithm. The second uses $\log(1+s)\le s$ with
$s=-\eta\E_{X\sim p_t}[Z_t(X)]+2\eta^2\E_{X\sim p_t}[Z_t(X)^2]$. Combining \eqref{eq:omd-kl-update} with
\eqref{eq:omd-log-moment} and rearranging gives
\begin{align}
    &\eta\left(\E_{X\sim p_t}[Z_t(X)]
        -\E_{X\sim q}[Z_t(X)]\right)\notag\\
    &\quad\le\KL(q\|p_t)-\KL(q\|p_{t+1})
        +\eta\E_{X\sim p_t}[Z_t(X)]+\log\E_{X\sim p_t}[\exp(-\eta Z_t(X))]
        +\varepsilon_{\rm p}\notag\\
    &\quad\le\KL(q\|p_t)-\KL(q\|p_{t+1})
        +2\eta^2\E_{X\sim p_t}[Z_t(X)^2]
        +\varepsilon_{\rm p}.
        \label{eq:omd-one-step}
\end{align}
Summing \eqref{eq:omd-one-step} over $t$ and dividing by
$\eta$ yields
\begin{align*}
    \sum_{t=1}^T\left(\E_{X\sim p_t}[Z_t(X)]
        -\E_{X\sim q}[Z_t(X)]\right) &\le\frac{\KL(q\|p_1)-\KL(q\|p_{T+1})}{\eta}
        +2\eta\sum_{t=1}^T\E_{X\sim p_t}[Z_t(X)^2]
        +\frac{T\varepsilon_{\rm p}}{\eta}\\
    &\le\frac{\KL(q\|p_1)}{\eta}
        +2\eta\sum_{t=1}^T\E_{X\sim p_t}[Z_t(X)^2]+1.
\end{align*}
The first inequality uses telescoping of the KL terms. The second
uses $\KL(q\|p_{T+1})\ge0$ and $\varepsilon_{\rm p}=\eta/T$,
which proves the claim.
\end{proof}

\begin{lemma}
\label{lem:kw-concentration}
When running \Cref{alg:exp3-kw}, let $\delta_{\rm aux}\in(0,1)$, and let
$(q_t)_{t=1}^T$ be a sequence of distributions in
$\Delta(\cX)$ {such that $q_t$ is
$\mathcal F_{t-1}$-measurable for every $t\in[T]$}.
For every $t\in[T]$, define
\[
L_t(q_t)
:=
\E_{X\sim q_t}[X^\top \ell_t],
\qquad
\widehat L_t(q_t)
:=
\E_{X\sim q_t}[\widehat y_t(X)].
\]
Then, with probability at least $1-\delta_{\rm aux}$,
\[
\left|
\sum_{t=1}^T
\left(
\widehat L_t(q_t)-L_t(q_t)
\right)
\right|
\le
\frac{\eta}{4}
\sum_{t=1}^T
\E_{X\sim q_t}[c_t^\top X]
+
\frac{\log(2/\delta_{\rm aux})}{\eta}.
\]
\end{lemma}

\allowdisplaybreaks[1]
\begin{proof}
For every $t\in[T]$, define
$\Delta_t:=\widehat L_t(q_t)-L_t(q_t)$.
Recall that $\E_t[\cdot]=\E[\cdot\mid\mathcal F_{t-1}]$.
Since $p_t$, $q_t$, and $\ell_t$ are
$\mathcal F_{t-1}$-measurable, they are fixed conditionally on
$\mathcal F_{t-1}$, and the only randomness at round $t$ comes
from $X_t\sim p_t$.

We first show that $\Delta_t$ has conditional mean zero. By the
definition of the loss estimate,
\begin{equation}
    \E_t[\widehat\ell_t]
    =M_t^{-1}\E_t[X_tX_t^\top]\ell_t
    =M_t^{-1}M_t\ell_t=\ell_t.
    \label{eq:kw-concentration-loss-unbiased}
\end{equation}
Since $q_t$ is $\mathcal F_{t-1}$-measurable,
\eqref{eq:kw-concentration-loss-unbiased} gives
\begin{align}
    \E_t[\widehat L_t(q_t)]
    &=\E_{X\sim q_t}[X^\top\E_t[\widehat\ell_t]]=\E_{X\sim q_t}[X^\top\ell_t]=L_t(q_t).
    \label{eq:kw-concentration-comparator-unbiased}
\end{align}
Therefore, by the definition of $\Delta_t$ and
\eqref{eq:kw-concentration-comparator-unbiased},
\begin{equation}
    \E_t[\Delta_t]
    =\E_t[\widehat L_t(q_t)]-L_t(q_t)=0.
    \label{eq:kw-concentration-unbiased}
\end{equation}
Define $\bar x_t:=\E_{X\sim q_t}[X]$, so that
$\widehat L_t(q_t)=\bar x_t^\top M_t^{-1}X_tY_t$.
Then
\begin{align}
    \E_t[\widehat L_t(q_t)^2]
    &\le\bar x_t^\top M_t^{-1}\E_t[X_tX_t^\top]
        M_t^{-1}\bar x_t\notag\\
    &=\bar x_t^\top M_t^{-1}\bar x_t\notag\\
    &\le\E_{X\sim q_t}[X^\top M_t^{-1}X]\notag\\
    &\le\frac14\E_{X\sim q_t}[c_t^\top X].
    \label{eq:kw-concentration-second-moment}
\end{align}
The first inequality uses $|Y_t|\le1$, and the equality uses
$\E_t[X_tX_t^\top]=M_t$. The second inequality is Jensen's
inequality for the convex function $x\mapsto x^\top M_t^{-1}x$,
since $M_t^{-1}$ is positive semidefinite.
The third follows from \Cref{lem:linear-majorant-final} and
$c_t^\top X=4\phi_t(X)$, which give
$X^\top M_t^{-1}X\le c_t^\top X/4$.
Combining \eqref{eq:kw-concentration-comparator-unbiased} with
\eqref{eq:kw-concentration-second-moment} yields
\begin{align}
    \E_t[\Delta_t^2]
    =\E_t[\widehat L_t(q_t)^2]-L_t(q_t)^2 \le\E_t[\widehat L_t(q_t)^2]
      \le\frac14\E_{X\sim q_t}[c_t^\top X],
    \label{eq:kw-concentration-variance}
\end{align}
where the equality uses \eqref{eq:kw-concentration-comparator-unbiased}, the first
inequality drops the nonnegative square, and the second applies
\eqref{eq:kw-concentration-second-moment}.
We next bound $\Delta_t$ uniformly. By \eqref{eq:kw-concentration-second-moment} and
$c_t^\top x\le4\left(1+{4d}/{\lambda}\right)$ for every $x\in\cX$,
\begin{equation}
    \bar x_t^\top M_t^{-1}\bar x_t
    \le\frac14\E_{X\sim q_t}[c_t^\top X]\le \left(1+\frac{4d}{\lambda}\right).
    \label{eq:kw-concentration-mean-leverage-range}
\end{equation}
Here, the first inequality is contained in
\eqref{eq:kw-concentration-second-moment}, and the second uses
$c_t^\top x\le4\left(1+{4d}/{\lambda}\right)$. For the sampled action, we similarly have
\begin{equation}
    X_t^\top M_t^{-1}X_t
    \le\frac14 c_t^\top X_t\le \left(1+\frac{4d}{\lambda}\right),
    \label{eq:kw-concentration-leverage-range}
\end{equation}
where the first inequality is due to
\Cref{lem:linear-majorant-final}, and the second uses
$c_t^\top X_t\le4\left(1+{4d}/{\lambda}\right)$.
Consequently,
\begin{align}
    |\widehat L_t(q_t)|
    =|\bar x_t^\top M_t^{-1}X_tY_t|
    \le\sqrt{\bar x_t^\top M_t^{-1}\bar x_t}
        \sqrt{X_t^\top M_t^{-1}X_t}
    \le1+\frac{4d}{\lambda}.
    \label{eq:kw-concentration-estimate-range}
\end{align}
By \eqref{eq:kw-concentration-estimate-range}, the action-normalized condition
$|L_t(q_t)|\le1$, and ${\left(1+{4d}/{\lambda}\right)}\ge1$, we have
\begin{equation}
\label{eq:kw-concentration-delta-range}
\begin{aligned}
    |\Delta_t|
    &\le|\widehat L_t(q_t)|+|L_t(q_t)|
    \le2+\frac{4d}{\lambda}
    \le2\left(1+\frac{4d}{\lambda}\right).
\end{aligned}
\end{equation}
Moreover, the parameter choices give
\begin{equation}
    \eta \left(1+\frac{4d}{\lambda}\right)
    =\eta+\frac{4\eta d}{\lambda}
    =\eta+\frac1{32}\le\frac3{64},
    \label{eq:kw-concentration-eta-range}
\end{equation}
where the second equality uses $\lambda=128\eta d$ and the
inequality uses $\eta\le1/64$, both ensured by
\Cref{alg:exp3-kw}.
Combining \eqref{eq:kw-concentration-delta-range} with
\eqref{eq:kw-concentration-eta-range} yields
\begin{equation}
    |\eta\Delta_t|\le2\eta \left(1+\frac{4d}{\lambda}\right)\le\frac3{32}<1.
    \label{eq:kw-concentration-increment}
\end{equation}
At the same time, combining \eqref{eq:kw-concentration-increment},
\eqref{eq:kw-concentration-unbiased}, and \eqref{eq:kw-concentration-variance} yields
\begin{align}
    \E_t[\exp(\eta\Delta_t)]
    &\le1+\eta\E_t[\Delta_t]+\eta^2\E_t[\Delta_t^2]\notag\\
    &=1+\eta^2\E_t[\Delta_t^2]\notag\\
    &\le\exp\!\left(\eta^2\E_t[\Delta_t^2]\right)\notag\\
    &\le\exp\!\left(\frac{\eta^2}{4}
        \E_{X\sim q_t}[c_t^\top X]\right).
    \label{eq:kw-concentration-exponential-moment}
\end{align}
The first inequality uses $e^u\le1+u+u^2$ for $|u|\le1$,
which applies by \eqref{eq:kw-concentration-increment}. The equality uses
\eqref{eq:kw-concentration-unbiased}; the second inequality uses
$1+v\le e^v$; and the third uses \eqref{eq:kw-concentration-variance}.

For $s\in\{0,\ldots,T\}$, define
\begin{equation}
    \mathcal M_s
    :=\exp\!\left(\eta\sum_{t=1}^s\Delta_t
        -\frac{\eta^2}{4}\sum_{t=1}^s
        \E_{X\sim q_t}[c_t^\top X]\right).
    \label{eq:kw-concentration-supermartingale}
\end{equation}
Since $\E_{X\sim q_t}[c_t^\top X]$ is
$\mathcal F_{t-1}$-measurable, \eqref{eq:kw-concentration-exponential-moment}
and \eqref{eq:kw-concentration-supermartingale} imply
\[
    \E_t[\mathcal M_t]
    =\mathcal M_{t-1}
      \exp\!\left(-\frac{\eta^2}{4}\E_{X\sim q_t}[c_t^\top X]\right)
      \E_t[\exp(\eta\Delta_t)]
    \le\mathcal M_{t-1}.
\]
Thus $(\mathcal M_s)_{s=0}^T$ is a nonnegative supermartingale
with $\mathcal M_0=1$, and $\E[\mathcal M_T]\le1$.
By Markov's inequality,
$\Pp(\mathcal M_T\ge2/\delta_{\rm aux})\le(\delta_{\rm aux}/2)\E[\mathcal M_T]\le\delta_{\rm aux}/2$.
Using \eqref{eq:kw-concentration-supermartingale}, we conclude that, with
probability at least $1-\delta_{\rm aux}/2$,
\begin{equation}
    \sum_{t=1}^T\Delta_t
    \le\frac\eta4\sum_{t=1}^T\E_{X\sim q_t}[c_t^\top X]
        +\frac{\log(2/\delta_{\rm aux})}{\eta}.
    \label{eq:kw-concentration-upper-tail}
\end{equation}
The same argument applies to $-\Delta_t$: the conditional mean
remains zero by \eqref{eq:kw-concentration-unbiased}, and the bounds
\eqref{eq:kw-concentration-variance} and \eqref{eq:kw-concentration-increment} are
unchanged. Hence, with probability at least $1-\delta_{\rm aux}/2$,
\begin{equation}
    -\sum_{t=1}^T\Delta_t
    \le\frac\eta4\sum_{t=1}^T\E_{X\sim q_t}[c_t^\top X]
        +\frac{\log(2/\delta_{\rm aux})}{\eta}.
    \label{eq:kw-concentration-lower-tail}
\end{equation}
A union bound on the events in \eqref{eq:kw-concentration-upper-tail} and
\eqref{eq:kw-concentration-lower-tail} gives, with probability at least $1-\delta_{\rm aux}$,
\[
    \left|\sum_{t=1}^T\Delta_t\right|
    \le\frac\eta4\sum_{t=1}^T\E_{X\sim q_t}[c_t^\top X]
        +\frac{\log(2/\delta_{\rm aux})}{\eta}.
\]
Substituting $\Delta_t=\widehat L_t(q_t)-L_t(q_t)$ proves the claim.
\end{proof}

\begin{lemma}
\label{lem:leverage}
{When running \Cref{alg:exp3-kw}}, for every $\delta_{\rm aux}\in(0,1)$, with
probability at least $1-\delta_{\rm aux}$,
\[
\sum_{t=1}^T\phi_t(X_t)
\le
2(d+1)T
+
2\left(
1+\frac{4d}{\lambda}
\right)
\log\frac{1}{\delta_{\rm aux}}.
\]
\end{lemma}

\begin{proof}
Conditionally on $\mathcal F_{t-1}$, the distribution $p_t$ and the
function $\phi_t$ are fixed, while $X_t\sim p_t$. Therefore,
\Cref{lem:linear-majorant-final} gives
\[
\E_t[\phi_t(X_t)]
=
\E_{X\sim p_t}[\phi_t(X)]
=
d+1,
\qquad
0\le\phi_t(X_t)\le \left(1+\frac{4d}{\lambda}\right).
\]
Since $\phi_t(X_t)/{\left(1+{4d}/{\lambda}\right)}\in[0,1]$, convexity of the exponential gives
\(
e^u
\le
1+(e-1)u,\) \(
u\in[0,1].
\)
It follows that
\[
\begin{aligned}
\E_t\left[
\exp\left(\frac{\phi_t(X_t)}{\left(1+{4d}/{\lambda}\right)}\right)
\right]
&\le
1+
\frac{(e-1)(d+1)}{\left(1+{4d}/{\lambda}\right)} \le
\exp\left(
\frac{(e-1)(d+1)}{\left(1+{4d}/{\lambda}\right)}
\right).
\end{aligned}
\]
For $s\in\{0,\ldots,T\}$, define
\[
\mathcal M_s
:=
\exp\left(
\frac{1}{\left(1+{4d}/{\lambda}\right)}\sum_{t=1}^s\phi_t(X_t)
-
\frac{(e-1)(d+1)s}{\left(1+{4d}/{\lambda}\right)}
\right).
\]
The previous conditional moment bound shows that
$(\mathcal M_s)_{s=0}^T$ is a nonnegative supermartingale with
$\mathcal M_0=1$. Hence,
\(
\E[\mathcal M_T]\le1.
\)
By Markov's inequality, with probability at least $1-\delta_{\rm aux}$,
\[
\sum_{t=1}^T\phi_t(X_t)
\le
(e-1)(d+1)T
+
\left(1+\frac{4d}{\lambda}\right)\log\frac{1}{\delta_{\rm aux}}.
\]
Using $e-1\le2$ proves the result.
\end{proof}

\subsection{Proof of the upper bound}
\label{app:proof-main}

\mainRegretTheorem*

\begin{proof}
For every $x\in\cX$, define the smoothed comparator
\[
q_x
:=
(1-\lambda)\delta_x+\lambda U,
\]
where $\delta_x$ assigns probability one to $x$, and $U$ is the
uniform distribution over $\cX$.\footnote{For $p\in\Delta(\cX)$, its entropy is $H(p):=-\sum_{x\in\cX}p(x)\log p(x)$, with the convention $0\log0=0$.}
If $x_i=1$, the $i$-th marginal
of $q_x$ is
\(
1-\lambda(1-r),
\)
while, if $x_i=0$, it is $\lambda r$. Therefore,
$q_x\in\cP_\lambda$. Moreover, since $p_1=U$,
\[
\KL(q_x\|p_1)
=
\KL(q_x\|U)
=
\log K-H(q_x)
\le
\log K.
\]
Before comparing losses, observe that the learner's average estimated loss
equals its realized loss at every round. Indeed, $\mathbf1^\top X=m$ for
every $X\in\cX$, so
\(
    M_t\mathbf1
    =\E_{X\sim p_t}[X(X^\top\mathbf1)]
    =m\mu_t.
\)
Since $M_t$ is symmetric and invertible, this gives
\begin{equation}
\label{eq:learner-estimate-identity}
    \E_{X\sim p_t}[\widehat y_t(X)]
    =\mu_t^\top M_t^{-1}X_tY_t
    =\frac{\mathbf1^\top X_t}{m}Y_t
    =Y_t.
\end{equation}
Applying \Cref{lem:omd-stability} with $q=q_x$ gives
\begin{equation}
\label{eq:main-proof-omd}
\sum_{t=1}^T
\left(
\E_{X\sim p_t}[Z_t(X)]
-
\E_{X\sim q_x}[Z_t(X)]
\right)
\le
\frac{\log K}{\eta}
+1
+
2\eta
\sum_{t=1}^T
\E_{X\sim p_t}[Z_t(X)^2].
\end{equation}
We first bound the second-order term. Since
$
Z_t(X)
=
\widehat y_t(X)-\eta c_t^\top X,
$
we have
\[
Z_t(X)^2
\le
2\widehat y_t(X)^2
+
2\eta^2(c_t^\top X)^2.
\]
Furthermore, we have
\[
\begin{aligned}
\E_{X\sim p_t}[\widehat y_t(X)^2]
&=
Y_t^2
X_t^\top M_t^{-1}
\E_{X\sim p_t}[XX^\top]
M_t^{-1}X_t \\
&=
Y_t^2X_t^\top M_t^{-1}X_t \\
&\le
\phi_t(X_t),
\end{aligned}
\]
where the last inequality follows from $|Y_t|\le1$ and
\Cref{lem:linear-majorant-final}.
Recall that $0\le c_t^\top X\le4(1+4d/\lambda)$ and
$\E_{X\sim p_t}[c_t^\top X]=4(d+1)$.
Hence,
\[
\begin{aligned}
    \E_{X\sim p_t}[(c_t^\top X)^2]
    \le4\left(1+\frac{4d}{\lambda}\right)
      \E_{X\sim p_t}[c_t^\top X]=16\left(1+\frac{4d}{\lambda}\right)(d+1).
\end{aligned}
\]
It follows that
\begin{equation}
\label{eq:main-proof-second-order}
\begin{aligned}
    2\eta\E_{X\sim p_t}[Z_t(X)^2]
    &\le4\eta\phi_t(X_t)
      +64\eta^3\left(1+\frac{4d}{\lambda}\right)(d+1)\\
    &=4\eta\phi_t(X_t)+2\eta^2(1+32\eta)(d+1)\\
    &\le4\eta\phi_t(X_t)+3\eta^2(d+1).
\end{aligned}
\end{equation}
The equality uses $\lambda=128\eta d$, and the last inequality uses
$\eta\le1/64$.
Using
$
Z_t(X)
=
\widehat y_t(X)-\eta c_t^\top X
$
in \eqref{eq:main-proof-omd}, and then applying
\eqref{eq:main-proof-second-order} and
\eqref{eq:learner-estimate-identity}, gives
\begin{align}
\label{eq:main-proof-estimated-loss}
&\sum_{t=1}^T
\left(
Y_t
-
\E_{X\sim q_x}[\widehat y_t(X)]
\right)
\nonumber\\
&\quad\le
\frac{\log K}{\eta}
+1
+
4\eta\sum_{t=1}^T\phi_t(X_t)
+
3\eta^2(d+1)T
+
4\eta(d+1)T
-
\eta
\sum_{t=1}^T
\E_{X\sim q_x}[c_t^\top X].
\end{align}
We now apply \Cref{lem:leverage} with $\delta_{\rm aux}=\delta/2$.
With probability at least $1-\delta/2$,
\begin{equation}
\label{eq:main-proof-leverage}
\begin{aligned}
    4\eta\sum_{t=1}^T\phi_t(X_t)
    \le8\eta(d+1)T+\left(8\eta+\frac14\right)\log\frac2\delta \le8\eta(d+1)T+\log\frac2\delta.
\end{aligned}
\end{equation}
Here we used $\lambda=128\eta d$ and
$8\eta+1/4\le3/8\le1$.
For each fixed $x\in\cX$, apply \Cref{lem:kw-concentration} to the
constant {deterministic} sequence $q_t=q_x$, with $\delta_{\rm aux}=\delta/(2K)$.
A union bound over the $K$ actions gives, with probability at least
$1-\delta/2$, simultaneously for every $x\in\cX$,
\begin{equation}
\label{eq:main-proof-comparator-concentration}
    \sum_{t=1}^T\E_{X\sim q_x}
    [\widehat y_t(X)-X^\top \ell_t]
    \le\frac\eta4\sum_{t=1}^T\E_{X\sim q_x}[c_t^\top X]
    +\frac{\log(4K/\delta)}\eta.
\end{equation}
We work on the intersection of these two events, whose probability is at
least $1-\delta$. Combining
\eqref{eq:main-proof-estimated-loss},
\eqref{eq:main-proof-leverage}, and
\eqref{eq:main-proof-comparator-concentration} gives, for every $x\in\cX$,
\begin{align}
\label{eq:main-proof-combined}
    \sum_{t=1}^T\left(Y_t-\E_{X\sim q_x}[X^\top \ell_t]\right)
    &\le\frac{\log K+\log(4K/\delta)}\eta+1
    +12\eta(d+1)T\nonumber\\
    &\quad+3\eta^2(d+1)T+\log\frac2\delta
    -\frac{3\eta}{4}\sum_{t=1}^T\E_{X\sim q_x}[c_t^\top X].
\end{align}
The last term is nonpositive because $c_t^\top X\ge0$. The negative
correction $-\eta$ in \eqref{eq:main-proof-estimated-loss} absorbs the
$\eta/4$ term in \eqref{eq:main-proof-comparator-concentration}.
Furthermore, $3\eta^2\le3\eta$ since $\eta\le1$.
Thus, simultaneously for every $x\in\cX$,
\begin{equation}
\label{eq:main-proof-smoothed-comparator}
    \sum_{t=1}^T\left(Y_t-\E_{X\sim q_x}[X^\top \ell_t]\right)
    \le\frac{\log K+\log(4K/\delta)}\eta+1
    +15\eta(d+1)T+\log\frac2\delta.
\end{equation}
We may now choose a best action in hindsight $x^\star$, even though it can
depend on the realized history. Since
$q_{x^\star}=(1-\lambda)\delta_{x^\star}+\lambda U$ and every action loss
belongs to $[-1,1]$,
\[
    \E_{X\sim q_{x^\star}}[X^\top \ell_t]
    -(x^\star)^\top \ell_t
    =\lambda\left(\E_{X\sim U}[X^\top \ell_t]
    -(x^\star)^\top \ell_t\right)
    \le2\lambda.
\]
Applying \eqref{eq:main-proof-smoothed-comparator} with $x=x^\star$, using
$d+1\le2d$ and $\lambda=128\eta d$, therefore yields
\[
    R_T\le\frac{\log K+\log(4K/\delta)}\eta
    +1+286\eta dT+\log\frac2\delta.
\]
Since $\log K+\log(4K/\delta)\le2\log(12K/\delta)$,
$1+\log(2/\delta)\le\log(12K/\delta)$, and $\eta\le1$, we obtain
\begin{equation}
\label{eq:main-proof-final-parameters}
\begin{aligned}
    R_T
    \le\frac{2\log(12K/\delta)}{\eta}
      +286\eta dT+\log(12K/\delta)
    \le\frac{3\log(12K/\delta)}{\eta}+320\eta dT.
\end{aligned}
\end{equation}
Moreover, $K=\binom dm\ge2$, so $\log(K/\delta)\ge\log2$ and
\(
    \log(12K/\delta)
    \le5\log(K/\delta).
\)
We now substitute the learning rate from \Cref{alg:exp3-kw}.
If $\eta=\sqrt{\log(12K/\delta)/(320dT)}$, then
\[
\begin{aligned}
    R_T
    \le4\sqrt{320dT\log(12K/\delta)} \le160\sqrt{dT\log(K/\delta)}.
\end{aligned}
\]
Otherwise, $\eta=1/(256d)$, which implies
$T\le(256^2/320)d\log(12K/\delta)$.
The deterministic bound $R_T\le2T$ then gives
\[
\begin{aligned}
    R_T
    \le2T
    \le\frac{512}{\sqrt{320}}\sqrt{dT\log(12K/\delta)}
    \le64\sqrt{dT\log(K/\delta)}.
\end{aligned}
\]
Thus in both cases
$R_T\le160\sqrt{dT(\log K+\log(1/\delta))}$, as claimed.
\end{proof}

\section{Polynomial-time implementation}
\label{app:implementation}

This appendix proves \Cref{prop:efficient-implementation} by giving
a polynomial-time implementation of \Cref{alg:exp3-kw} that does not
enumerate $\cX$.
We show how to compute the required moments and draw samples while
maintaining each $p_t$ as a weighted $m$-set distribution.
We also implement an approximate KL-projection procedure that returns
a weighted $m$-set distribution in $\cP_{\lambda/2}$ satisfying
\eqref{eq:alg-approximate-projection}.

The appendix is organized as follows.
We first introduce additional notation.
Appendix~\ref{app:kl-projection} presents the approximate KL-projection
procedure and establishes its guarantees.
Appendix~\ref{app:moment-sampling} describes the moment and sampling
routines.
Finally, Appendix~\ref{app:proof-efficient-implementation} combines
these results to complete the proof of \Cref{prop:efficient-implementation}.

We introduce the following additional notation.
For a vector $z=(z_1,\ldots,z_n)$ and $k\in[n]$, we define its $k$-th elementary symmetric polynomial as
\[
    e_k(z)
:=
    \sum_{\substack{S\subseteq[n]:\,|S|=k}} \prod_{i\in S}z_i,
\qquad
    e_0(z):=1,
\qquad
    e_k(z):=0
\quad
    \text{for }k<0\text{ or }k>n.
\]
Let $\theta\in\R^d$. For the weights $w_i=e^{\theta_i}$, the normalizing constant of $p_\theta$ is $e_m(w)$.
The log-partition function associated with $p_\theta$ is
\[
    F(\theta)
:=
    \log \sum_{A\subseteq[d]:\,|A|=m} \exp\left(\sum_{i\in A}\theta_i\right)
=
    \log e_m(e^{\theta_1},\ldots,e^{\theta_d}).
\]
In particular, if $S\sim p_\theta$ and $X:=\mathbf 1_S\in\cX$ is its indicator vector, then
\[
    \nabla F(\theta)
=
    \E_{X\sim p_\theta}[X]
=
    \mu(p_\theta),
\qquad
    \nabla^2F(\theta)
=
    \Cov_{X\sim p_\theta}(X)
=
    \Sigma(p_\theta).
\]
Furthermore, given $v\in\R^d$, we write
$(v_+)_i:=\max\{v_i,0\}$ and
$(v_-)_i:=\max\{-v_i,0\}$ for every $i\in[d]$.
Moreover, we recall that, for a convex function $f:\R^d\to\R$, a vector $\zeta\in\R^d$
is a subgradient of $f$ at $v\in\R^d$ if
\[
f(u)\ge f(v)+\zeta^\top(u-v)
\qquad\text{for every }u\in\R^d.
\]
\subsection{{Approximate KL projection}}
\label{app:kl-projection}

This subsection describes the procedure used in
line~\ref{line:alg-kl-projection} of Algorithm~\ref{alg:exp3-kw}.
Its input is a parameter vector $\widetilde\theta\in\R^d$
representing the weighted $m$-set distribution
$p_{\widetilde\theta}$, together with $T\ge1$ and
$0<\eta\le1/(256d)$.
Let $\lambda$ and $\varepsilon_{\rm p}$ be defined from $\eta$ and $T$
as in \Cref{alg:exp3-kw}, and define
\begin{equation}
\label{eq:dual-required-accuracy}
    B_{\rm p}:=4dT,\qquad
    \mathcal D_{\rm p}:=\{v\in\R^d:\|v\|_1\le B_{\rm p}\},\qquad
    \tau:=\frac{\varepsilon_{\rm p}^2}{2B_{\rm p}^2}
         =\frac{\eta^2}{32d^2T^4}.
\end{equation}
Our procedure consists of applying the ellipsoid method
\citep[Sections~2.1 and~3.1]{groetschel1988geometric}, initialized
with the Euclidean ball
$B_2(\mathbf0,B_{\rm p}):=\{u\in\R^d:\|u\|_2\le B_{\rm p}\}$, to the convex problem
\begin{equation}
\label{eq:approximate-projection-problem}
    \min_{v\in \mathcal D_{\rm p}}G_{\widetilde\theta}(v),
\end{equation}
where
\begin{equation}
\label{eq:projection-dual-objective}
    G_{\widetilde\theta}(v)
    :=F(\widetilde\theta+v)-\underline{\mu}^\top v_++ \overline{\mu}^\top v_-,
\end{equation}
with $\underline{\mu}_i:=\lambda r$, $\overline{\mu}_i:=1-\lambda(1-r)$ for every
$i\in[d]$, and $r=m/d$.

Under the assumption $\KL(U\|p_{\widetilde\theta})\le T$,
\Cref{lem:approximate-projection} shows that the ellipsoid method
computes a vector $v\in \mathcal D_{\rm p}$ satisfying
\[
    G_{\widetilde\theta}(v)
    \le \min_{u\in\mathcal D_{\rm p}}G_{\widetilde\theta}(u)+\tau
\]
using $\poly(d,m,\log  B_{\rm p},\log(1/\tau))$ arithmetic operations
and $\poly(d,m)$ space.
At round $t$, the input vector is
$\widetilde\theta=\widetilde\theta_{t+1}$.
Setting $\theta_{t+1}:=\widetilde\theta_{t+1}+v$ then gives
the next sampling distribution $p_{t+1}:=p_{\theta_{t+1}}$.
By \Cref{lem:approximate-projection}, this is a weighted
$m$-set distribution in $\cP_{\lambda/2}$ satisfying
\eqref{eq:alg-approximate-projection} for every $q\in\cP_\lambda$.

\begin{lemma}
\label{lem:projection-dual}
Let $0<\lambda\le1/2$ and $\widetilde\theta\in\R^d$.
The function $G_{\widetilde\theta}$ defined in
\eqref{eq:projection-dual-objective} is convex and attains its
minimum over $\R^d$.
For every minimizer $v^\star\in\R^d$, we have
\begin{equation}
	\label{eq:projection-kl-appendix}
p_{\widetilde\theta+v^\star}
\in\argmin_{p\in\cP_\lambda}
\KL(p\|p_{\widetilde\theta}).
\end{equation}
Moreover, for every $v\in\R^d$, writing
$\mu:=\mu(p_{\widetilde\theta+v})$, the vector
\begin{equation}
\label{eq:projection-subgradient}
    \zeta_i(v):=
    \begin{cases}
        \mu_i-\underline{\mu}_i,&v_i>0,\\
        \mu_i-\overline{\mu}_i,&v_i<0,\\
        \mu_i-(\underline{\mu}_i+ \overline{\mu}_i)/2,&v_i=0
    \end{cases}
    \qquad(i\in[d])
\end{equation}
is a subgradient of $G_{\widetilde\theta}$ at $v$, with
$\|\zeta(v)\|_\infty\le1$.
\end{lemma}

\begin{proof}
The projection in~\eqref{eq:projection-kl-appendix} minimizes $\KL(p\|p_{\widetilde\theta})$ over
$p\in\Delta(\cX)$ subject to $\underline{\mu}\le\mu(p)\le \overline{\mu}$.
The uniform distribution has positive mass on every action and
satisfies $\underline{\mu}_i<r<\overline{\mu}_i$ for every $i\in[d]$. Hence, the strong-duality
and Karush-Kuhn-Tucker (KKT) optimality results of \citet[Sections~5.2.3 and~5.5.3]{boyd2004convex}
apply.
For multipliers $\alpha,\beta\ge\mathbf0$ associated with the constraints
$\underline{\mu}-\mu(p)\le\mathbf0$ and $\mu(p)-\overline{\mu}\le\mathbf0$, respectively, set
$v:=\alpha-\beta$. Using
$\log p_\theta(x)=\theta^\top x-F(\theta)$, we obtain
\[
\begin{aligned}
\mathcal L(p,\alpha,\beta)
&=\KL(p\|p_{\widetilde\theta})
  +\alpha^\top(\underline{\mu}-\mu(p))+\beta^\top(\mu(p)-\overline{\mu})\\[4pt]
&=\KL(p\|p_{\widetilde\theta})
  -v^\top\mu(p)+\underline{\mu}^\top\alpha-\overline{\mu}^\top\beta\\[4pt]
&=\sum_{x\in\cX}p(x)
  \Bigl[\log p(x)-\widetilde\theta^\top x
  +F(\widetilde\theta)-v^\top x\Bigr]
  +\underline{\mu}^\top\alpha-\overline{\mu}^\top\beta\\[4pt]
&=\sum_{x\in\cX}p(x)
  \Bigl[\log p(x)-(\widetilde\theta+v)^\top x
  +F(\widetilde\theta+v)\Bigr] +F(\widetilde\theta)-F(\widetilde\theta+v)
  +\underline{\mu}^\top\alpha-\overline{\mu}^\top\beta\\[4pt]
&=\KL(p\|p_{\widetilde\theta+v})
  +F(\widetilde\theta)-F(\widetilde\theta+v)
  +\underline{\mu}^\top\alpha-\overline{\mu}^\top\beta\\[4pt]
&=\KL(p\|p_{\widetilde\theta+\alpha-\beta})
  +F(\widetilde\theta)-F(\widetilde\theta+\alpha-\beta)
  +\underline{\mu}^\top\alpha-\overline{\mu}^\top\beta.
\end{aligned}
\]
Only the KL term in the last expression depends on $p$.
It is nonnegative and equals zero at
$p=p_{\widetilde\theta+\alpha-\beta}$.
For fixed multipliers, minimizing over $p\in\Delta(\cX)$ defines the
dual function
\[
\begin{aligned}
D(\alpha,\beta):=\min_{p\in\Delta(\cX)}\mathcal L(p,\alpha,\beta)=F(\widetilde\theta)-F(\widetilde\theta+\alpha-\beta)
  +\underline{\mu}^\top\alpha-\overline{\mu}^\top\beta.
\end{aligned}
\]
The dual problem is then $\max_{\alpha,\beta\ge\mathbf0}D(\alpha,\beta)$.

Take any pair $\alpha,\beta\ge\mathbf0$. We define $h\in\R^d$ by
$h_i:=\min\{\alpha_i,\beta_i\}$ for every $i\in[d]$, and set
$\alpha':=\alpha-h$ and $\beta':=\beta-h$.
We observe that $\alpha',\beta'\ge\mathbf0$ and
$\alpha'-\beta'=\alpha-\beta=v$. Hence,
\begin{align*}
D(\alpha',\beta')
&=D(v_+,v_-)\\
&=F(\widetilde\theta)-F(\widetilde\theta+v)
  +\underline{\mu}^\top(\alpha-h)-\overline{\mu}^\top(\beta-h)\\
&=D(\alpha,\beta)+(\overline{\mu}-\underline{\mu})^\top h\\
&\ge D(\alpha,\beta),
\end{align*}
where the first equality follows from
$\alpha_i'=\max\{\alpha_i-\beta_i,0\}=(v_+)_i$ and
$\beta_i'=\max\{\beta_i-\alpha_i,0\}=(v_-)_i$ for every $i\in[d]$,
the second uses $\alpha'-\beta'=v$, and the inequality follows from
$\overline{\mu}_i\ge \underline{\mu}_i$ for every $i$ and $h\ge\mathbf0$.
This shows that every feasible pair can be replaced by
$(v_+,v_-)$ without decreasing the dual objective.
We may therefore restrict to $\alpha=v_+$ and $\beta=v_-$, for which
$D(v_+,v_-)=F(\widetilde\theta)-G_{\widetilde\theta}(v)$.
Since $F(\widetilde\theta)$ is independent of $v$, maximizing the dual
is equivalent to minimizing \eqref{eq:projection-dual-objective}.
By strong duality and the KKT conditions, a minimizer $v^\star$
yields the KL projection $p_{\widetilde\theta+v^\star}$.

The function $G_{\widetilde\theta}$ is convex because $F$ is convex
and each coordinate term $-\underline{\mu}_i(v_+)_i+\overline{\mu}_i(v_-)_i$ has slope
$-\overline{\mu}_i$ for $v_i<0$ and $-\underline{\mu}_i$ for $v_i>0$, with $-\overline{\mu}_i\le-\underline{\mu}_i$.
Since $\nabla F(\widetilde\theta+v)=\mu$, the corresponding
coordinate derivatives of $G_{\widetilde\theta}$ are
$\mu_i-\overline{\mu}_i$ for $v_i<0$ and $\mu_i-\underline{\mu}_i$ for $v_i>0$.
At $v_i=0$, any value in $[\mu_i-\overline{\mu}_i,\mu_i-\underline{\mu}_i]$ is valid;
choosing the midpoint gives \eqref{eq:projection-subgradient}.
Consequently, for every $u,v\in\R^d$,
\[
G_{\widetilde\theta}(u)
\ge G_{\widetilde\theta}(v)+\zeta(v)^\top(u-v).
\]
Finally, $\underline{\mu}_i,\overline{\mu}_i,\mu_i\in[0,1]$ implies $|\zeta_i(v)|\le1$ for every
$i\in[d]$, and hence $\|\zeta(v)\|_\infty\le1$.
\end{proof}

\begin{lemma}
\label{lem:approximate-projection}
Let $T\ge1$, $0<\eta\le1/(256d)$, and
$\widetilde\theta\in\R^d$ satisfy
$\KL(U\|p_{\widetilde\theta})\le T$.
Let $\lambda$ and $\varepsilon_{\rm p}$ be defined from $(\eta,T)$
as in \Cref{alg:exp3-kw}, and let $B_{\rm p}$, $\mathcal D_{\rm p}$, and $\tau$
be as in \eqref{eq:dual-required-accuracy}.
The ellipsoid method applied to
\eqref{eq:approximate-projection-problem}, initialized with
$B_2(\mathbf0,B_{\rm p})$, computes a vector $v\in\mathcal D_{\rm p}$ satisfying
\begin{equation}
\label{eq:dual-objective-accuracy}
    G_{\widetilde\theta}(v)
    \le\min_{u\in\mathcal D_{\rm p}}G_{\widetilde\theta}(u)+\tau
\end{equation}
using $\poly(d,m,\log B_{\rm p},\log(1/\tau))$ arithmetic operations
and $\poly(d,m)$ space.
Setting $\theta^+:=\widetilde\theta+v$ gives a weighted $m$-set
distribution $p_{\theta^+}\in\cP_{\lambda/2}$ that satisfies,
for every $q\in\cP_\lambda$,
\begin{equation}
\label{eq:approximate-pythagorean}
    \KL(q\|p_{\theta^+})+\KL(p_{\theta^+}\|p_{\widetilde\theta})
    \le\KL(q\|p_{\widetilde\theta})+\varepsilon_{\rm p}.
\end{equation}
\end{lemma}

\begin{proof}
\label{proof:alt-approximate-projection}
Choose any global
minimizer $v^\star\in\argmin_{u\in\R^d}G_{\widetilde\theta}(u)$, whose existence
is guaranteed by \Cref{lem:projection-dual}.
We first bound $\|v^\star\|_1$ to show that this minimizer lies
inside $\mathcal D_{\rm p}$.
Since $\E_{X\sim U}[X]=r\mathbf1$, Jensen's inequality gives,
for every $v\in\R^d$,
\allowdisplaybreaks[1]
\[
\begin{aligned}
    F(\widetilde\theta+v)-F(\widetilde\theta)
    &=\log\left(
        \frac{\sum_{x\in\cX}e^{(\widetilde\theta+v)^\top x}}
             {\sum_{z\in\cX}e^{\widetilde\theta^\top z}}
      \right)\\
    &=\log\left(
        \sum_{x\in\cX}
        \frac{e^{\widetilde\theta^\top x}}
             {\sum_{z\in\cX}e^{\widetilde\theta^\top z}}
        e^{v^\top x}
      \right)\\
    &=\log\left(
        \sum_{x\in\cX}p_{\widetilde\theta}(x)e^{v^\top x}
      \right)\\
    &=\log\E_{X\sim U}\!\left[
        \frac{p_{\widetilde\theta}(X)}{U(X)}e^{v^\top X}
      \right]\\
    &\ge\E_{X\sim U}\!\left[
        \log\frac{p_{\widetilde\theta}(X)}{U(X)}+v^\top X
      \right]\\
    &=-\KL(U\|p_{\widetilde\theta})+r\mathbf1^\top v.
\end{aligned}
\]
By the definition of $G_{\widetilde\theta}$ in \eqref{eq:projection-dual-objective},
the preceding bound, and $\KL(U\|p_{\widetilde\theta})\le T$, we obtain
\[
    G_{\widetilde\theta}(v)\ge F(\widetilde\theta)-T
     +(1-\lambda)\bigl(r\|v_+\|_1+(1-r)\|v_-\|_1\bigr).
\]
Because $v^\star$ minimizes $G_{\widetilde\theta}$ over all of $\R^d$, comparison
with $v=\mathbf0$ gives $G_{\widetilde\theta}(v^\star)\le G_{\widetilde\theta}(\mathbf0)=F(\widetilde\theta)$.
Applying the preceding bound at $v^\star$ therefore yields
\[
\begin{aligned}
    (1-\lambda)\min\{r,1-r\}\|v^\star\|_1
    &\le(1-\lambda)\bigl(
        r\|(v^\star)_+\|_1+(1-r)\|(v^\star)_-\|_1
      \bigr)\le T,
\end{aligned}
\]
where we used
$\|(v^\star)_+\|_1+\|(v^\star)_-\|_1=\|v^\star\|_1$.
Since $\lambda\le1/2$ and
$\min\{r,1-r\}=\min\{m,d-m\}/d\ge1/d$, it follows that
\begin{equation}
\label{eq:alt-projection-minimizer-radius}
    \|v^\star\|_1
    \le\frac{T}{(1-\lambda)\min\{r,1-r\}}
    \le2dT=B_{\rm p}/2.
\end{equation}
In particular, $v^\star\in \mathcal D_{\rm p}$.
Thus, $\mathcal D_{\rm p}$ contains a global minimizer of $G_{\widetilde\theta}$.
Together with $\mathcal D_{\rm p}\subseteq\R^d$, this gives
\[
    \min_{u\in\R^d}G_{\widetilde\theta}(u)
    \le\min_{u\in \mathcal D_{\rm p}}G_{\widetilde\theta}(u)
    \le G_{\widetilde\theta}(v^\star)
    =\min_{u\in\R^d}G_{\widetilde\theta}(u).
\]
Hence the minimum over $\mathcal D_{\rm p}$ in
\eqref{eq:dual-objective-accuracy} equals the minimum over all of
$\R^d$.

Let $v\in \mathcal D_{\rm p}$ satisfy
\eqref{eq:dual-objective-accuracy}.
We first show that $p_{\widetilde\theta+v}\in\cP_{\lambda/2}$ and that
\eqref{eq:approximate-pythagorean} holds for every $q\in\cP_\lambda$.
We then show that the ellipsoid method computes such a vector $v$.
Since $v^\star\in \mathcal D_{\rm p}$ is a global minimizer,
\eqref{eq:dual-objective-accuracy} gives
\begin{equation}
\label{eq:alt-projection-objective-gap}
    0\le G_{\widetilde\theta}(v)-G_{\widetilde\theta}(v^\star)
    \le\tau.
\end{equation}

By \Cref{lem:projection-dual}, $p_{\widetilde\theta+v^\star}$ minimizes
$\KL(q\|p_{\widetilde\theta})$ over $q\in\cP_\lambda$.
In particular, $p_{\widetilde\theta+v^\star}\in\cP_\lambda$, therefore $\underline{\mu}\le\mu(p_{\widetilde\theta+v^\star})\le \overline{\mu}$
by \eqref{eq:P_lambda_set} and the definitions of $\underline{\mu}_i=\lambda r$, $\overline{\mu}_i=1-\lambda(1-r)$ for every
$i\in[d]$.

We also show two equalities that follow directly from the definition of 
$v^\star$.
For every $i\in[d]$ with $v_i^\star>0$, differentiating
\eqref{eq:projection-dual-objective} in coordinate $i$ at $v^\star$
gives $0=\mu_i(p_{\widetilde\theta+v^\star})-\underline{\mu}_i$, using
$\nabla F(\widetilde\theta+v^\star)=\mu(p_{\widetilde\theta+v^\star})$.
If $v_i^\star<0$, the same argument gives $0=\mu_i(p_{\widetilde\theta+v^\star})-\overline{\mu}_i$.
If $v_i^\star=0$, both $((v^\star)_+)_i$ and
$((v^\star)_-)_i$ are zero.
Multiplying by the corresponding coordinates of
$(v^\star)_+$ and $(v^\star)_-$ and summing, we obtain
\begin{equation}
\label{eq:alt-projection-boundary-products}
    (v^\star)_+^\top(\mu(p_{\widetilde\theta+v^\star})-\underline{\mu})
    =(v^\star)_-^\top(\overline{\mu}-\mu(p_{\widetilde\theta+v^\star}))=0.
\end{equation}
The definitions of the weighted $m$-set distribution $p_\theta$
and of $F$ give
\begin{equation}
\label{eq:alt-projection-log-probability}
    \log p_\theta(x)=\theta^\top x-F(\theta)
    \qquad\text{for every }x\in\cX.
\end{equation}
Substituting \eqref{eq:alt-projection-log-probability} into
the definition of KL divergence yields
\begin{align}
    \KL(p_{\widetilde\theta+v^\star}\|p_{\widetilde\theta+v})
    &=F(\widetilde\theta+v)-F(\widetilde\theta+v^\star)
      -\mu(p_{\widetilde\theta+v^\star})^\top(v-v^\star).
    \label{eq:alt-projection-kl-identity}
\end{align}
Combining \eqref{eq:alt-projection-boundary-products} with
\eqref{eq:alt-projection-kl-identity} and using
\eqref{eq:projection-dual-objective}, we obtain
\begin{equation}
\label{eq:alt-projection-dual-gap}
\begin{aligned}
    G_{\widetilde\theta}(v)-G_{\widetilde\theta}(v^\star)
    &=\KL(p_{\widetilde\theta+v^\star}\|p_{\widetilde\theta+v})+v_+^\top\bigl(\mu(p_{\widetilde\theta+v^\star})-\underline{\mu}\bigr)
      +v_-^\top\bigl(\overline{\mu}-\mu(p_{\widetilde\theta+v^\star})\bigr).
\end{aligned}
\end{equation}
All terms on the right-hand side of
\eqref{eq:alt-projection-dual-gap} are nonnegative.
Combining \eqref{eq:alt-projection-dual-gap} with
\eqref{eq:alt-projection-objective-gap} therefore gives
$\KL(p_{\widetilde\theta+v^\star}\|p_{\widetilde\theta+v})\le\tau$.
Pinsker's inequality now implies
\begin{align}
    \|\mu(p_{\widetilde\theta+v})-\mu(p_{\widetilde\theta+v^\star})\|_\infty
    &\le\TV(p_{\widetilde\theta+v},p_{\widetilde\theta+v^\star}) \le\sqrt{\tau/2}
      =\frac{\eta}{8dT^2}
      \le\frac{\lambda}{2d}.
    \label{eq:alt-projection-marginal-distance}
\end{align}
Each marginal of $p_{\widetilde\theta+v^\star}$, namely $\mu_i(p_{\widetilde\theta+v^\star})$ for $i\in[d]$,
belongs to $[\lambda r,1-\lambda(1-r)]$.
The lower and upper boundaries in the definition of
$\cP_{\lambda/2}$ are $\lambda r/2$ and
$1-\lambda(1-r)/2$, respectively.
Their distances from the corresponding endpoints of the preceding
interval are $\lambda r/2$ and $\lambda(1-r)/2$, both at least
$\lambda\min\{r,1-r\}/2\ge\lambda/(2d)$.
Thus \eqref{eq:alt-projection-marginal-distance} ensures that
every $\mu_i(p_{\widetilde\theta+v})$ lies between these boundaries.
By \eqref{eq:P_lambda_set}, this proves $p_{\widetilde\theta+v}\in\cP_{\lambda/2}$.

\noindent We next prove \eqref{eq:approximate-pythagorean}.
Since $\KL(p_{\widetilde\theta+v^\star}\|p_{\widetilde\theta+v})\ge0$, combining
\eqref{eq:alt-projection-dual-gap} with
\eqref{eq:alt-projection-objective-gap} also gives
\begin{equation}
\label{eq:alt-projection-constraint-terms}
    v_+^\top(\mu(p_{\widetilde\theta+v^\star})-\underline{\mu})+v_-^\top(\overline{\mu}-\mu(p_{\widetilde\theta+v^\star}))\le\tau.
\end{equation}
For every $q\in\cP_\lambda$, we have
\begin{align*}
    &\KL(q\|p_{\widetilde\theta+v})
      +\KL(p_{\widetilde\theta+v}\|p_{\widetilde\theta})
      -\KL(q\|p_{\widetilde\theta})\\
    &\quad=v^\top\bigl(\mu(p_{\widetilde\theta+v})-\mu(q)\bigr)\\
    &\quad\le v^\top\bigl(\mu(p_{\widetilde\theta+v})
                         -\mu(p_{\widetilde\theta+v^\star})\bigr)+v_+^\top\bigl(\mu(p_{\widetilde\theta+v^\star})-\underline{\mu}\bigr)
        +v_-^\top\bigl(\overline{\mu}-\mu(p_{\widetilde\theta+v^\star})\bigr)\\
    &\quad\le\|v\|_1
        \|\mu(p_{\widetilde\theta+v})-\mu(p_{\widetilde\theta+v^\star})\|_\infty
        +\tau\\
    &\quad\le B_{\rm p}\sqrt{\tau/2}+\tau
        =\frac{\varepsilon_{\rm p}}2
         +\frac{\varepsilon_{\rm p}^2}{2B_{\rm p}^2}
        \le\varepsilon_{\rm p}.
\end{align*}
The first equality follows by substituting
\eqref{eq:alt-projection-log-probability} into the three KL
divergences and canceling the terms involving $F$.
The first inequality uses $\underline{\mu}\le\mu(q)\le \overline{\mu}$ and $v=v_+-v_-$.
The second uses H\"older's inequality and
\eqref{eq:alt-projection-constraint-terms}; the third uses
$v\in\mathcal D_{\rm p}$ and \eqref{eq:alt-projection-marginal-distance}.
The last equality follows from \eqref{eq:dual-required-accuracy},
and the last inequality uses $\varepsilon_{\rm p}\le1\le B_{\rm p}^2$.
This proves \eqref{eq:approximate-pythagorean}.

It remains to show that the ellipsoid method computes a point
satisfying \eqref{eq:dual-objective-accuracy}, within the claimed
computational bounds.
By \Cref{lem:projection-dual}, $G_{\widetilde\theta}$ is convex,
and the subgradient in \eqref{eq:projection-subgradient} satisfies
$\|\zeta(u)\|_\infty\le1$. Hence $G_{\widetilde\theta}$ is
$\sqrt d$-Lipschitz in the Euclidean norm , meaning that $|G_{\widetilde\theta}(u)-G_{\widetilde\theta}(v)|\le\sqrt d\,\|u-v\|_2$ for all $u,v\in\R^d$.
Its value and a subgradient can be computed using $\poly(d,m)$
operations by \Cref{lem:esp-routines}.
Moreover, $\mathcal D_{\rm p}\subseteq B_2(\mathbf0,B_{\rm p})$ admits an efficient
separation oracle for the constraint $\|v\|_1\le B_{\rm p}$.
This procedure checks whether a point satisfies the constraint and, if not, returns a hyperplane separating that point from the feasible set.
At each iteration, we use a separation cut at an infeasible center and a subgradient cut at a feasible center. We retain the feasible center with the smallest objective value.

Set $r_{\rm ball}:=\tau/(2\sqrt d)$, where $B_2(w,r_{\rm ball})$ denotes the
Euclidean ball of radius $r_{\rm ball}$ centered at $w$.
By \eqref{eq:alt-projection-minimizer-radius} and $\tau\le B_{\rm p}$,
$B_2(v^\star,r_{\rm ball})\subseteq\mathcal D_{\rm p}$.
For every $u\in B_2(v^\star,r_{\rm ball})$, we have
\begin{equation}
\label{eq:alt-projection-near-minimum}
    G_{\widetilde\theta}(u)
    \le G_{\widetilde\theta}(v^\star)+\sqrt d\,\|u-v^\star\|_2
    \le G_{\widetilde\theta}(v^\star)+\tau/2.
\end{equation}
If none of the feasible centers visited so far satisfies \eqref{eq:dual-objective-accuracy},
\eqref{eq:alt-projection-near-minimum} ensures that the ellipsoid
at each iteration contains the ball $B_2(v^\star,r_{\rm ball})$.
Since the ellipsoid volume decreases by a fixed factor at each
iteration, such a point must eventually be found.
The standard ellipsoid bound
\citep[Sections~2.1 and~3.1]{groetschel1988geometric}
therefore yields a vector $v\in\mathcal D_{\rm p}$ satisfying
\eqref{eq:dual-objective-accuracy} in
$O(d^2\log(B_{\rm p}/r_{\rm ball}))$ oracle calls. We run the method for a predetermined number of iterations given by this bound, with a sufficiently large constant, and return the retained feasible center.
For this vector, setting $\theta^+:=\widetilde\theta+v$ gives
$p_{\theta^+}\in\cP_{\lambda/2}$ and
\eqref{eq:approximate-pythagorean}, as proved above.
Since $\log(B_{\rm p}/r_{\rm ball})=O(\log d+\log B_{\rm p}+\log(1/\tau))$,
the procedure uses $\poly(d,m,\log B_{\rm p},\log(1/\tau))$ operations.
The moment tables, ellipsoid matrix, and retained best feasible center
require $\poly(d,m)$ space and can be reused at each iteration.
\end{proof}

\subsection{Moment and sampling routines}
\label{app:moment-sampling}

We use the standard conditional-Bernoulli routines of
\citet[Sections~2 and~4]{chen1997statistical} to compute $F(\theta)$,
$\mu(p_\theta)$, and $M(p_\theta)$, and to draw exact samples from
$p_\theta$ without enumerating $\cX$.

\begin{lemma}
\label{lem:esp-routines}
Given $\theta\in\R^d$, the log-partition function $F(\theta)$, all marginals
$\mu_i(p_\theta)$, and all pairwise marginals
\[
\pi_{ij}(p_\theta)
:=
\Pp_{X\sim p_\theta}(X_i=X_j=1)
\]
can be computed, and an exact sample from $p_\theta$ can be drawn, in
polynomial time, without enumerating $\cX$.
\end{lemma}
\begin{proof}
\label{proof:alt-esp-routines}
By its definition in \Cref{sec:setting}, $p_\theta$ is the
distribution of independent Bernoulli variables with success probabilities
$w_i/(1+w_i)$, conditioned on their sum being $m$.
Thus the conditional-Bernoulli routines of
\citet[Sections~2 and~4]{chen1997statistical} apply.

We first compute the normalizing constant.
Splitting the sum defining $e_k$ according to whether item $j$ is
selected gives, for $j\in[d]$ and $1\le k\le m$,
\begin{equation}
\label{eq:alt-esp-recurrence}
    e_k(w_1,\ldots,w_j)
    =e_k(w_1,\ldots,w_{j-1})
     +w_j e_{k-1}(w_1,\ldots,w_{j-1}).
\end{equation}
This is Method~2 of \citet[Section~2]{chen1997statistical}.
With the conventions for $e_0$ and $e_k$ given above,
\eqref{eq:alt-esp-recurrence} computes the table for
$0\le k\le m$ and $0\le j\le d$ in $O(dm)$ operations and space.
In particular, it computes $e_m(w)$ and hence
$F(\theta)=\log e_m(w)$.

For every $i\in[d]$, the definitions of $\mu_i(p_\theta)$ and
$p_\theta$ give
\begin{align}
    \mu_i(p_\theta)
    &=\sum_{\substack{S\subseteq[d]:\,|S|=m\\i\in S}}p_\theta(S)
      =\frac{w_i}{e_m(w)}
      \sum_{\substack{A\subseteq[d]\setminus\{i\}\\|A|=m-1}}
          \prod_{\ell\in A}w_\ell
      =\frac{w_i e_{m-1}((w_\ell)_{\ell\in[d]\setminus\{i\}})}
             {e_m(w)}.
    \label{eq:alt-esp-marginals}
\end{align}
Similarly, for distinct $i,j\in[d]$,
\begin{align}
    \pi_{ij}(p_\theta)
    &=\sum_{\substack{S\subseteq[d]:\,|S|=m\\i,j\in S}}p_\theta(S)
      =\frac{w_iw_j}{e_m(w)}
      \sum_{\substack{A\subseteq[d]\setminus\{i,j\}\\|A|=m-2}}
          \prod_{\ell\in A}w_\ell
      =\frac{w_iw_j e_{m-2}((w_\ell)_{\ell\in[d]\setminus\{i,j\}})}
             {e_m(w)}.
    \label{eq:alt-esp-pairwise}
\end{align}
The last expression is zero when $m=1$, by the convention
$e_{-1}=0$. On the diagonal, $\pi_{ii}(p_\theta)=\mu_i(p_\theta)$.
Since $X_iX_j$ is the indicator of $X_i=X_j=1$, these probabilities
also give every entry of $M(p_\theta)$.
Applying \eqref{eq:alt-esp-recurrence} to the $O(d^2)$ lists obtained
by deleting one or two weights therefore computes
\eqref{eq:alt-esp-marginals} and \eqref{eq:alt-esp-pairwise}
in $O(d^3m)$ operations.
Processing one list at a time and storing the resulting moments
requires $O(dm+d^2)$ space.

Finally, Procedure~3 of \citet[Section~4]{chen1997statistical}
examines items $1,\ldots,d$ in order.
If $k$ items remain to be selected from $\{i,\ldots,d\}$,
it selects item $i$ with probability
\begin{equation}
\label{eq:alt-esp-sampling-probability}
    \frac{w_i e_{k-1}(w_{i+1},\ldots,w_d)}
         {e_k(w_i,\ldots,w_d)}.
\end{equation}
By the same calculation as in \eqref{eq:alt-esp-marginals},
\eqref{eq:alt-esp-sampling-probability} is the conditional probability
of selecting item $i$ under $p_\theta$, given all preceding choices.
Thus this procedure draws an exact sample from $p_\theta$.
Applying \eqref{eq:alt-esp-recurrence} in reverse order computes
the required values $e_k(w_i,\ldots,w_d)$ in $O(dm)$ operations
and space.
After this preprocessing, sampling uses $O(d)$ operations and
at most $d$ independent uniform random draws.
All these computations are polynomial in $d$ and $m$ and do not
enumerate $\cX$.
\end{proof}

\subsection{Proof of Proposition \ref{prop:efficient-implementation}}
\label{app:proof-efficient-implementation}

\efficientimplementation*

\begin{proof}
We first prove by induction that, for every $t\in[T+1]$,
$p_t=p_{\theta_t}$ is a weighted $m$-set distribution in
$\cP_{\lambda/2}$ and
\begin{equation}
\label{eq:uniform-kl-induction}
    \KL(U\|p_t)\le t-1.
\end{equation}
The additional KL bound in \eqref{eq:uniform-kl-induction} will allow us to apply
\Cref{lem:approximate-projection} at each round.
The base case holds because $\theta_1=\mathbf0$ and
$p_1=p_{\theta_1}=U\in\cP_\lambda\subseteq\cP_{\lambda/2}$.

Suppose that the claims hold at round $t\in[T]$.
By \Cref{lem:inverse-covariance-identity}, $M_t$ is positive
definite. In addition, $c_t$ and $\widehat\ell_t=M_t^{-1}X_tY_t$ are well defined.
Since $Z_t(x)=z_t^\top x$, with
$z_t=\widehat\ell_t-\eta c_t$, the MWU step in
\eqref{eq:main-mwu-update} gives
$\widetilde p_{t+1}=p_{\widetilde\theta_{t+1}}$, where
$\widetilde\theta_{t+1}=\theta_t-\eta z_t$.
Thus, the unprojected distribution remains a weighted $m$-set
distribution.

We next verify the input condition of
\Cref{lem:approximate-projection}.
The induction hypothesis ensures that $p_t$ is a weighted $m$-set
distribution in $\cP_{\lambda/2}$. Together with the algorithm's
parameter choices, this verifies the assumptions of
\Cref{lem:bounded-increments} at round $t$, giving
$|\eta Z_t(x)|\le1/4$ for every $x\in\cX$.
Hence $|\log\E_{X\sim p_t}[e^{-\eta Z_t(X)}]|\le1/4$, and
\[
    \left|\log\frac{\widetilde p_{t+1}(x)}{p_t(x)}\right|
    =\left|-\eta Z_t(x)
      -\log\E_{X\sim p_t}[e^{-\eta Z_t(X)}]\right|
    \le\frac12.
\]
By the definition of KL divergence, this implies
\[
    \KL(U\|\widetilde p_{t+1})
    \le\KL(U\|p_t)+\frac12
    \le t-1+\frac12\le T.
\]
The projection procedure described in Appendix~\ref{app:kl-projection} with input
$(\widetilde\theta_{t+1},\eta,T)$ therefore satisfies all the
assumptions of \Cref{lem:approximate-projection}.
Its output $\theta_{t+1}$ defines a weighted $m$-set distribution
$p_{t+1}=p_{\theta_{t+1}}\in\cP_{\lambda/2}$ satisfying
\eqref{eq:alg-approximate-projection}.
Taking $q=U\in\cP_\lambda$ in that inequality and dropping the
nonnegative term $\KL(p_{t+1}\|\widetilde p_{t+1})$ gives
\[
    \KL(U\|p_{t+1})
    \le\KL(U\|\widetilde p_{t+1})+\varepsilon_{\rm p}
    \le t-1+\frac12+\varepsilon_{\rm p}\le t,
\]
where $\varepsilon_{\rm p}=\eta/T\le1/64$.
This closes the induction.

We now bound the computational cost.
Given $\theta_t$, \Cref{lem:esp-routines} computes the marginals
and $M_t$ in $O(d^3m)$ operations and draws $X_t$ in polynomial
time without enumerating $\cX$.
Computing $c_t$ takes $O(d)$ operations, solving
$M_t\widehat\ell_t=X_tY_t$ takes $O(d^3)$, and updating $z_t$ and
$\widetilde\theta_{t+1}$ takes $O(d)$.
For the projection step, since $K\ge2$ and $\delta<1$, the learning
rate in \Cref{line:alg-learning-rate} satisfies
\[
    \eta\ge\frac1{256d\sqrt T}.
\]
Together with $B_{\rm p}=4dT$ and $\tau=\eta^2/(32d^2T^4)$, this gives
$\log B_{\rm p}+\log(1/\tau)=O(\log(dT))$.
Thus, \Cref{lem:approximate-projection} bounds the projection cost
by $\poly(d,m,\log T)$ operations per round.
The moment and sampling routines and the projection use
$\poly(d,m)$ space, which can be reused across rounds.
Including the polynomial-time initialization, the total cost is
$T\cdot\poly(d,m,\log T)$ time and $\poly(d,m)$ space, without
enumerating $\cX$.
\end{proof}

\section{Proof of the matching lower bound}
\label{app:lower-bound}

We prove the following fully explicit version of \Cref{thm:lower-bound}.

\begin{restatable}{theorem}{explicitLowerBoundTheorem}
\label{thm:lower-bound-explicit}
Let $d\ge2$, $1\le m\le d-1$, and $\delta\in(0,2^{-36}]$, and assume that
\[
    T\ge\left\lceil64d^2+4d \, {\min\{m,d-m\}^2}\log(1/\delta)\right\rceil.
\]
Then, for every randomized algorithm on $\cX_{d,m}$, there exists a
deterministic adaptive non-anticipating loss process whose action losses
belong to $[0,1]$ and such that
\[
    \Pp\left(
    R_T
    \ge
    2^{-43}
    \sqrt{dT\left(
    \log\binom dm+
    \log(1/\delta)
    \right)}
    \right)
    \ge
    \delta.
\]
\end{restatable}
Thus, \Cref{thm:lower-bound-explicit} implies \Cref{thm:lower-bound} by
taking $c=2^{-43}$ and $\delta_0=2^{-36}$.
The horizon threshold is at most $1+64d^2+4d\, {\min\{m,d-m\}^2}\log(1/\delta)$,
and is therefore bounded by a polynomial in $d$, $m$, and $\log(1/\delta)$ {since $\min\{m,d-m\}\le m$}.
The rest of this appendix proves \Cref{thm:lower-bound-explicit}.  

Our strategy is to prove two separate lower bounds of order
$\sqrt{dT\log K}$ (see \Cref{prop:lower-action-count,prop:lower-action-count-dense})
and $\sqrt{dT\log(1/\delta)}$ (see \Cref{prop:lower-confidence}) on the action space
$\cX_{d,m}$ with $m\le d/2$.
We then combine the two bounds and use \Cref{lem:lower-complement}
to extend the result to $m>d/2$.
In these lower bounds, we temporarily allow randomized adaptive
non-anticipating adversaries. Thus, conditionally on $\mathcal F_{t-1}$,
the adversary may randomize the choice of $\ell_t$ before the learner
draws $X_t\sim p_t$.
Notice that such auxiliary randomness used by the adversary is independent of the
learner's private randomness.
By \Cref{lem:lower-fixing-table}, we can then transfer the regret guarantees
obtained against a randomized adversary to a deterministic adaptive
non-anticipating adversary, as required by our model.

Before presenting \Cref{lem:lower-fixing-table,lem:lower-complement}, we introduce some additional notation. For $t\in[T]$, we denote by
$\cH_t:=(X_1,Y_1,\ldots,X_t,Y_t)$ the history of actions and observed
feedback up to round $t$, where $Y_t:=X_t^\top\ell_t$.
We set $\cH_0:=\varnothing$.
The filtration $(\mathcal F_t)_{t=0}^T$ introduced in
\Cref{sec:setting} satisfies $\mathcal F_t=\sigma(\cH_t)$
for every $t\in\{0,\ldots,T\}$.
Thus, conditioning on $\cH_{t-1}$ is equivalent to
conditioning on $\mathcal F_{t-1}$.
We use the KL divergence and total variation distance defined in \Cref{sec:setting}. For probability laws $Q$ and $P$ on the same measurable space, we also define the $\chi^2$ divergence by
\[
    \chi^2(Q\|P)
    :=
    \int\left(\frac{dQ}{dP}-1\right)^2dP
\]
when $Q\ll P$, and set it to $+\infty$ otherwise.

To formalize the randomized adversaries introduced above, let $\Upsilon$ collect
the algorithm's private randomness and let $\Xi$ collect the adversary's
auxiliary randomness.
We assume, without loss of generality, that these two collections are
sampled independently before the interaction begins. This is equivalent
to generating fresh randomness at each round since sampling it in advance
does not change the law of the interaction, and decisions still depend
on the history observed so far.
We can then write the algorithm's chosen action at time $t\in[T]$ as
$X_t=h_t(\cH_{t-1},\Upsilon)$ and the adversary's chosen loss vector as
$\ell_t=f_t(\cH_{t-1},\Xi)$, where $h_t$ and $f_t$ are deterministic functions.
Fixing $\Xi$ therefore yields a deterministic adaptive non-anticipating
adversary, as required by our model.

We can prove that the following lemma holds.

\begin{lemma}
	\label{lem:lower-fixing-table}
	Fix an algorithm, and let $\Upsilon$ and $\Xi$ denote its private randomness
	and the adversary's independent randomness, respectively.
	For any $b\ge0$ and $p\in[0,1]$, if $\Pp(R_T\ge b)\ge p$ over $\Upsilon$ and $\Xi$,
	then there exists a realization $\bar\xi$ of $\Xi$ such that the
	deterministic adversary obtained by fixing $\Xi=\bar\xi$ satisfies
	$\Pp(R_T\ge b)\ge p$ over $\Upsilon$ alone.
\end{lemma}

\begin{proof}
For each realization $\xi$ of $\Xi$, let $q(\xi)$ be the probability
of $R_T\ge b$ against the deterministic adversary obtained by fixing
$\Xi=\xi$, where the probability is taken over $\Upsilon$ alone.
By independence and the law of total expectation,
\[
    \Pp(R_T\ge b)=\E_{\Xi}[q(\Xi)]\ge p,
\]
where the probability on the left is taken over both $\Upsilon$ and $\Xi$.
If no such realization $\bar\xi$ existed, then $q(\xi)<p$ for every
realization $\xi$ of $\Xi$, and hence $\E_{\Xi}[q(\Xi)]<p$.
This contradicts the above inequality and proves the claim.
\end{proof}

Lemma~\ref{lem:lower-fixing-table} allows us to use a randomized adaptive non-anticipating adversary in our lower-bound construction. This is because, for every fixed algorithm, we can transfer the lower bound on regret to a deterministic adaptive non-anticipating adversary, preserving the same probability guarantee.
The following lemma transfers a regret lower bound from $m$-set to $(d-m)$-set actions at the same time horizon. It therefore suffices to
construct the lower-bound instances for $m\le d/2$.
\begin{lemma}
\label{lem:lower-complement}
Fix integers $d\ge2$, $1\le m\le d-1$, and $T\ge1$, and constants
$c_0\ge0$ and $\delta_0\in[0,1]$.
Suppose that, for every randomized algorithm on $\mathcal X_{d,m}$,
there exists a deterministic adaptive non-anticipating adversary whose
action losses belong to $[0,1]$ at every round and such that the regret
is at least $c_0$ with probability at least $\delta_0$.
Then the same statement holds on $\mathcal X_{d,d-m}$, with the same
horizon $T$ and parameters $c_0$ and $\delta_0$.
\end{lemma}

\begin{proof}
Fix any randomized algorithm on $\mathcal X_{d,d-m}$, and call it algorithm~1.
Construct algorithm~2 on $\mathcal X_{d,m}$ that runs algorithm~1
internally, plays $y_t=\mathbf1-x_t$ whenever algorithm~1
selects $x_t$, and returns the observed scalar feedback
to algorithm~1.
By assumption, there exists a deterministic adaptive non-anticipating
adversary 2 for algorithm~2, with loss vectors $g_t$ satisfying
$y^\top g_t\in[0,1]$ for every $y\in\mathcal X_{d,m}$, such that algorithm~2's
regret is at least $c_0$ with probability at least $\delta_0$.
We construct adversary~1 for algorithm~1 as follows.
At each round, replace every past action $x_u$ by $\mathbf1-x_u$ and
keep its observed feedback to obtain the corresponding $m$-set history.
Given this history, adversary~2 returns the loss vector $g_t$.
We then set
\begin{equation}
\label{eq:lower-complement-loss}
    \widetilde\ell_t
    :=
    -g_t+\frac{\mathbf1^\top g_t}{d-m}\mathbf1.
\end{equation}
This defines adversary~1, a deterministic adaptive non-anticipating adversary {for algorithm~1}, since
it uses only past actions and feedback.
For every $x\in\mathcal X_{d,d-m}$ and its complement
$y=\mathbf1-x\in\mathcal X_{d,m}$, we have $x^\top\mathbf1=d-m$ and hence
\[
    x^\top\widetilde\ell_t
    =-x^\top g_t+\mathbf1^\top g_t
    =y^\top g_t.
\]
In particular, all $(d-m)$-set action losses lie in $[0,1]$.

Run {algorithm~1 against adversary~1 and algorithm~2 against adversary~2}
with the same private randomness.
The above identity implies, by induction over the rounds, that they
select complementary actions and receive the same scalar feedback.
Let $R_T^{(d-m)}$ and $R_T^{(m)}$ denote their respective regrets.
The same identity pairs every fixed comparator in $\mathcal X_{d,d-m}$
with its complement in $\mathcal X_{d,m}$, so both the accumulated losses
of the algorithms and the minimum accumulated comparator losses are equal.
Thus $R_T^{(d-m)}=R_T^{(m)}$ for every realization of their common private
randomness, and consequently
\[
    \Pp\bigl(R_T^{(d-m)}\ge c_0\bigr)
    =\Pp\bigl(R_T^{(m)}\ge c_0\bigr)
    \ge\delta_0.
\]
Since algorithm~1 was arbitrary, this proves the claim.
\end{proof}

\subsection{Preliminary lemmas}
\label{app:lower-feedback-density}
In the following we denote
\[
    g(z):=1-z^2,
     \,\,
    \psi_u(z):=\frac z2+u g(z)
    \quad \textnormal{with}\quad
    z\in[-1,1].
\]
The next lemma bounds the KL divergence between the distributions that we use
to generate the feedback in our lower-bound instances.

\begin{lemma}
\label{lem:lower-feedback-density}
For every $|u|\le 1/16$, the map $\psi_u$ is an increasing
diffeomorphism from $[-1,1]$ onto $[-1/2,1/2]$, meaning a differentiable bijection with a differentiable inverse. Let $p_u$ be the density
of $\psi_u(Z)$ when $Z\sim\operatorname{Unif}[-1,1]$. For every
$|u|,|v|\le 1/16$,
\begin{equation}
\label{eq:lower-feedback-kl}
    \KL(p_u\|p_v)\le80(u-v)^2.
\end{equation}
\end{lemma}
\begin{proof}
For $|u|\le1/16$ and $z\in[-1,1]$,
\begin{equation}
    \frac38\le \psi_u'(z)=\frac12-2uz\le\frac58.
\end{equation}
Together with $\psi_u(-1)=-1/2$ and $\psi_u(1)=1/2$, this proves the
diffeomorphism claim.
Fix $y\in[-1/2,1/2]$, and let $z:=\psi_u^{-1}(y)$.
Since $y=\psi_u(z)=z/2+u(1-z^2)$, we have
\[
\begin{aligned}
	1-16uy+16u^2
	&=1-16u\left(\frac z2+u(1-z^2)\right)+16u^2\\
	&=1-8uz-16u^2+16u^2z^2+16u^2\\
	&=1-8uz+16u^2z^2\\
	&=(1-4uz)^2\\
	&=4\left(\frac12-2uz\right)^2\\
	&=4\psi_u'(z)^2.
\end{aligned}
\]
The change-of-variables formula therefore gives
\[
    p_u(y)=\frac1{2\psi_u'(z)}=\frac1{\sqrt{1-16uy+16u^2}},
\]
with
\( 9/{16}\le 1-16uy+16u^2\le{25}/{16}.\)
Moreover, since $|y|\le1/2$,
\[
    (32u-16y)^2
    =64(-16uy+16u^2)+256y^2
    \le64(1-16uy+16u^2).
\]
It follows that
\[
\begin{aligned}
	|\partial_u p_u(y)|
	=\frac{|32u-16y|}{2(1-16uy+16u^2)^{3/2}}
	\le\frac4{1-16uy+16u^2}
	\le\frac{64}{9}<8.
\end{aligned}
\]
Moreover, since $1-16vy+16v^2\le25/16$, we have
\[
p_v(y)
=\frac1{\sqrt{1-16vy+16v^2}}
\ge\frac1{\sqrt{25/16}}
=\frac45.
\]
The mean-value theorem, on the common support of length one, yields
\[
    \chi^2(p_u\|p_v)
    =\int_{-1/2}^{1/2}\frac{(p_u(y)-p_v(y))^2}{p_v(y)}\,dy
    \le80(u-v)^2.
\]
Since $\KL\le\chi^2$, this proves \eqref{eq:lower-feedback-kl}.
\end{proof}

The next lemma turns a lower bound on the expectation into a probability bound.

\begin{lemma}
\label{lem:lower-mean-to-probability}
Let $X$ be a real-valued random variable, and let $0<\mu\le M$.
Suppose that $X\le M$ almost surely and $\E[X]\ge\mu$. Then
\[
\Pp\left(X\ge\frac{\mu}{2}\right)
\ge \frac{\mu}{2M}.
\]
\end{lemma}

\begin{proof}
Since
\[
    \mu \le \E [X] \le  \E \left[ X \mathbf1\left\{X<\frac{\mu}{2}\right\}
    \right] + \E \left[ X \mathbf1\left\{X\ge\frac{\mu}{2}\right\}
    \right] 
    \le
    \frac\mu2
    +M\Pp(X\ge\mu/2),
\]
rearranging proves the claim.
\end{proof}

\subsection{{An} \texorpdfstring{$\Omega(\sqrt{dT\log K})$}{Omega(sqrt(d T log K))} lower bound}
\label{app:lower-action-count}

We prove the lower bound separately for
$r=m/d\le r_\star$ and $r=m/d>r_\star$, with $m\le d/2$ and $r_\star:=2^{-32}$.

\subsubsection{The case \texorpdfstring{$r\le r_\star$}{r <= r-star}}
\label{app:lower-combinatorial}

\begin{proposition}
\label{prop:lower-action-count}
Let $1\le m\le d/2$, with $m/d\le r_\star$, and assume
\(T\ge64d^2.\)
For every randomized algorithm on $\mathcal X_{d,m}$, there exists an adaptive non-anticipating adversary choosing action losses
in $[0,1]$ at every round such that
\begin{equation}
\label{eq:lower-action-count-conclusion}
    \Pp\left(
    R_T
    \ge
    2^{-24}\sqrt{dT\log\binom dm}
    \right)
    \ge
    \frac18.
\end{equation}
\end{proposition}

\begin{proof}
\medskip\noindent\textbf{Instances and regret.}
First, we set \[N_0:=\lceil8Tr\rceil \quad \textnormal{and} \quad \varepsilon:=m\sqrt{2^{-36}\log(1/r)/N_0}.\]
Independently draw indicators $U_i\sim\operatorname{Bernoulli}(\rho)$, with $\rho:=2^{-8}$, for every
item $i\in[d]$ and uniform noises $Z_t\sim\operatorname{Unif}[-1,1]$ for every round
$t\in[T]$, and collect these random variables in
$\Xi=(U_1,\ldots,U_d,Z_1,\ldots,Z_T)$, independently of the algorithm's
private randomness $\Upsilon$.

At round $t \in [T]$, let
$A_{t,i}:=\mathbf1\{\sum_{\tau<t}X_{\tau,i}<N_0\}$ for each item $i$, and
define its loss by
\begin{equation}
\label{eq:lower-sparse-loss}
    \ell_{t,i}
    :=\frac{1+Z_t}{2m}
    -\frac{\varepsilon}{m}g(Z_t)A_{t,i}U_i.
\end{equation}

Intuitively, $A_{t,i}=1$ exactly when the algorithm has selected item $i$
fewer than $N_0$ times before round $t$, and $A_{t,i}=0$ otherwise.
Thus, when $U_i=1$, we subtract a positive quantity from the loss of item $i$
until that item has been selected $N_0$ times.
To verify that action losses lie in $[0,1]$, note that $N_0\ge8Tr$ and $T\ge64d^2$ give
\[
\varepsilon^2
=\frac{m^2\log(1/r)}{2^{36}N_0}
\le\frac{m^2\log(1/r)}{2^{39}Tr}
\le\frac{m^2\log(1/r)}{2^{45}d^2r}
=2^{-45}r\log(1/r)\le2^{-45},
\]
so $\varepsilon<1/16$.

Since $m\ell_{t,i}=1/2+\psi_{-\varepsilon A_{t,i}U_i}(Z_t)$,
\Cref{lem:lower-feedback-density} therefore gives
$\ell_{t,i}\in[0,1/m]$, and hence every action $x$ is such that $x^\top \ell_t \in [0,1]$ for every $t$.
Moreover, once $\Xi$ is fixed, the losses in \eqref{eq:lower-sparse-loss}
are deterministic functions of past actions, as required for an
adaptive non-anticipating adversary.

Finally, to express the regret, we define the following scores
\[
    S_i:=U_i\sum_{t=1}^T A_{t,i}g(Z_t),
    \qquad
    S_{\rm L}:=\sum_{t=1}^T g(Z_t)
    \sum_{i=1}^d A_{t,i}X_{t,i}U_i.
\]
The quantity $(\varepsilon/m)S_i$ is the amount subtracted from the common term $\sum_{t=1}^T(1+Z_t)/(2m)$ in the cumulative loss of item $i$. Similarly, $(\varepsilon/m)S_{\rm L}$ is the amount subtracted from the common term $\sum_{t=1}^T(1+Z_t)/2$ in the algorithm's cumulative loss.
Thus, a larger $S_i$ means a smaller cumulative
item loss, so the best fixed action in hindsight selects the $m$ items with
the largest $S_i$.
We denote their sum by $\operatorname{Top}_m(S)$.
Since the common loss term cancels, the regret can be written as follows:
\begin{equation}
\label{eq:lower-regret-identity}
    R_T=\frac{\varepsilon}{m}\Gamma,
    \qquad
    \Gamma:=\operatorname{Top}_m(S)-S_{\rm L}\le mT,
\end{equation}
since $S_i\le T$ for every $i\in[d]$.

\medskip\noindent\textbf{Two deterministic bounds.}
For each item $i$, we have $A_{t,i}X_{t,i}=1$ at most $N_0$ times, because of the definition of $A_{t,i}$. Since
$0\le g(x)\le1$, for every $x \in [-1,1],$ we have:
\[
    S_{\rm L}\le N_0\sum_{i=1}^d U_i.
\]
Furthermore, $Tr\ge1$, so $N_0\le8Tr+1\le9Tr$. Consequently,
\begin{equation}
\label{eq:lower-sparse-learner}
    \E [S_{\rm L}] \le\frac{N_0d}{256}\le\frac9{256}mT,
\end{equation}
since $\E [U_i] = 1/256$ and $r=m/d$ by definition.
Before round $T$, there have been exactly $m(T-1)$ item selections in total, and each item with $A_{T,i}=0$ accounts for at least $N_0$ of them. Thus
\begin{equation}
\label{eq:lower-sparse-active-count}
	\begin{aligned}
		\sum_{i=1}^d A_{T,i}
		&=d-\sum_{i=1}^d (1-A_{T,i})\\
		&=d-\sum_{i=1}^d
		\mathbf1\!\left\{\sum_{\tau=1}^{T-1}X_{\tau,i}\ge N_0\right\}\\
		&\ge d-\frac1{N_0}\sum_{i=1}^d\sum_{\tau=1}^{T-1}X_{\tau,i}\\
		&=d-\frac{m(T-1)}{N_0}\\
		&\ge d-\frac{d(T-1)}{8T}
		\ge\frac78d.
	\end{aligned}
\end{equation}
The first inequality uses $N_0\,\mathbf1\{x\ge N_0\}\le x$
for $x\ge0$, and the third equality uses $\sum_{i=1}^dX_{\tau,i}=m$.
The remaining inequality follows from $N_0\ge8Tm/d$.
Moreover, $A_{T,i}=1$ implies $A_{t,i}=1$ for every $t\in[T]$.
Since $S_i=U_i\sum_{t=1}^T A_{t,i}g(Z_t)$ and $g\ge0$, we have $0\le S_i\le\sum_{t=1}^T g(Z_t)$ for every item $i$, with $S_i=\sum_{t=1}^T g(Z_t)$ whenever $U_iA_{T,i}=1$.
Consequently, on the event
$\{\sum_{i=1}^d U_iA_{T,i}\ge m\}$, at least $m$ items attain this
upper bound, and therefore
\[
    \operatorname{Top}_m(S)=m\sum_{t=1}^T g(Z_t).
\]
We next show that this event has large probability.

\medskip\noindent\textbf{Comparing the feedback distributions.}
Fix an item $i\in[d]$, a realization of the algorithm's private randomness
$\Upsilon$, and realizations of all indicators $U_j$ with $j\ne i$.
Let $P_i^0$ and $P_i^1$ denote the resulting laws of $\cH_T$ when $U_i=0$
and $U_i=1$, respectively, with only $Z_1,\ldots,Z_T$ remaining random.
Let $U_{-i}:=(U_j)_{j\ne i}$. The dependence of these laws on the fixed
values of $\Upsilon$ and $U_{-i}$ is left implicit.
We write $\E_{P_i^0}$ and $\E_{P_i^1}$ for the corresponding expectations.
Summing the item losses in \eqref{eq:lower-sparse-loss} over the $m$
selected items gives
\[
    Y_t-\frac12=\psi_{u_t}(Z_t),\qquad
    u_t=-\frac{\varepsilon}{m}\sum_{j=1}^d A_{t,j}X_{t,j}U_j,
\]
with $|u_t|\le\varepsilon<1/16$.
Given the fixed values of $\Upsilon$ and $U_{-i}$, the same preceding history
$\cH_{t-1}$ determines the same action $X_t$ and indicators $A_{t,j}$
under the two laws. Since $Z_t$ remains independent and uniform,
the corresponding parameters satisfy
\[
    |u_t^0-u_t^1|=\frac{\varepsilon}{m}A_{t,i}X_{t,i}.
\]
The chain rule for KL divergence, \Cref{lem:lower-feedback-density},
and the bound $\sum_{t\in [T]} A_{t,i}X_{t,i}\le N_0$ give
\begin{equation}
\label{eq:lower-sparse-bit-kl}
\begin{aligned}
    \KL(P_i^0\|P_i^1)
    &\le\frac{128\varepsilon^2}{m^2}
      \E_{P_i^0}\!\left[\sum_{t=1}^T A_{t,i}X_{t,i}\right]\le\frac{128\varepsilon^2 N_0}{m^2}
     =2^{-29}\log(1/r).
\end{aligned}
\end{equation}
This bound holds for every fixed realization of $\Upsilon$ and $U_{-i}$.

\medskip\noindent\textbf{Counting items that remain active.}
For each $i\in[d]$, define the event
\[
    H_i:=\left\{A_{T,i}=1,\ \sum_{j\ne i}U_jA_{T,j}<m\right\}.
\]
For fixed $\Upsilon,U_{-i}$, the indicators $A_{T,j}$ are determined by
$\cH_T$, so $H_i$ is the same measurable event of the history under
$P_i^0$ and $P_i^1$.
We now average over $\Upsilon$ and $U_{-i}$ and define
\[
\begin{aligned}
    h_{0,i}:=\Pp(H_i\mid U_i=0) \qquad\textnormal{and}\qquad
    h_{1,i}:=\Pp(H_i\mid U_i=1),
\end{aligned}
\]
where the probabilities are taken over $\Upsilon$, $U_{-i}$, and $Z_1,\ldots,Z_T$,
with $U_i$ fixed to $0$ and $1$, respectively.
Let
\[
    \bar h_0:=\frac1d\sum_{i=1}^d h_{0,i},\qquad
    \bar h_1:=\frac1d\sum_{i=1}^d h_{1,i}.
\]
\allowdisplaybreaks[1]
First, we bound $(1-\rho)d\bar h_0$ as follows
\begin{align}
    (1-\rho)d\bar h_0
    &=(1-\rho)\sum_{i=1}^d h_{0,i}
    \nonumber\\
    &=(1-\rho)\sum_{i=1}^d
      \Pp(H_i\mid U_i=0)
    \nonumber\\
    &=\sum_{i=1}^d
      \Pp(U_i=0,H_i)
    \nonumber\\
    &=\sum_{i=1}^d
      \E\!\left[
        \mathbf1\{U_i=0,H_i\}
      \right]
    \nonumber\\
    &=\sum_{i=1}^d
    \E\!\left[
    \mathbf1\left\{U_i=0, A_{T,i}=1,\ \sum_{j\ne i}U_jA_{T,j}<m\right\}
    \right]
    \nonumber\\
    &=\sum_{i=1}^d
    \E\!\left[
    \mathbf1\left\{U_i=0, A_{T,i}=1,\ \sum_{j= 1}^{d}U_jA_{T,j}<m\right\}
    \right]
    \nonumber\\
    &=\sum_{i=1}^d
      \E\!\left[
        (1-U_i)A_{T,i}
        \mathbf1\!\left\{\sum_{j= 1}^{d}U_jA_{T,j}<m\right\}
      \right]
    \nonumber\\
    &=\E\!\left[
        \left(\sum_{i=1}^d A_{T,i}
              -\sum_{i=1}^d U_iA_{T,i}\right)
        \mathbf1\!\left\{\sum_{j=1}^dU_jA_{T,j}<m\right\}
      \right]
    \nonumber\\
    &\ge\E\!\left[
        \left(\frac78d-\sum_{i=1}^dU_iA_{T,i}\right)
        \mathbf1\!\left\{\sum_{j=1}^dU_jA_{T,j}<m\right\}
      \right]
    \nonumber\\
    &\ge\left(\frac78d-m\right)
      \E\!\left[
        \mathbf1\!\left\{\sum_{j=1}^dU_jA_{T,j}<m\right\}
      \right]
    \nonumber\\
    &=\left(\frac78d-m\right)
      \Pp\!\left(\sum_{i=1}^dU_iA_{T,i}<m\right),
    \label{eq:lower-sparse-counting}
\end{align}
where the probabilities are taken over $\Upsilon$, $U_{1},\ldots, U_{d}$, and $Z_1,\ldots,Z_T$. The first inequality uses
$\sum_{i=1}^d A_{T,i}\ge7d/8$ from
\eqref{eq:lower-sparse-active-count}.
Writing $N_{\mathrm{act}}:=\sum_{i=1}^d U_iA_{T,i}$, we also bound $\rho d\bar h_1$ as follows:
\begin{align}
    \rho d\bar h_1
    =\sum_{i=1}^d
      \E\!\left[
        U_iA_{T,i}
        \mathbf1\!\left\{\sum_{j\ne i}U_jA_{T,j}<m\right\}
      \right] =\E\!\left[N_{\mathrm{act}}\mathbf1\{N_{\mathrm{act}}\le m\}\right]
    \le m.
    \label{eq:lower-sparse-counting-b}
\end{align}
The second equality holds pointwise: if $U_iA_{T,i}=1$, then
$\sum_{j\ne i}U_jA_{T,j}=N_{\mathrm{act}}-1$, and $N_{\mathrm{act}}-1<m$ is equivalent to
$N_{\mathrm{act}}\le m$ because $N_{\mathrm{act}}$ is integer-valued. If $U_iA_{T,i}=0$,
the summand vanishes. Thus, this identity requires no independence
assumption.
In particular, \eqref{eq:lower-sparse-counting-b} gives
$\bar h_1\le m/(\rho d)=256r$.
Recall that $\kl(a,b)$ denotes the KL divergence between Bernoulli laws with
parameters $a$ and $b$.
Then
\begin{equation}
\label{eq:lower-sparse-binary-kl}
\begin{aligned}
    \kl(\bar h_0,\bar h_1)
    &\le\frac1d\sum_{i=1}^d\kl(h_{0,i},h_{1,i})\\
    &\le\frac1d\sum_{i=1}^d
       \E_{\Upsilon,U_{-i}}\!\left[\kl\bigl(P_i^0(H_i),P_i^1(H_i)\bigr)\right]\\
    &\le\frac1d\sum_{i=1}^d
       \E_{\Upsilon,U_{-i}}\!\left[\KL(P_i^0\|P_i^1)\right]\\
    &\le2^{-29}\log(1/r).
\end{aligned}
\end{equation}
The first two inequalities use joint convexity of binary KL divergence,
first over the items and then over $\Upsilon,U_{-i}$.
The third applies data processing to $H_i$ for each fixed $\Upsilon,U_{-i}$,
and the last uses the uniform bound in
\eqref{eq:lower-sparse-bit-kl}.

We notice that if $\bar h_1=0$, we must also have $\bar h_0=0$, since otherwise
$\kl(\bar h_0,\bar h_1)=+\infty$, contradicting the upper bound in
\eqref{eq:lower-sparse-binary-kl}.
Then \eqref{eq:lower-sparse-counting} gives
$\Pp(\sum_{i=1}^d U_iA_{T,i}<m)=0$.
We may therefore assume that $\bar h_1>0$.
Using $\kl(a,b)\ge a\log(1/b)-\log2$ and $\bar h_1\le256r$ according to \eqref{eq:lower-sparse-counting-b},
we obtain
\[
    \bar h_0
    \le\frac{2^{-29}\log(1/r)+\log2}{\log(1/r)-8\log2}
    \le\frac43\left(2^{-29}+\frac1{32}\right)
    \le\frac3{64},
\]
where we used $r\le r_\star=2^{-32}$.
Since $r\le1/8$, \eqref{eq:lower-sparse-counting} now gives
\begin{equation}
\label{eq:lower-sparse-survival}
    \Pp\!\left(\sum_{i=1}^d U_iA_{T,i}<m\right)
    \le\frac{1-\rho}{7/8-r}\bar h_0
    \le\frac43\bar h_0\le\frac1{16}.
\end{equation}
Since $\operatorname{Top}_m(S)=m\sum_{t=1}^T g(Z_t)$ on the event
$\{\sum_{i=1}^d U_iA_{T,i}\ge m\}$, we have
\begin{equation}
\label{eq:lower-sparse-top-score}
\begin{aligned}
    \E[\operatorname{Top}_m(S)]
    &\ge m\E\!\left[
        \left(\sum_{t=1}^T g(Z_t)\right)
        \mathbf1\!\left\{\sum_{i=1}^d U_iA_{T,i}\ge m\right\}
      \right]\\
    &\ge m\E\!\left[\sum_{t=1}^T g(Z_t)\right]
        -mT\Pp\!\left(\sum_{i=1}^d U_iA_{T,i}<m\right)\\
    &\ge\left(\frac23-\frac1{16}\right)mT
     =\frac{29}{48}mT,
\end{aligned}
\end{equation}
where the last inequality holds by~\eqref{eq:lower-sparse-survival}.
Notice that this bound does not require the sums $\sum_{t=1}^T g(Z_t)$ and
$\sum_{i=1}^d U_iA_{T,i}$ to be independent since we only use $\sum_{t=1}^T g(Z_t)\le T$ and
$\E[\sum_{t=1}^T g(Z_t)]=2T/3$.

\medskip\noindent\textbf{Putting it all together.}
Combining \eqref{eq:lower-sparse-learner} and
\eqref{eq:lower-sparse-top-score} gives
\[
    \E[\Gamma]
    \ge\left(\frac{29}{48}-\frac9{256}\right)mT
    =\frac{437}{768}mT
    \ge\frac14mT.
\]
Since $\E[\Gamma]\ge mT/4$ and $\Gamma\le mT$ by
\eqref{eq:lower-regret-identity},
\Cref{lem:lower-mean-to-probability} gives
\begin{equation}
	\label{eq:lower-sparse-positive-event}
	\Pp(\Gamma\ge mT/8)\ge\frac18.
\end{equation}
Here the probability is taken over the algorithm's private randomness $\Upsilon$
and the adversary's randomness
$\Xi=(U_1,\ldots,U_d,Z_1,\ldots,Z_T)$.
Notice that this argument does not require $\Gamma$ to be nonnegative.
We next bound $\varepsilon T/8$. Using $N_0\le9Tr$, we have
\[
    (\varepsilon T)^2
    =
    \frac{m^2\log(1/r)T^2}{2^{36}{{N_0}}}
    \ge\frac{dmT\log(1/r)}{9\cdot2^{36}}.
\]
Moreover, $\binom dm\le(ed/m)^m$ and $\log(1/r)\ge1$ imply
$\log\binom dm\le2m\log(1/r)$.
Therefore,
\[
    \frac{\varepsilon T}{8}
    \ge
    \frac{2^{-21}}{\sqrt{18}}
    \sqrt{dT\log\binom dm}
    \ge
    2^{-24}\sqrt{dT\log\binom dm}.
\]
Using $R_T=(\varepsilon/m)\Gamma$ from \eqref{eq:lower-regret-identity}
and \eqref{eq:lower-sparse-positive-event}, we conclude that
\begin{equation}
\label{eq:lower-sparse-regret-eta}
    \Pp\!\left(
        R_T\ge2^{-24}\sqrt{dT\log\binom dm}
    \right)
    \ge\Pp(\Gamma\ge mT/8)
    \ge\frac18.
\end{equation}
Since $\Xi$ is independent of $\Upsilon$, \Cref{lem:lower-fixing-table}, with
$b=2^{-24}\sqrt{dT\log\binom dm}$ and $p=1/8$, guarantees the existence
of a realization $\bar\xi$ of $\Xi$ for which this regret bound holds
with probability at least $1/8$ over $\Upsilon$ alone.
This proves \eqref{eq:lower-action-count-conclusion}.
\end{proof}

\subsubsection{The case \texorpdfstring{$r>r_\star$}{r > r-star}}
\label{app:lower-bounded-density}

\begin{proposition}
\label{prop:lower-action-count-dense}
Let $1\le m\le d/2$, with $m/d>r_\star$, and assume
\(T\ge64d^2.\)
For every randomized algorithm on $\mathcal X_{d,m}$, there exists a
deterministic oblivious loss sequence with action losses in $[0,1]$
at every round such that
\begin{equation}
\label{eq:lower-dense-action-count-conclusion}
    \Pp\left(
    R_T\ge2^{-42}\sqrt{dT\log\binom dm}
    \right)
    \ge2^{-14}.
\end{equation}
\end{proposition}

\begin{proof}
\medskip\noindent\textbf{Instances and regret.}
We use the first $2m$ items, arranged into the pairs
$(1,2),\ldots,(2m-1,2m)$, and set $\varepsilon:=m/(128\sqrt T)$.
Independently draw a sign $U_i\in\{-1,1\}$ with equal probabilities for every
$i\in[2m]$ and a uniform noise $Z_t\sim\operatorname{Unif}[-1,1]$ for every $t\in[T]$.
Set $U_i:=0$ for $i>2m$.
Collect the adversary's randomness in
$\Xi=(U_1,\ldots,U_{2m},Z_1,\ldots,Z_T)$, independently of the
algorithm's private randomness $\Upsilon$.
At every round, define the item losses by
\begin{equation}
\label{eq:lower-dense-item-loss}
    \ell_{t,i}:=\frac{1+Z_t}{2m}-\frac\varepsilon m g(Z_t)U_i,
\end{equation}
for every $i\in[d]$.
This construction is oblivious as fixing $\Xi$ determines the entire loss sequence before the interaction starts.
We observe that $U_i=1$ reduces the loss of item $i$, while $U_i=-1$ increases it.
Since $T\ge d^2$ and $m\le d/2$, we have $\varepsilon\le1/256<1/16$.
The identity $m\ell_{t,i}=1/2+\psi_{-\varepsilon U_i}(Z_t)$ and
\Cref{lem:lower-feedback-density} give $\ell_{t,i}\in[0,1/m]$.
Hence $x ^\top \ell_t \in [0,1]$ for every $t$ and $x \in \cX$.
In particular, the observed loss satisfies
\begin{equation}
\label{eq:lower-dense-action-loss}
    Y_t=\frac12+\psi_{u_t}(Z_t),
    \qquad
    u_t:=-\frac\varepsilon m\sum_{i=1}^dX_{t,i}U_i\in[-\varepsilon,\varepsilon].
\end{equation}
As in the proof of \Cref{prop:lower-action-count}, define the item and algorithm scores by
\[
    S_i:=U_i\sum_{t=1}^Tg(Z_t),
    \qquad
    S_{\rm L}:=\sum_{t=1}^Tg(Z_t)\sum_{i=1}^dX_{t,i}U_i.
\]
As before, a larger $S_i$ means a smaller cumulative item loss.
To build a comparator, we choose the item with the larger $S_i$ in each pair $(1,2),\ldots,(2m-1,2m)$.
This gives a fixed comparator selecting exactly $m$ items, with total
score $\sum_{j=1}^m\max\{S_{2j-1},S_{2j}\}$.
As in the proof of \Cref{prop:lower-action-count}, the common loss terms cancel, so the regret compared with this fixed action gives
\begin{equation}
\label{eq:lower-dense-regret-comparison}
    R_T\ge\frac\varepsilon m\Gamma,
    \qquad
    \Gamma:=\sum_{j=1}^m\max\{S_{2j-1},S_{2j}\}-S_{\rm L}.
\end{equation}

\medskip\noindent\textbf{The comparator's score.}
The expected comparator score satisfies
\begin{align}
\label{eq:lower-dense-comparator-score}
    \E\left[\sum_{j=1}^m\max\{S_{2j-1},S_{2j}\}\right]
    &=\E\left[\sum_{t=1}^Tg(Z_t)\right]
      \sum_{j=1}^m\E\left[\max\{U_{2j-1},U_{2j}\}\right]\nonumber\\
    &=\frac{2T}{3}\cdot\frac m2=\frac{mT}{3},
\end{align}
where the expectation is taken over $U_1,\ldots,U_{2m}$ and the noises $Z_1,\ldots,Z_T$.
The first equality uses $g(Z_t)\ge0$ and the independence of the signs
and noises. The second uses $\E[g(Z_t)]=2/3$ and
$\E[\max\{U_{2j-1},U_{2j}\}]=1/2$.

\medskip\noindent\textbf{The algorithm's score.}
Fix an item $i\in[2m]$, a realization of the algorithm's private randomness
$\Upsilon$, and realizations of all signs $U_j$ with $j\ne i$.
Let $P_i^+$ and $P_i^-$ denote the resulting laws of $\cH_T$ when $U_i=1$
and $U_i=-1$, respectively, with only $Z_1,\ldots,Z_T$ remaining random.
Write $\E_{P_i^+}$ and $\E_{P_i^-}$ for the corresponding expectations.
Given the algorithm's private randomness $\Upsilon$ and the history up to
round $t-1$, namely $\cH_{t-1}$, the action $X_t$ played by the algorithm is fixed.
Therefore, the two parameters in \eqref{eq:lower-dense-action-loss},
\[
u_t^+
:=-\frac{\varepsilon}{m}
\left(\sum_{j\ne i}X_{t,j}U_j+X_{t,i}\right),
\qquad
u_t^-
:=-\frac{\varepsilon}{m}
\left(\sum_{j\ne i}X_{t,j}U_j-X_{t,i}\right),
\]
are also fixed, and their absolute difference is
$2\varepsilon X_{t,i}/m$.
By \eqref{eq:lower-dense-action-loss}, the KL divergence between the
conditional distributions of $Y_t$ in the two instances, given the
same history up to round $t-1$ and the same realization of the private
randomness $\Upsilon$, is equal to that between the distributions of
$\psi_{u_t^+}(Z)$ and $\psi_{u_t^-}(Z)$, where
$Z\sim\operatorname{Unif}[-1,1]$, since adding the common constant
$1/2$ does not change the KL divergence.
Writing $p_u$ for the
density of $\psi_u(Z)$, the KL chain rule and
\eqref{eq:lower-feedback-kl} give
\[
\begin{aligned}
    \KL(P_i^+\|P_i^-)
    &=\sum_{t=1}^T\E_{P_i^+}\!\left[
        \KL(p_{u_t^+}\|p_{u_t^-})\right]\\
    &\le128\sum_{t=1}^T\E_{P_i^+}\!\left[(u_t^+-u_t^-)^2\right]\\
    &=128\left(\frac{2\varepsilon}{m}\right)^2
        \E_{P_i^+}\left[\sum_{t=1}^T X_{t,i}\right]
    \le\frac1{32}.
\end{aligned}
\]
The second equality uses $X_{t,i}^2=X_{t,i}$, since $X_{t,i}\in\{0,1\}$.
The last inequality follows from $\sum_{t=1}^T X_{t,i}\le T$ and
$\varepsilon=m/(128\sqrt T)$.
Pinsker's inequality then yields
\[
    \TV(P_i^+,P_i^-)
    \le\sqrt{\frac12\KL(P_i^+\|P_i^-)}
    \le\frac18.
\]
The number of times item $i$ is selected depends only on $\cH_T$ and lies in $[0,T]$.
Since the two signs of $U_i$, corresponding to instances $P_i^+$ and $P_i^-$, are equally likely, it follows that
\[
\begin{aligned}
    \left|\E\!\left[U_i\sum_{t=1}^TX_{t,i}
       \,\middle|\,\Upsilon,(U_j)_{j\ne i}\right]\right|
    &=\frac12\left|
        \E_{P_i^+}\left[\sum_{t=1}^T X_{t,i}\right]
        -\E_{P_i^-}\left[\sum_{t=1}^T X_{t,i}\right]
    \right|\\
    &\le\frac T2\TV(P_i^+,P_i^-)\le\frac T{16}.
\end{aligned}
\]
Conditionally on $\Upsilon$, $U_1,\ldots,U_{2m}$, and $Z_1,\ldots,Z_{t-1}$, the action $X_t$ is fixed, while $Z_t$ remains uniformly distributed on $[-1,1]$. Using the above inequality, the tower property, and the fact that $\E[g(Z_t)]=2/3$ for every $t$, we obtain
\begin{equation*}
    \E[S_{\rm L}]
    =\frac23\sum_{i=1}^{2m}
       \E\left[U_i\sum_{t=1}^TX_{t,i}\right]
    \le\frac23\cdot2m\cdot\frac T{16}
     =\frac{mT}{12},
\end{equation*}
where the expectations are taken over $U_1,\ldots,U_{2m}$, the noises $Z_1,\ldots,Z_T$, and the algorithm's private randomness $\Upsilon$.
The equality uses $\E[g(Z_t)]=2/3$ and $U_i=0$ for $i>2m$.
The inequality follows by averaging the preceding conditional bound
over the private randomness and the remaining signs.

\medskip\noindent\textbf{Putting it all together.}
Taking expectations over the signs $U_1,\ldots,U_{2m}$, the noises $Z_1,\ldots,Z_T$, and the algorithm's private randomness $\Upsilon$, and using \eqref{eq:lower-dense-comparator-score}, we obtain
\begin{equation}
	\label{eq:lower-dense-positive-mean}
	\E[\Gamma]\ge\frac{mT}{3}-\frac{mT}{12}=\frac{mT}{4}.
\end{equation}
Since $0\le g\le1$ and $|U_i|\le1$, we have $S_i\in[-T,T]$ for
every $i\in[d]$. Also, each round contributes at least $-m$ to
$S_{\rm L}$, so $S_{\rm L}\ge-mT$. Consequently,
\[
    \Gamma
    =\sum_{j=1}^m\max\{S_{2j-1},S_{2j}\}-S_{\rm L}
    \le mT-(-mT)=2mT.
\]
Using \eqref{eq:lower-dense-regret-comparison} and
\Cref{lem:lower-mean-to-probability}, we obtain
\begin{equation}
\label{eq:lower-dense-positive-event}
    \Pp\left(R_T\ge\frac{m\sqrt T}{1024}\right)
    \ge\Pp\left(\Gamma\ge\frac{mT}{8}\right)
    \ge\frac1{16},
\end{equation}
where the probability is taken over both the adversary's randomness
and the algorithm's private randomness.
The first inequality uses $R_T\ge(\varepsilon/m)\Gamma$ and
$\varepsilon=m/(128\sqrt T)$. The second applies
\Cref{lem:lower-mean-to-probability} using
\eqref{eq:lower-dense-positive-mean} and $\Gamma\le2mT$.
Since $r=m/d>2^{-32}$, we have $m>2^{-32}d$.
Moreover, there are $2^d$ subsets of $[d]$, so
$\binom dm\le2^d$ and $\log\binom dm\le d\log2<d$.
Consequently,
\[
    \frac{m\sqrt T}{1024}
    >2^{-42}d\sqrt T
    >2^{-42}\sqrt{dT\log\binom dm}.
\]
Applying \Cref{lem:lower-fixing-table} with
$b=m\sqrt T/1024$ and $p=1/16$, we can fix a realization of $\Xi$ for
which the same guarantee holds over $\Upsilon$ alone.
The losses in \eqref{eq:lower-dense-item-loss} are then fixed in advance
and remain in $[0,1/m]$ at every round.
Since $1/16\ge2^{-14}$, this proves
\eqref{eq:lower-dense-action-count-conclusion}.
\end{proof}

\subsection{{An} \texorpdfstring{${\Omega(\sqrt{dT\log (1/\delta)})}$}{Omega(sqrt(d T log(1/delta)))} lower bound}
\label{app:lower-confidence}

\begin{proposition}
\label{prop:lower-confidence}
Let $1\le m\le d/2$, let $0<\delta\le2^{-36}$, and suppose that
\(T\ge4dm^2\log(1/\delta).\)
For every randomized algorithm on $\mathcal X_{d,m}$, there exists a
deterministic oblivious loss sequence with every action loss in $[0,1]$
such that
\begin{equation}
\label{eq:lower-confidence-conclusion}
    \Pp\left(R_T\ge2^{-16}\sqrt{dT\log\frac1\delta}\right)
    \ge\delta.
\end{equation}
\end{proposition}

\begin{proof}
\medskip
\noindent\textbf{Instances and regret.}
Fix a core set of items $C\subset[d]$ with $|C|=m-1$, and let $V:=[d]\setminus C$ be the set of remaining items. Define
\(\Delta:=2^{-12}\sqrt{{|V|\log(1/\delta)}/{T}}.\)

 We consider $|V|+1$ instances, indexed by $V\cup\{0\}$. Instance $0$ is the reference instance, while each instance $i\in V$ is designed to favor item $i$. Specifically, for each $i\in V\cup\{0\}$ and $t\in[T]$, we define the loss vector $\ell_t^{(i)}$ as follows:
\begin{align}
\label{eq:lower-confidence-planted-loss}
    \ell_{t,j}^{(0)}
    &:=\frac{1/4+Z_t/8}{m}
       +\frac{\mathbf1\{j\in V\}}{8m},
       \qquad j\in[d],\nonumber\\
    \ell_{t,j}^{(i)}
    &:=\ell_{t,j}^{(0)}-\Delta g(Z_t)\mathbf1\{j=i\},
\end{align}
where $Z_1,\ldots,Z_T$ are i.i.d. random variables with distribution $\operatorname{Unif}[-1,1]$, independent of the algorithm's private randomness $\Upsilon$. In each instance $i\in V\cup\{0\}$, fixing $Z_1,\ldots,Z_T$ determines all loss vectors before the interaction starts, independently of the algorithm's actions; hence the adversary is oblivious. Notice that every core item has a lower loss than every item in $V$ since $\Delta< 1/(8m)$ and this is true in all the instances $i \in V\cup\{0\}$. Moreover, the loss of an item $i\in V$ is reduced only in its associated instance.

 In addition,  for every action $x\in\mathcal X_{d,m}$, its loss in instance $0$ satisfies $1/8+1/(8m)\le x^\top\ell_t^{(0)}\le1/2$. In each instance $i\in V$, this loss is reduced by at most $\Delta$. Since $\Delta<1/(8m)$, every action has loss in $[0,1]$ in every instance.

Furthermore, we have
\begin{align}
\label{eq:lower-confidence-delta-small}
    m\Delta
    &=2^{-12}\sqrt{\frac{m^2|V|\log(1/\delta)}{T}}
      \le2^{-13}\sqrt{\frac{|V|}{d}}\le2^{-13},
\end{align}
where the first inequality uses $T\ge4dm^2\log(1/\delta)$ and the second uses
$|V|\le d$.

As observed above, selecting more than one item outside the core
incurs an additional loss. To track the cumulative cost of these
additional selections, we define
\begin{equation}
	\label{eq:lower-confidence-penalty}
	D_t:=\frac1{8m}\sum_{\tau=1}^t
	\left(\sum_{j\in V}X_{\tau,j}-1\right),
	\qquad t=0,\ldots,T.
\end{equation}
Thus $D_0=0$ and $D_t$ is nondecreasing.
In each instance $i\in V$, the action selecting $C\cup\{i\}$
is optimal at every round, since every item in $C$ has smaller loss
than every item in $V$, and item $i$ has no larger loss than any
other item in $V$.
In instance $i\in V$, subtracting the loss of $C\cup\{i\}$ from
the learner's loss and summing over rounds gives the following
regret decomposition:
\begin{equation}
	\label{eq:lower-confidence-score}
	R_T^{(i)}
	=D_T+\Delta\sum_{t=1}^Tg(Z_t)(1-X_{t,i})
	\ge D_T.
\end{equation}
The first term accounts for selecting additional items outside
the core, while the second accounts for not selecting item $i$.
The inequality follows from $g\ge0$. Notice that \(R_T^{(i)}\) denotes the regret of the algorithm when run in instance \(i\), for every \(i\in V\). Indeed, in the reference instance \(0\), the regret is simply \(R_T^{(0)}=D_T\).

\medskip\noindent\textbf{An auxiliary algorithm.}
We aim to force regret at least $\Delta T/8$. Given the original
algorithm $\cA$, construct an auxiliary algorithm $\overline{\cA}$
that follows $\cA$ as long as its proposed action would keep the
quantity $D_t$ strictly below $\Delta T/8$. At the first action recommendation by $\cA$
that would make $D_t\ge\Delta T/8$, it instead selects all
core items and a fixed item in $V$, the same in every instance,
and repeats this action thereafter.

We use bars to distinguish the quantities associated with
$\overline{\cA}$. In particular, $\bar X_t$, $\overline{\cH}_t$, $\bar D_t$,
and $\bar R_T^{(i)}$ denote its actions, history, cumulative cost,
and regret in instance $i$, respectively.
The identities in \eqref{eq:lower-confidence-penalty} and
\eqref{eq:lower-confidence-score} hold with the quantities associated to $\overline{\cA}$.
Thus, we have
\begin{align}
\label{eq:lower-confidence-total-pulls-alternative}
    \bar D_T<{\Delta T}/{8} \quad \textnormal{and} \quad
    \sum_{t=1}^T\sum_{j\in V}\bar X_{t,j}
    =T+8m\bar D_T<T+m\Delta T=T(1+m\Delta).
\end{align}
The inequality $ \bar D_T<{\Delta T}/{8} $ follows because the fixed action chosen by $\overline{\cA}$ does not increase $\bar D_t$.
The second inequality follows from
\eqref{eq:lower-confidence-penalty}, and $ \bar D_T<{\Delta T}/{8} $.

To compare the two regrets $ R_T^{(i)}$ and $\bar R_T^{(i)}$ , fix an instance $i\in V$ and a realization
of $Z_1,\ldots,Z_T$, and run both algorithms using the same private
randomness $\Upsilon$. Let $\tau$ be the first round at which $\overline{\cA}$ replaces
the action selected by $\cA$, with $\tau:=T+1$ if no replacement occurs.
Their actions and feedback coincide up to round $\tau-1$.
If $\tau\le T$, the action selected by $\cA$ at round $\tau$ makes
$D_\tau\ge\Delta T/8$. Since $D_t$ is nondecreasing,
\eqref{eq:lower-confidence-score} gives
$R_T^{(i)}\ge D_T\ge\Delta T/8$.
If $\tau=T+1$, their entire trajectories
agree, so $R_T^{(i)}=\bar R_T^{(i)}$.
Consequently, with $i$ and $Z_1,\ldots,Z_T$ fixed and
$\Pp_{\Upsilon}$ denoting probability over the algorithm's private randomness,
\begin{align}
\label{eq:lower-confidence-transfer-alternative}
    \Pp_{\Upsilon}\left(R_T^{(i)}\ge\frac{\Delta T}{8}\right)
    &=\Pp_{\Upsilon}\left( R_T^{(i)}\ge\frac{\Delta T}{8},\,\tau\le T\right)+\Pp_{\Upsilon}\left(
        R_T^{(i)}\ge\frac{\Delta T}{8},\,
        \tau=T+1\right)\nonumber\\
    &=\Pp_{\Upsilon}(\tau\le T)+\Pp_{\Upsilon}\left(
        \bar R_T^{(i)}\ge\frac{\Delta T}{8},\,
        \tau=T+1\right)\nonumber\\
    &\ge\Pp_{\Upsilon}\left(
        \bar R_T^{(i)}\ge\frac{\Delta T}{8},\,
        \tau\le T\right)\nonumber+\Pp_{\Upsilon}\left(
        \bar R_T^{(i)}\ge\frac{\Delta T}{8},\,
        \tau=T+1\right)\nonumber\\
    &=\Pp_{\Upsilon}\left(\bar R_T^{(i)}\ge\frac{\Delta T}{8}\right).
\end{align}
Therefore, it is enough to prove the lower bound for
$\overline{\cA}$.
To simplify notation, from now on we omit the bars from the auxiliary
algorithm's quantities. Thus, $X_t$, $Y_t$, $\cH_t$, and $D_t$ now refer to
$\overline{\cA}$. We retain the bar in its regret $\bar R_T^{(i)}$
to distinguish it from the original algorithm's regret $R_T^{(i)}$.

Fix a realization of the algorithm's private randomness $\Upsilon$.
Let $P$ denote the resulting law of $\cH_T$ when the auxiliary
algorithm is run in instance $0$, and, for each $i\in V$, let
$Q_i$ denote its law in instance $i$, with only
$Z_1,\ldots,Z_T$ remaining random. The dependence of these laws
on the fixed realization of $\Upsilon$ is left implicit. We write $\E_P$ and $\E_{Q_i}$ for expectations under $P$ and $Q_i$, respectively.

For each $i\in V$, consider the event
\begin{equation}
\label{eq:lower-confidence-event}
    E_i:=\left\{
        D_T+\Delta\sum_{t=1}^Tg(Z_t)(1-X_{t,i})
        \ge\frac{\Delta T}{8}
    \right\}.
\end{equation}
Under $P$, the algorithm is run in instance $0$. At every round,
we can recover $Z_t$ from the observed action and feedback using
\begin{equation}
\label{eq:lower-confidence-recover-z}
    Z_t=8\left(Y_t-\frac14-\frac1{8m}
                   \sum_{j\in V}X_{t,j}\right).
\end{equation}
Also, $D_T$ depends only on the actions $X_1,\ldots,X_T$.
Therefore, all the quantities defining $E_i$ can be computed
from the history $\cH_T$ under $P$.
Now fix $i\in V$ and consider $Q_i$. There are two cases:
\begin{itemize}
    \item If $X_{t,i}=0$, Equation \eqref{eq:lower-confidence-recover-z} recovers $Z_t$.
    \item If $X_{t,i}=1$, the summand $g(Z_t)(1-X_{t,i})$
    is zero, so we do not need to recover $Z_t$.
\end{itemize}
Since $D_T$ depends only on the actions, we can thus compute
\(
    D_T+\Delta\sum_{t=1}^T g(Z_t)(1-X_{t,i})
\)
from the history under $Q_i$ as well.
Under both $P$ and $Q_i$, the event $E_i$ can therefore be written as
\[
    E_i=\left\{
        D_T+\Delta\sum_{\substack{t=1\\X_{t,i}=0}}^T
        g\!\left(
            8\left(Y_t-\frac14-\frac1{8m}
                \sum_{j\in V}X_{t,j}\right)
        \right)
        \ge\frac{\Delta T}{8}
    \right\}.
\]
Indeed, when $X_{t,i}=0$, the argument of $g$ equals $Z_t$ in Equation~\eqref{eq:lower-confidence-event}
under both laws. When $X_{t,i}=1$, the corresponding term in
the original definition of $E_i$ is equal to zero. Since $D_T$ depends
only on the actions, the above condition involves only
the observed history $\cH_T$. Thus, for each $i\in V$, the event $E_i$ is entirely determined
by the observed history $\cH_T$, in the same way under both
$P$ and $Q_i$. Consequently, $P(E_i)$ and $Q_i(E_i)$ are well-defined. Moreover, under $Q_i$, the left-hand side of the inequality
defining $E_i$ equals $\bar R_T^{(i)}$, whereas under $P$
the auxiliary algorithm's regret is only $D_T$.
Moreover, under $P$, \eqref{eq:lower-confidence-recover-z}
holds at every round, including when $X_{t,i}=1$.
Consequently, conditioning on $\cH_T$ fixes the actions, $D_T$,
and all the variables $Z_t$. Averaging over the indices $i\in V$ gives
\begin{align}
\label{eq:lower-confidence-average-score}
    &\frac1{|V|}\sum_{i\in V}\E_P\left[
        D_T+\Delta\sum_{t=1}^Tg(Z_t)(1-X_{t,i})
        \,\middle|\,\cH_T
    \right]\nonumber\\
    &\quad=D_T+\Delta\sum_{t=1}^Tg(Z_t)
        \left(1-\frac1{|V|}\sum_{j\in V}X_{t,j}\right)\nonumber\\
    &\quad=\Delta\left(1-\frac1{|V|}\right)\sum_{t=1}^Tg(Z_t)
      +\sum_{t=1}^T
       \left(\sum_{j\in V}X_{t,j}-1\right)
       \left(\frac1{8m}-\frac{\Delta g(Z_t)}{|V|}\right)\nonumber\\
    &\quad\ge\frac\Delta2\sum_{t=1}^Tg(Z_t).
\end{align}
The first equality uses the fact that the quantities
inside each conditional expectation are fixed by that history.
The second equality substitutes \eqref{eq:lower-confidence-penalty}.
The inequality uses $|V|\ge2$, $0\le g\le1$, and
\eqref{eq:lower-confidence-delta-small}. 
Taking expectation under $P$ and using the tower property yields
\begin{equation}
\label{eq:lower-confidence-average-mean}
    \frac{\Delta T}{3}
    \le\frac1{|V|}\sum_{i\in V}\E_P\left[
        D_T+\Delta\sum_{t=1}^Tg(Z_t)(1-X_{t,i})
    \right]
    \le\frac{\Delta T}{8}
         +\frac{\Delta T}{|V|}\sum_{i\in V}P(E_i).
\end{equation}
The first inequality follows from
\eqref{eq:lower-confidence-average-score}, because
$\E_P[g(Z_t)]=2/3$ for every $t$.
For the second inequality, we split each expectation over $E_i$ and $E_i^c$.
The expression  $D_T+\Delta\sum_{t=1}^Tg(Z_t)(1-X_{t,i})
$ is at most $\Delta T/8$ on $E_i^c$
and at most $9\Delta T/8$ on $E_i$, by
\eqref{eq:lower-confidence-total-pulls-alternative} and $0\le g(x)\le1$ for every $x \in [-1,1]$.
Thus each expectation is at most
$(\Delta T/8)P(E_i^c)+(9\Delta T/8)P(E_i)
=\Delta T/8+\Delta T\,P(E_i)$.
Subtracting $\Delta T/8$ and dividing by $\Delta T$ in
\eqref{eq:lower-confidence-average-mean} gives
\begin{align}
\label{eq:lower-confidence-average-event}
    \frac1{|V|}\sum_{i\in V}P(E_i)
    &\ge\frac13-\frac18=\frac5{24}\ge\frac1{16}.
\end{align}
Thus the events $E_i$ have constant probability on average under
$P$. We next compare $P$ with each $Q_i$ to show that the average
of $Q_i(E_i)$ over $i\in V$ is larger than $\delta$, for every fixed
realization of $\Upsilon$.

\medskip\noindent\textbf{Comparing $P$ and $Q_i$.}
Fix $i\in V$. Since $\Upsilon$ is already fixed, conditioning on the
same preceding history $\cH_{t-1}$ determines the same action
$X_t$ under $P$ and $Q_i$.
In either instance, $Z_t$ remains independent of $\cH_{t-1}$
and uniform on $[-1,1]$.
If $X_{t,i}=0$, the conditional feedback laws coincide.
If $X_{t,i}=1$, subtracting
$1/4+(8m)^{-1}\sum_{j\in V}X_{t,j}$ from the feedback and multiplying
by four gives $\psi_0(Z_t)$ under $P$ and $\psi_{-4\Delta}(Z_t)$ under $Q_i$.
This common invertible affine transformation preserves KL divergence. Thus, we have
\begin{equation}
	\label{eq:lower-confidence-conditional-kl}
	\KL\left(
	P(\cdot\mid\cH_{t-1})
	\,\middle\|\,
	Q_i(\cdot\mid\cH_{t-1})
	\right)
	\le128(4\Delta)^2X_{t,i}.
\end{equation}
Here, $P(\cdot\mid\cH_{t-1})$ and $Q_i(\cdot\mid\cH_{t-1})$
denote the conditional distributions of $Y_t$ under $P$ and $Q_i$,
respectively, given the same preceding history $\cH_{t-1}$,
with $\Upsilon$ fixed throughout.
The bound in \eqref{eq:lower-confidence-conditional-kl} is zero when $X_{t,i}=0$. Otherwise, we can employ  \Cref{lem:lower-feedback-density}
with parameters $0$ and $-4\Delta$. Since
$4\Delta\le2^{-11}<1/16$, we have	\eqref{eq:lower-confidence-conditional-kl}.
Using the KL chain rule and \eqref{eq:lower-confidence-conditional-kl}, we have
\begin{align}
\label{eq:lower-confidence-instance-kl}
	\KL(P\|Q_i)
	&=\sum_{t=1}^T
	\E_P\!\left[
	\KL\!\left(
	P(\cdot\mid\mathcal H_{t-1})
	\,\middle\|\,
	Q_i(\cdot\mid\mathcal H_{t-1})
	\right)
	\right] \nonumber \\
	&\le
	\sum_{t=1}^T
	\E_P\!\left[128(4\Delta)^2X_{t,i}\right] \nonumber\\
	&=
	128(4\Delta)^2
	\E_P\!\left[\sum_{t=1}^T X_{t,i}\right].
\end{align}
Averaging over $i\in V$ gives
\begin{align}
\label{eq:lower-confidence-kl}
    \frac1{|V|}\sum_{i\in V}\KL(P\|Q_i)
    &\le\frac{128(4\Delta)^2}{|V|}
         \E_P\left[\sum_{t=1}^T\sum_{i\in V}X_{t,i}\right]\nonumber\\
    &\le2^{11}\frac{\Delta^2T}{|V|}(1+m\Delta)\nonumber\\
    &\le2^{-12}\log\frac1\delta.
\end{align}
The first inequality averages \eqref{eq:lower-confidence-instance-kl}
and uses linearity of expectation under the same law $P$.
The second uses \eqref{eq:lower-confidence-total-pulls-alternative}.
The last uses $1+m\Delta<2$ and
$\Delta^2T/|V|=2^{-24}\log(1/\delta)$.

We now bound
the average of $Q_i(E_i)$ with $\Upsilon$ still fixed. In particular, we have
\begin{align}
\label{eq:lower-confidence-binary}
    \left(\frac1{|V|}\sum_{i\in V}P(E_i)\right)
      \log\frac1{\frac1{|V|}\sum_{i\in V}Q_i(E_i)}
    &\le
      \kl\left(
          \frac1{|V|}\sum_{i\in V}P(E_i),
          \frac1{|V|}\sum_{i\in V}Q_i(E_i)
      \right)+\log2
      \nonumber\\
    &\le
      \frac1{|V|}\sum_{i\in V}
      \kl\bigl(P(E_i),Q_i(E_i)\bigr)+\log2
      \nonumber\\
    &\le
      \frac1{|V|}\sum_{i\in V}\KL(P\|Q_i)+\log2.
\end{align}
As before, $\kl(a,b)$ denotes the KL divergence between Bernoulli
laws with success probabilities $a$ and $b$.
The first inequality uses the binary entropy bound
$\kl(a,b)\ge a\log(1/b)-\log2$.
The second follows from joint convexity of KL divergence.
The last follows from data processing applied to the indicator
of $E_i$ under $P$ and $Q_i$, for each $i\in V$.

Suppose by contradiction that $|V|^{-1}\sum_{i\in V}Q_i(E_i)\le\delta$.
Then
\begin{align*}
    \frac1{16}\log\frac1\delta
    &\le\left(\frac1{|V|}\sum_{i\in V}P(E_i)\right)
         \log\frac1{\frac1{|V|}\sum_{i\in V}Q_i(E_i)}\\
    &\le\frac1{|V|}\sum_{i\in V}\KL(P\|Q_i)+\log2\\
    &\le\left(2^{-12}+\frac1{36}\right)\log\frac1\delta
      <\frac1{16}\log\frac1\delta.
\end{align*}
The first inequality uses \eqref{eq:lower-confidence-average-event}
and the assumed bound on the average of $Q_i(E_i)$.
The second uses \eqref{eq:lower-confidence-binary}.
The third uses \eqref{eq:lower-confidence-kl} and
$\log(1/\delta)\ge36\log2$.
This contradiction proves that
$|V|^{-1}\sum_{i\in V}Q_i(E_i)>\delta$
for every fixed realization of $\Upsilon$.

For each fixed realization of $\Upsilon$, $Q_i(E_i)$ is the probability
of $E_i$ with respect to the randomness of $Z_1,\ldots,Z_T$.
Averaging over $\Upsilon$, before choosing $i$, gives
\begin{align}
\label{eq:lower-confidence-average-randomness}
    \frac1{|V|}\sum_{i\in V}\E_{\Upsilon}[Q_i(E_i)]
    &=\E_{\Upsilon}\left[\frac1{|V|}\sum_{i\in V}Q_i(E_i)\right]
      >\delta,
\end{align}
where $\E_{\Upsilon}$ denotes expectation over the algorithm's private
randomness $\Upsilon$.
Therefore, there is a fixed index $i^\star\in V$, which does not depend
on the realization of $\Upsilon$, such that $\E_{\Upsilon}\!\left[Q_{i^\star}(E_{i^\star})\right]>\delta.$
By the law of total expectation,
\[
    \E_{\Upsilon}\left[Q_{i^\star}(E_{i^\star})\right]= \E_{\Upsilon}\left[\E_{Q_{i^\star}} \left[ \mathbf1\left\{\bar R_T^{(i^\star)}\ge\frac{\Delta T}{8}\right\} \right] \right]
    =\Pp\left(
        \bar R_T^{(i^\star)}\ge\frac{\Delta T}{8}
    \right),
\]
where the probability on the right is taken over both the algorithm's
private randomness $\Upsilon$ and $Z=(Z_1,\ldots,Z_T)$, in instance $i^\star$.
Fix this index $i^\star$. The comparison in
\eqref{eq:lower-confidence-transfer-alternative} holds for every
fixed realization of $Z_1,\ldots,Z_T$, with probabilities taken
over $\Upsilon$. Averaging it over $Z_1,\ldots,Z_T$ therefore yields
\begin{align*}
    \Pp\left(R_T^{(i^\star)}\ge\frac{\Delta T}{8}\right)
    \ge\Pp\left(\bar R_T^{(i^\star)}\ge\frac{\Delta T}{8}\right)=\E_{\Upsilon}\!\left[Q_{i^\star}(E_{i^\star})\right]>\delta,
\end{align*}
where both probabilities are over $Z_1,\ldots,Z_T$ and the
algorithm's private randomness $\Upsilon$.
Moreover,
\begin{align*}
    \frac{\Delta T}{8}
    &=2^{-15}\sqrt{|V|T\log\frac1\delta}
      \ge2^{-16}\sqrt{dT\log\frac1\delta}.
\end{align*}
The equality follows from the definition of $\Delta$, and the inequality uses
$|V|=d-m+1\ge d/2$.
For this fixed $i^\star$, we encode the adversary's randomness in $\Xi=(Z_1,\ldots,Z_T)$, independently of $\Upsilon$.
By \Cref{lem:lower-fixing-table}, there exists a realization $\bar\xi$ of $\Xi$ for which the original algorithm incurs regret at least $\Delta T/8$ with probability at least $\delta$ over $\Upsilon$ alone.
The losses in \eqref{eq:lower-confidence-planted-loss} then form a
deterministic oblivious sequence, proving the claim.
\end{proof}

\subsection{Proof of Theorem~\ref{thm:lower-bound-explicit}}
\label{app:proof-lower-bound-explicit}

\explicitLowerBoundTheorem*

\begin{proof}
We first assume that $1\le m\le d/2$.
The assumption on $T$ implies both
\(T\ge64d^2\) and \(
    T\ge4dm^2\log(1/\delta).\)
Thus the horizon requirements of
\Cref{prop:lower-action-count,prop:lower-action-count-dense,prop:lower-confidence}
are satisfied.
Suppose that $\log\binom dm\ge\log(1/\delta)$.
If $m/d\le2^{-32}$, \Cref{prop:lower-action-count} gives
\[
    \Pp\left(R_T\ge2^{-24}\sqrt{dT\log\binom dm}\right)\ge\frac18.
\]
As in the proof of that proposition, \Cref{lem:lower-fixing-table}
allows us to fix the adversary's auxiliary randomness while preserving
this guarantee over the algorithm's private randomness.
The resulting adversary is deterministic, adaptive, and non-anticipating.
If $m/d>2^{-32}$, \Cref{prop:lower-action-count-dense} gives a
deterministic oblivious loss sequence satisfying
\[
    \Pp\left(R_T\ge2^{-42}\sqrt{dT\log\binom dm}\right)\ge2^{-14}.
\]
In both cases, item losses belong to $[0,1/m]$, so every $m$-set action
has loss in $[0,1]$.
Since $2^{-24}\ge2^{-42}$ and $1/8\ge2^{-14}\ge\delta$, we obtain
\[
    \Pp\left(R_T\ge2^{-42}\sqrt{dT\log\binom dm}\right)\ge\delta.
\]
Suppose instead that $\log(1/\delta)>\log\binom dm$.
Since $0<\delta\le2^{-36}$ and $T\ge4dm^2\log(1/\delta)$,
\Cref{prop:lower-confidence} gives a deterministic oblivious loss
sequence with action losses in $[0,1]$ such that
\[
    \Pp\left(R_T\ge2^{-16}\sqrt{dT\log(1/\delta)}\right)\ge\delta.
\]
This also implies
$\Pp(R_T\ge2^{-42}\sqrt{dT\log(1/\delta)})\ge\delta$.

Consequently, in either case, every randomized algorithm on
$\mathcal X_{d,m}$ admits a deterministic adaptive non-anticipating
adversary with action losses in $[0,1]$ for which
\[
    \Pp\left(
        R_T\ge2^{-42}\sqrt{dT\max\left\{\log\binom dm,\log(1/\delta)\right\}}
    \right)\ge\delta.
\]
Since
\[
    \max\left\{\log\binom dm,\log(1/\delta)\right\}
    \ge\frac12\left(\log\binom dm+\log(1/\delta)\right),
\]
we obtain
\[
    \Pp\left(
        R_T\ge2^{-43}\sqrt{dT\left(\log\binom dm+\log(1/\delta)\right)}
    \right)\ge\delta.
\]
This proves the theorem when $m\le d/2$.

Now suppose that $m>d/2$.
{Then $1\le d-m\le d/2$ and $\min\{m,d-m\}=d-m$,
so the horizon assumption gives}
{
\[
    T\ge\left\lceil64d^2+4d(d-m)^2\log(1/\delta)\right\rceil.
\]
}
The case already proved therefore applies to every randomized algorithm
on $\mathcal X_{d,d-m}$ at the same horizon $T$ and confidence level $\delta$.
By \Cref{lem:lower-complement}, the resulting lower bound transfers to
$\mathcal X_{d,m}$, preserving the probability guarantee and the
action-loss range $[0,1]$.
Since $\binom d{d-m}=\binom dm$, this is exactly the claimed bound.
\end{proof}

\section*{AI methodology}

We used OpenAI's GPT-5.5, GPT-5.6 Sol, and GPT-6 Astra to brainstorm and refine algorithmic ideas, develop preliminary proof arguments, and assist with literature searches and \LaTeX{} formatting.
Since the authors are not native English speakers, we also used these models to improve grammar, punctuation, and phrasing throughout the manuscript.
All AI-assisted material incorporated into the manuscript was subsequently checked by the authors and revised where appropriate.

OpenAI's GPT-6 Astra was also used to produce the accompanying Lean formalization.
Its correspondence with the manuscript was assessed using separate instances of GPT-6 Astra, distinct from those used to produce the formalization.
The authors thoroughly reviewed the resulting reports.
The Lean sources and instructions for independent verification are available in the
\href{https://github.com/robertocolomboni/an-efficient-minimax-optimal-algorithm-for-adversarial-m-set-bandits-lean-formalization}{companion GitHub repository}.

The authors take full responsibility for the paper and its accompanying materials.
\clearpage
\begingroup
\raggedright
\bibliographystyle{plainnat}
\bibliography{references_v2}

@article{kiefer1959optimum,
  author  = {Kiefer, Jack},
  title   = {Optimum Experimental Designs},
  journal = {Journal of the Royal Statistical Society: Series B (Methodological)},
  volume  = {21},
  number  = {2},
  pages   = {272--304},
  year    = {1959},
  doi     = {10.1111/j.2517-6161.1959.tb00338.x},
  url     = {https://doi.org/10.1111/j.2517-6161.1959.tb00338.x}
}

@article{kiefer1960equivalence,
  author  = {Kiefer, Jack and Wolfowitz, Jacob},
  title   = {The Equivalence of Two Extremum Problems},
  journal = {Canadian Journal of Mathematics},
  volume  = {12},
  pages   = {363--366},
  year    = {1960},
  doi     = {10.4153/CJM-1960-030-4},
  url     = {https://doi.org/10.4153/CJM-1960-030-4}
}

@inproceedings{bartlett2008highprobability,
  author    = {Bartlett, Peter L. and Dani, Varsha and Hayes, Thomas P. and Kakade, Sham M. and Rakhlin, Alexander and Tewari, Ambuj},
  title     = {High-Probability Regret Bounds for Bandit Online Linear Optimization},
  booktitle = {Proceedings of the 21st Annual Conference on Learning Theory},
  pages     = {335--342},
  publisher = {Omnipress},
  year      = {2008},
  url       = {https://www.learningtheory.org/colt2008/papers/30-Bartlett.pdf}
}

@inproceedings{abernethy2008competing,
  author    = {Abernethy, Jacob and Hazan, Elad and Rakhlin, Alexander},
  title     = {Competing in the Dark: An Efficient Algorithm for Bandit Linear Optimization},
  booktitle = {Proceedings of the 21st Annual Conference on Learning Theory},
  pages     = {263--274},
  year      = {2008},
  url       = {https://www.learningtheory.org/colt2008/papers/123-Abernethy.pdf}
}

@inproceedings{abernethy2009beating,
  author    = {Abernethy, Jacob and Rakhlin, Alexander},
  title     = {Beating the Adaptive Bandit with High Probability},
  booktitle = {Proceedings of the 22nd Annual Conference on Learning Theory},
  year      = {2009},
  url       = {https://www.learningtheory.org/colt2009/papers/025.pdf}
}

@inproceedings{bubeck2012towards,
  author    = {Bubeck, S{\'e}bastien and Cesa-Bianchi, Nicol{\`o} and Kakade, Sham M.},
  title     = {Towards Minimax Policies for Online Linear Optimization with Bandit Feedback},
  booktitle = {Proceedings of the 25th Annual Conference on Learning Theory},
  editor    = {Mannor, Shie and Srebro, Nathan and Williamson, Robert C.},
  series    = {Proceedings of Machine Learning Research},
  volume    = {23},
  pages     = {41.1--41.14},
  publisher = {PMLR},
  year      = {2012},
  eprint    = {1202.3079},
  archivePrefix = {arXiv},
  url       = {https://proceedings.mlr.press/v23/bubeck12a.html}
}

@article{cesabianchi2012combinatorial,
  author  = {Cesa-Bianchi, Nicol{\`o} and Lugosi, G{\'a}bor},
  title   = {Combinatorial Bandits},
  journal = {Journal of Computer and System Sciences},
  volume  = {78},
  number  = {5},
  pages   = {1404--1422},
  year    = {2012},
  doi     = {10.1016/j.jcss.2012.01.001},
  url     = {https://doi.org/10.1016/j.jcss.2012.01.001}
}

@inproceedings{awerbuch2004adaptive,
  author    = {Awerbuch, Baruch and Kleinberg, Robert D.},
  title     = {Adaptive Routing with End-to-End Feedback: Distributed Learning and Geometric Approaches},
  booktitle = {Proceedings of the Thirty-Sixth Annual ACM Symposium on Theory of Computing},
  pages     = {45--53},
  publisher = {ACM},
  year      = {2004},
  doi       = {10.1145/1007352.1007367},
  url       = {https://doi.org/10.1145/1007352.1007367}
}

@inproceedings{mcmahan2004online,
  author    = {McMahan, H. Brendan and Blum, Avrim},
  title     = {Online Geometric Optimization in the Bandit Setting Against an Adaptive Adversary},
  booktitle = {Proceedings of the 17th Annual Conference on Learning Theory},
  series    = {Lecture Notes in Computer Science},
  volume    = {3120},
  pages     = {109--123},
  publisher = {Springer},
  year      = {2004},
  doi       = {10.1007/978-3-540-27819-1_8},
  url       = {https://doi.org/10.1007/978-3-540-27819-1_8}
}

@article{audibert2014regret,
  author  = {Audibert, Jean-Yves and Bubeck, S{\'e}bastien and Lugosi, G{\'a}bor},
  title   = {Regret in Online Combinatorial Optimization},
  journal = {Mathematics of Operations Research},
  volume  = {39},
  number  = {1},
  pages   = {31--45},
  year    = {2014},
  doi     = {10.1287/moor.2013.0598},
  eprint  = {1204.4710},
  archivePrefix = {arXiv},
  url     = {https://doi.org/10.1287/moor.2013.0598}
}

@inproceedings{combes2015combinatorial,
  author    = {Combes, Richard and Talebi Mazraeh Shahi, Mohammad Sadegh and Prouti{\`e}re, Alexandre and Lelarge, Marc},
  title     = {Combinatorial Bandits Revisited},
  booktitle = {Advances in Neural Information Processing Systems},
  volume    = {28},
  pages     = {2116--2124},
  publisher = {Curran Associates, Inc.},
  year      = {2015},
  eprint    = {1502.03475},
  archivePrefix = {arXiv},
  url       = {https://proceedings.neurips.cc/paper/2015/hash/0ce2ffd21fc958d9ef0ee9ba5336e357-Abstract.html}
}

@misc{braun2016efficient,
  author        = {Braun, G{\'a}bor and Pokutta, Sebastian},
  title         = {An Efficient High-Probability Algorithm for Linear Bandits},
  year          = {2016},
  eprint        = {1610.02072},
  archivePrefix = {arXiv},
  primaryClass  = {cs.DS},
  url           = {https://arxiv.org/abs/1610.02072}
}

@article{hazan2016volumetric,
  author  = {Hazan, Elad and Karnin, Zohar},
  title   = {Volumetric Spanners: An Efficient Exploration Basis for Learning},
  journal = {Journal of Machine Learning Research},
  volume  = {17},
  number  = {119},
  pages   = {1--34},
  year    = {2016},
  url     = {https://jmlr.org/papers/v17/hazan16a.html}
}

@inproceedings{cohen2017tight,
  author    = {Cohen, Alon and Hazan, Tamir and Koren, Tomer},
  title     = {Tight Bounds for Bandit Combinatorial Optimization},
  booktitle = {Proceedings of the 30th Conference on Learning Theory},
  editor    = {Kale, Satyen and Shamir, Ohad},
  series    = {Proceedings of Machine Learning Research},
  volume    = {65},
  pages     = {629--642},
  publisher = {PMLR},
  year      = {2017},
  eprint    = {1702.07539},
  archivePrefix = {arXiv},
  url       = {https://proceedings.mlr.press/v65/cohen17a.html}
}

@inproceedings{sakaue2018efficient,
  author    = {Sakaue, Shinsaku and Ishihata, Masakazu and Minato, Shin-ichi},
  title     = {Efficient Bandit Combinatorial Optimization Algorithm with Zero-Suppressed Binary Decision Diagrams},
  booktitle = {Proceedings of the Twenty-First International Conference on Artificial Intelligence and Statistics},
  editor    = {Storkey, Amos and Perez-Cruz, Fernando},
  series    = {Proceedings of Machine Learning Research},
  volume    = {84},
  pages     = {585--594},
  publisher = {PMLR},
  year      = {2018},
  url       = {https://proceedings.mlr.press/v84/sakaue18a.html}
}

@inproceedings{ito2019improved,
  author    = {Ito, Shinji and Hatano, Daisuke and Sumita, Hanna and Takemura, Kei and Fukunaga, Takuro and Kakimura, Naonori and Kawarabayashi, Ken-ichi},
  title     = {Improved Regret Bounds for Bandit Combinatorial Optimization},
  booktitle = {Advances in Neural Information Processing Systems},
  volume    = {32},
  pages     = {12050--12059},
  publisher = {Curran Associates, Inc.},
  year      = {2019},
  url       = {https://proceedings.neurips.cc/paper/2019/hash/d5b3d8dadd770c460b1cde910a711987-Abstract.html}
}

@inproceedings{ito2019oracle,
  author    = {Ito, Shinji and Hatano, Daisuke and Sumita, Hanna and Takemura, Kei and Fukunaga, Takuro and Kakimura, Naonori and Kawarabayashi, Ken-ichi},
  title     = {Oracle-Efficient Algorithms for Online Linear Optimization with Bandit Feedback},
  booktitle = {Advances in Neural Information Processing Systems},
  volume    = {32},
  pages     = {10590--10599},
  publisher = {Curran Associates, Inc.},
  year      = {2019},
  url       = {https://proceedings.neurips.cc/paper/2019/hash/e6385d39ec9394f2f3a354d9d2b88eec-Abstract.html}
}

@inproceedings{lee2020bias,
  author    = {Lee, Chung-Wei and Luo, Haipeng and Wei, Chen-Yu and Zhang, Mengxiao},
  title     = {Bias No More: High-Probability Data-Dependent Regret Bounds for Adversarial Bandits and {MDPs}},
  booktitle = {Advances in Neural Information Processing Systems},
  volume    = {33},
  pages     = {15522--15533},
  publisher = {Curran Associates, Inc.},
  year      = {2020},
  eprint    = {2006.08040},
  archivePrefix = {arXiv},
  url       = {https://proceedings.neurips.cc/paper/2020/hash/b2ea5e977c5fc1ccfa74171a9723dd61-Abstract.html}
}

@inproceedings{zimmert2022return,
  author    = {Zimmert, Julian and Lattimore, Tor},
  title     = {Return of the Bias: Almost Minimax Optimal High Probability Bounds for Adversarial Linear Bandits},
  booktitle = {Proceedings of the Thirty-Fifth Conference on Learning Theory},
  editor    = {Loh, Po-Ling and Raginsky, Maxim},
  series    = {Proceedings of Machine Learning Research},
  volume    = {178},
  pages     = {3285--3312},
  publisher = {PMLR},
  year      = {2022},
  url       = {https://proceedings.mlr.press/v178/zimmert22b.html}
}

@inproceedings{maiti2025efficient,
  author    = {Maiti, Arnab and Fan, Zhiyuan and Jamieson, Kevin and Ratliff, Lillian J. and Farina, Gabriele},
  title     = {Efficient Near-Optimal Algorithm for Online Shortest Paths in Directed Acyclic Graphs with Bandit Feedback Against Adaptive Adversaries},
  booktitle = {Proceedings of the Thirty-Eighth Conference on Learning Theory},
  editor    = {Haghtalab, Nika and Moitra, Ankur},
  series    = {Proceedings of Machine Learning Research},
  volume    = {291},
  pages     = {3881--3932},
  publisher = {PMLR},
  year      = {2025},
  eprint    = {2504.00461},
  archivePrefix = {arXiv},
  url       = {https://proceedings.mlr.press/v291/maiti25a.html}
}

@inproceedings{kontogiannis2025kernelized,
  author    = {Kontogiannis, Andreas and Pollatos, Vasilis and Farina, Gabriele and Mertikopoulos, Panayotis and Panageas, Ioannis},
  title     = {Efficient Kernelized Learning in Polyhedral Games beyond Full Information: From {Colonel Blotto} to Congestion Games},
  booktitle = {Advances in Neural Information Processing Systems},
  volume    = {38},
  year      = {2025},
  note      = {NeurIPS 2025 Main Conference Track},
  url       = {https://proceedings.neurips.cc/paper_files/paper/2025/hash/d6b452383b070d5ccf042d8d25e0cdb2-Abstract-Conference.html}
}

@article{hajek1964asymptotic,
  author  = {H{\'a}jek, Jaroslav},
  title   = {Asymptotic Theory of Rejective Sampling with Varying Probabilities from a Finite Population},
  journal = {The Annals of Mathematical Statistics},
  volume  = {35},
  number  = {4},
  pages   = {1491--1523},
  year    = {1964},
  doi     = {10.1214/aoms/1177700375},
  url     = {https://doi.org/10.1214/aoms/1177700375}
}

@article{chen1994weighted,
  author  = {Chen, Xiang-Hui and Dempster, Arthur P. and Liu, Jun S.},
  title   = {Weighted Finite Population Sampling to Maximize Entropy},
  journal = {Biometrika},
  volume  = {81},
  number  = {3},
  pages   = {457--469},
  year    = {1994},
  doi     = {10.1093/biomet/81.3.457},
  url     = {https://doi.org/10.1093/biomet/81.3.457}
}

@article{chen1997statistical,
  author  = {Chen, Sean X. and Liu, Jun S.},
  title   = {Statistical Applications of the {Poisson-binomial} and Conditional {Bernoulli} Distributions},
  journal = {Statistica Sinica},
  volume  = {7},
  number  = {4},
  pages   = {875--892},
  year    = {1997},
  url     = {https://www3.stat.sinica.edu.tw/statistica/j7n4/j7n44/j7n44.htm}
}

@inproceedings{kulesza2011kdpps,
  author    = {Kulesza, Alex and Taskar, Ben},
  title     = {{k-DPPs}: Fixed-Size Determinantal Point Processes},
  booktitle = {Proceedings of the 28th International Conference on Machine Learning},
  pages     = {1193--1200},
  publisher = {Omnipress},
  year      = {2011},
  url       = {https://icml.cc/2011/papers/611_icmlpaper.pdf}
}

@misc{cesari2026sharp,
  author        = {Cesari, Tommaso and Colomboni, Roberto},
  title         = {Sharp Normalized Covariance Bounds and Constant-Stretch Correlated Sampling on the Hypersimplex},
  year          = {2026},
  eprint        = {2607.13990v2},
  archivePrefix = {arXiv},
  primaryClass  = {math.PR},
  doi           = {10.48550/arXiv.2607.13990},
  note          = {arXiv:2607.13990v2, version 2, 8 September 2026},
  url           = {https://arxiv.org/abs/2607.13990v2}
}

@misc{bacchiocchi2026msets-v2,
  author        = {Bacchiocchi, Francesco and Cesari, Tommaso and Colomboni, Roberto},
  title         = {An Efficient Minimax-Optimal Algorithm for Adversarial {$m$}-Set Bandits},
  year          = {2026},
  eprint        = {2608.12231v2},
  archivePrefix = {arXiv},
  primaryClass  = {cs.LG},
  doi           = {10.48550/arXiv.2608.12231},
  note          = {arXiv:2608.12231v2, version 2, 25 August 2026},
  url           = {https://arxiv.org/abs/2608.12231v2}
}

@book{boyd2004convex,
  author    = {Stephen Boyd and Lieven Vandenberghe},
  title     = {Convex Optimization},
  publisher = {Cambridge University Press},
  year      = {2004}
}

@book{groetschel1988geometric,
  author    = {Gr{\"o}tschel, Martin and Lov{\'a}sz, L{\'a}szl{\'o} and Schrijver, Alexander},
  title     = {Geometric Algorithms and Combinatorial Optimization},
  series    = {Algorithms and Combinatorics},
  volume    = {2},
  publisher = {Springer},
  address   = {Berlin},
  year      = {1988},
  doi       = {10.1007/978-3-642-97881-4},
  url       = {https://doi.org/10.1007/978-3-642-97881-4}
}

@inproceedings{gerchinovitz2016refined,
  author    = {Gerchinovitz, S{\'e}bastien and Lattimore, Tor},
  title     = {Refined Lower Bounds for Adversarial Bandits},
  booktitle = {Advances in Neural Information Processing Systems 29},
  editor    = {Lee, D. D. and Sugiyama, M. and von Luxburg, U. and Guyon, I. and Garnett, R.},
  pages     = {1198--1206},
  publisher = {Curran Associates, Inc.},
  year      = {2016},
  url       = {https://proceedings.neurips.cc/paper/2016/hash/2f37d10131f2a483a8dd005b3d14b0d9-Abstract.html}
}

@article{audibert2010regret,
  author  = {Audibert, Jean-Yves and Bubeck, S{\'e}bastien},
  title   = {Regret Bounds and Minimax Policies under Partial Monitoring},
  journal = {Journal of Machine Learning Research},
  volume  = {11},
  number  = {94},
  pages   = {2785--2836},
  year    = {2010},
  url     = {https://jmlr.org/papers/v11/audibert10a.html}
}

@book{cesabianchi2006prediction,
  author    = {Cesa-Bianchi, Nicol{\`o} and Lugosi, G{\'a}bor},
  title     = {Prediction, Learning, and Games},
  publisher = {Cambridge University Press},
  year      = {2006},
  doi       = {10.1017/CBO9780511546921},
  url       = {https://doi.org/10.1017/CBO9780511546921}
}

@book{lattimore2020bandit,
  author    = {Lattimore, Tor and Szepesv{\'a}ri, Csaba},
  title     = {Bandit Algorithms},
  publisher = {Cambridge University Press},
  year      = {2020},
  doi       = {10.1017/9781108571401},
  url       = {https://tor-lattimore.com/downloads/book/book.pdf}
}
\endgroup

\end{document}